\documentclass[twoside,11pt]{article}
\usepackage{blindtext}

\usepackage[preprint]{jmlr2e}

\usepackage{lastpage}
\jmlrheading{25}{2024}{1-\pageref{LastPage}}{10/23; Revised
12/24}{12/24}{23-1271}{Jakiw Pidstrigach, Youssef Marzouk, Sebastian Reich, Sven Wang}
\ShortHeadings{Infinite-Dimensional Diffusion Models}{Pidstrigach, Marzouk, Reich and Wang}
\firstpageno{1}

\usepackage{natbib}

\usepackage[textsize=tiny]{todonotes}

\usepackage[export]{adjustbox}
\usepackage{footnote}

\usepackage[british]{babel}
\usepackage{amsmath,amssymb,latexsym,mathtools}

\newcommand{\mathd}{\mathrm{d}}

\newcommand{\tmop}[1]{\ensuremath{\operatorname{#1}}}
\newcommand{\tmrsup}[1]{\textsuperscript{#1}}

\newcommand{\tmtextbf}[1]{\text{{\bfseries{#1}}}}
\newcommand{\tmtextit}[1]{\text{{\itshape{#1}}}}

\newenvironment{tmparmod}[3]{\begin{list}{}{\setlength{\topsep}{0pt}\setlength{\leftmargin}{#1}\setlength{\rightmargin}{#2}\setlength{\parindent}{#3}\setlength{\listparindent}{\parindent}\setlength{\itemindent}{\parindent}\setlength{\parsep}{\parskip}} \item[]}{\end{list}}

\newcounter{tmcounter}

\newcommand{\mudata}{\mu_{\tmop{data}}}

\newcommand{\musample}{\hat{\mu}_{\tmop{sample}}}

\newcommand{\R}{\mathbb R}

\newcommand{\Id}{\mathrm{Id}}
\newcommand{\Normal}{\mathcal{N}}

\definecolor{darkred}{rgb}{.7,0,0}

\definecolor{darkgreen}{rgb}{.15,.55,0}

\definecolor{darkblue}{rgb}{0,0,0.7}

\newcommand{\jp}[1]{{#1}}
\newcommand{\jpp}[1]{{#1}}

\usepackage{graphicx}
\graphicspath{ {./figures/} }

\usepackage{comment}

\usepackage{subcaption}
\usepackage{algorithm}
\usepackage{algpseudocode}

\begin{document}
\title{Infinite-Dimensional Diffusion Models}

\author{%
 \name{Jakiw Pidstrigach} \email{jakiw.pidstrigach@stats.ox.ac.uk}\\
 \addr  Institut für Mathematik\\
        Universität Potsdam\\
	Karl-Liebknecht-Str. 24/25\\
        14476 Potsdam, Germany 
 \AND
 \name{Youssef Marzouk} \email{ymarz@mit.edu}\\
 \addr  Statistics and Data Science Center \\
        Massachusetts Institute of Technology\\
        77 Massachusetts Ave\\ 
        Cambridge, MA 02139, USA
  \AND
 \name{Sebastian Reich} \email{sereich@uni-potsdam.de}\\
 \addr  Institut für Mathematik\\
        Universität Potsdam\\
	Karl-Liebknecht-Str. 24/25\\
        14476 Potsdam, Germany
  \AND
 \name{Sven Wang} \email{sven.wang@hu-berlin.de}\\
 \addr  Institut für Mathematik \\
        Humboldt-Universität zu Berlin\\
        Rudower Chaussee 25\\
        12489 Berlin, Germany
}
\editor{Chris Oates}

\maketitle

\begin{abstract}%
Diffusion models have had a profound impact on many application areas, including those where data are intrinsically infinite-dimensional, such as images or time series. The standard approach is first to discretize and then to apply diffusion models to the discretized data. While such approaches are 
practically appealing,
the performance of the resulting algorithms typically deteriorates as discretization parameters are refined. In this paper, we instead directly formulate diffusion-based generative models in infinite dimensions and apply them to the generative modelling of \emph{functions}. We prove that our formulations are well posed in the infinite-dimensional setting and provide \emph{dimension-independent} distance bounds from the sample to the target measure. Using our theory, we also develop guidelines for the design of infinite-dimensional diffusion models. For image distributions, these guidelines are in line with current canonical choices.
For other distributions, however, we can improve upon these canonical choices. 
We demonstrate these results 
both theoretically and empirically, by applying the algorithms to data distributions on manifolds and to distributions arising in Bayesian inverse problems or simulation-based inference.
\end{abstract}

\section{Introduction}
Diffusion models (also score-based generative models or SGMs) \citep{sohl2015deep, song2021scorebased} have recently shown 
great empirical success across a variety of domains. 
In many applications, ranging from image generation \citep{nichol2021improved,dhariwal2021diffusion},
audio \citep{DBLP:conf/iclr/KongPHZC21}, and time series \citep{tashiro2021csdi} to inverse problems \citep{kadkhodaie2021stochastic,batzolis2021conditional},
the signal to be modeled
is actually a discretization of an \emph{infinite-dimensional object} (i.e., a function of space and/or time).
In such a setting, it is natural to apply the algorithm in high dimensions,
corresponding to a fine discretization and a better approximation of the true quantity. Yet theoretical studies of current diffusion models suggest that performance guarantees deteriorate with increasing dimension \citep{chen2022improved,debortoli2022convergence}.

When studying a discretization of an infinite-dimensional object, many application areas  have found great success in directly studying the infinite-dimensional limit and only discretizing the problem in the last step, when implementing an algorithm on a computer. By accurately understanding the infinite-dimensional problem, one can gain valuable insights on how it should be discretized. Sometimes, this leads to algorithms that are \emph{dimension-independent} in that their performance does not degrade when one chooses a finer discretization.

Important areas where it is now standard to study the infinite-dimensional object directly are, for example, Bayesian inverse problems \citep{stuart2010inverse} and nonparametric statistics \citep{T09, gine_nickl_2015}. Accordingly, many Markov chain Monte Carlo algorithms used for sampling, such as the Metropolis-adjusted Langevin \citep{cotter2013statsci} or Hamiltonian Monte Carlo \citep{beskos2011hybrid} algorithms, have successfully been generalized to the infinite-dimensional setting; in addition to being an empirical success, these efforts have also led to dimension-independent convergence guarantees (see \citet{hairer2014spectral, bou2021two, 10.1093/imanum/drac052}).

In the common implementation of the diffusion model algorithm, one first discretizes the data (for example images to pixels or wavelet coefficients, or functions to their evaluations on a grid) and then applies the algorithm in $\mathbb{R}^D$, as described in \citet{song2021scorebased}. When doing so, one does not consider the implications of the discretization dimension $D$. In particular, if there is no well-defined limiting algorithm as $D \to \infty$, one cannot expect the algorithm's performance to be stable as $D$ becomes large. 
This instability can potentially be mitigated by \emph{defining the diffusion model algorithm directly in infinite dimensions}, and studying its properties there. Once the algorithm is modified so that it exists in infinite dimensions, the discretized formulations that are implementable on a computer will possess \emph{dimension-independent} properties. 

\subsection{Challenges in Extending Diffusion Models to Infinite Dimensions}\label{sec:finite_dim_diff_models}
Let us briefly recall the well-known finite-dimensional diffusion model setting. 
A forward SDE, typically an Ornstein--Uhlenbeck process, is used to diffuse the data $\mudata$:
\begin{equation}
	\mathd X_t = - \frac{1}{2} X_t \mathd t + \mathd W_t, \qquad X_0 \sim \mudata \, . \label{fd_forward_sde}
\end{equation}
The densities of its marginal distributions are denoted by $p_t$. The following so-called ``reverse SDE'' will traverse the marginals of $X_t$ backward:
\begin{equation}
	\mathd Y_t = \frac{1}{2} Y_t \mathd t + \nabla \log p_{T-t}(Y_t)\mathd t + \mathd W_t, \qquad Y_0 \sim p_T, \label{fd_reverse_sde}
\end{equation}
where $W_t$ is a different Wiener process/Brownian motion than in \eqref{fd_forward_sde}. In particular, $Y_T \sim p_{T-T} = \mudata$\jp{, where by a slight abuse of notation we denote both the density and the measure itself by $\mudata$}. The goal of the diffusion model algorithm is to approximate paths of $Y_t$ and use the realizations at time $T$ as approximate samples from $\mudata$. Since the marginals of the forward SDE $X_t$ converge to $\mathcal{N}(0, I)$ at an exponential speed, one can approximate the unknown term $p_T$ by $\mathcal{N}(0, I)$. Furthermore, $\nabla \log p_t$ can be approximated using score-matching techniques \citep{vincent2011connection}. There are three main challenges in generalizing this construction to infinite dimensions, which we highlight next.

\subsubsection{Choice of the noising process}
In finite dimensions, the process $(W_t:t\ge 0)$ in \eqref{fd_forward_sde} is standard Brownian motion. Therefore, the noise increments for different coordinates $i$ and $j$, e.g.,  $W_t^i - W_s^i$ and $W_t^j - W_s^j$, are independent and identically distributed. In infinite dimensions, one can associate a white-noise process $W_t^U$ to each Hilbert space $U$, with the property that the coordinates of $W_t^U$ in an orthonormal basis of $U$ are independent and identically distributed. Therefore, one has to determine \textit{which} white-noise process (i.e., which Hilbert space) to choose in the infinite-dimensional limit. 

The common case of discretizing the infinite-dimensional target object to $\mathbb{R}^D$, e.g., discretizing a function onto a grid or real-life scenery into image pixels, and then choosing $W_t$ as a standard Brownian motion on $\mathbb{R}^D$, means that at each grid point we will add independent noise values. In particular, as the grid grows finer, even for arbitrarily close values $a \approx b$, the evaluations $X_t(a)$ and $X_t(b)$ will be perturbed with independent noise. The limiting Wiener process will be $W_t^{L^2}$, i.e., the process associated to $U = L^2$, 
also called \emph{space-time white noise}. We have depicted space-time white noise in Figure \ref{fig:space_time_white_noise}. 

In infinite dimensions, however, the choice of the noise process is a subtle issue, as it has a crucial impact on the space on which the diffusion process is supported. For instance, the above `canonical' choice of space-time white noise will lead to $X_t$ and $W_t$ having such irregular samples that they are \textit{not} supported in $L^2$ anymore. While we will also study such processes due to their widespread use in practice, we will see that other choices can be beneficial from a theoretical as well as a practical standpoint. 

\subsubsection{Score function}
The score function $\nabla \log p_t$ in \eqref{fd_reverse_sde} is typically defined via the Lebesgue density $p_t$ of the law of $X_t$. Yet in infinite-dimensional vector spaces, the Lebesgue measure no longer exists; hence one can no longer specify the score functions in the same manner. Therefore, a key question is: How does one define and make sense of $\nabla \log p_t$ without relying on the notion of Lebesgue density, and still define an algorithm which provably samples from the correct measure?

\subsubsection{Denoising score matching objective}
The score $\nabla \log p_t$ is typically approximated by a neural network $\tilde{s}(t,x)$ in some chosen neural network class, and identified by minimizing the denoising score-matching objective,
\[
    \text{Loss}(\tilde{s}) = \int_0^T \mathbb{E}[\|\nabla \log p_t(X_t) - \tilde{s}(t, X_t)\|^2_K] \mathd t,
\]
over this class. Similarly to the choice of the noising process, it is not clear which Hilbert space $K$ and norm $\|\cdot\|_K$ should be used for the analogous objective in infinite dimensions.

\subsection{Contributions}

Our paper, for the first time, formulates the diffusion model algorithm directly on  infinite-dimensional spaces, and proves that this formulation is well-posed and satisfies crucial theoretical guarantees.

To formulate the reverse SDE in infinite dimensions, we must find a way to handle the $\nabla \log p_t$ term, as discussed in the last section. We do this \emph{by replacing the score with a conditional expectation}, in Definition \ref{def:reverse_drifts}. This definition then carries over to the infinite-dimensional case. Furthermore, we are able to show under which circumstances one can generalize the denoising score matching objective to identify the neural network $\tilde{s}(t,x)$ in Lemma \ref{lemma score matching}.

To justify approximating the reverse SDE to obtain samples from $\mudata$, we proceed in multiple steps. First, in Theorem \ref{theorem SDEs well defined}, we show that the time-reversal of the forward SDE also satisfies an appropriate reverse SDE. The terminal condition of the reverse SDE will have distribution $\mudata$. To simulate this reverse SDE in practice, however, both its initial conditions and drift must be approximated. In Lemma~\ref{lemma score matching} we establish under which conditions we can use the common denoising score matching objective to approximate the drift of the reverse SDE in infinite dimensions.

Second, we prove that the solution to such an SDE exists for general initial conditions---and in particular, for our approximate initial conditions. Moreover, we prove that the solution is unique; otherwise we could be approximating a different reverse SDE solution that does not sample $\mudata$ at the terminal time $T$. We provide rigorous uniqueness results under two distinct scenarios: first, in Theorem \ref{thm:uniqueness_manifold}, for $\mudata$ which satisfies a manifold hypothesis; and second, in Theorem \ref{thm:uniqueness_gaussian}, under the assumption that $\mudata$ has density with respect to a Gaussian measure. The first case is relevant for the typical use cases of diffusion models, as image data are usually supported on lower-dimensional manifolds or other substructures.
The second case is relevant when, for example, applying diffusion models to Bayesian inverse problems or related problems of simulation-based inference. 

Finally, building upon the preceding results, we establish \emph{dimension-independent} convergence rates in Theorem~\ref{theorem distance}. Our bound is quantitative and shows how the distance relies on different choices made in the diffusion model algorithm.

The theory described above guides choices for the noise process $W_t^U$ and the loss norm $\|\cdot\|_K$. Both will depend on the properties of $\mudata$. In Section \ref{sec:dms_in_infinite_dimensions} we 
{discuss the implications of the theory for implementing diffusion models in infinite dimensions}. In Section \ref{sec:guidance}, we work out guidelines for choosing $W_t^U$ and $K$ for a given $\mudata$. In Section \ref{sec:guidance_wndm}, we study the case of image distributions and see that our theorems indeed apply for the typical properties of $\mudata$ one expects in that setting; hence, we have proven 
that the standard diffusion model algorithm
is well-defined for image distributions as $D \to \infty$. 
Moreover, we see that the choices $W_t^U = W_t^{L^2}$ and $K=L^2$ actually follow the guidelines developed in the preceding subsection. Therefore, the canonical choices made for diffusion models seem to be good default choices for image distributions. 

For $\mudata$ with different smoothness properties, however, the insights from our theory dictate other choices for $U$ and $K$. In Section \ref{sec:numerics} we apply our guidelines to two specific data distributions $\mudata$. 
Our principled algorithms are compared to the common ad hoc implementation of diffusion models. These numerical findings confirm our theoretical insights: our modifications outperform the canonical choices, and the ways in which they do can be explained by the discussion in Section \ref{sec:dms_in_infinite_dimensions}.

\subsection{Related Work}
The two efforts most related to ours are the concurrent works \citet{hagemann2023multilevel} and \citet{lim2023score}. 

In \citet{hagemann2023multilevel}, methods are developed to train diffusion models simultaneously on multiple discretization levels of (infinite-dimensional) functions. They build upon our Wasserstein distance bounds to show that their multilevel approach is consistent. 

\citet{lim2023score} are also able to generalize the trained model over multiple discretization levels. They propose to run the annealed Langevin algorithm in infinite dimensions, and use existing results for infinite-dimensional Langevin algorithms to justify their algorithms theoretically. The forward-reverse SDE framework is not treated. 

Both of these efforts encounter difficulties when defining the infinite-dimensional score. \citet{hagemann2023multilevel} circumvent this issue by only treating time-reversals of the \textit{discretized} forward SDE. \citet{lim2023score}, on the other hand, only analyze the case in which the measure is supported on the Cameron--Martin space of $W_t^U$. One can then simplify the problem by working with densities of $X_t$ with respect to Gaussian measure. 
From a practical point of view, both of these works employ Fourier neural operators \citep{li2020fourier} as their neural network architecture, while we work directly in the space domain and use the popular U-Net architecture for our neural networks. 

\jp{In \cite{kerrigan2022diffusion} an infinite-dimensional time-discrete version of the diffusion model algorithm is proposed. It is not studied whether the proposed algorithm is well defined in infinite dimensions.}

Other works also transform  data into a representation that is well suited to functions, e.g., by applying a wavelet \citep{guth2022wavelet, phung2022wavelet} or spectral  \citep{phillips2022spectral} transform.
After the transformation, however, these works employ the finite-dimensional formulation of the diffusion model algorithm; infinite-dimensional limits are not treated. We discuss how different spatial discretization schemes can be related to our results in Section \ref{sec:algorithm}.

Lastly, the subject of convergence of diffusion models to the target distribution has been a very active field of research recently; see \citet{chen2022improved, chen2023sampling, debortoli2022convergence, lee2022convergence, yang2022convergence}. In all these works, however, bounds on the distance to the target measure depend at least linearly on the discretization dimension $D$, rendering them vacuous in infinite dimensions.

\subsection{A Primer on Probability in Hilbert Spaces}\label{section gaussian measures}
In this section, we will give a short summary of key concepts relating to probability theory on infinite-dimensional (Hilbert) spaces which are required to study the infinite-dimensional formulation of SGMs rigorously. For an extensive introduction to this topic, see \citet[Chapter 3]{hairer2009spde}. 

\subsubsection{Gaussian measures on Hilbert spaces}

Let $(H, \langle \cdot, \cdot \rangle_H)$ be a separable Hilbert space. We then say that a random variable $X$ taking values in $H$ is Gaussian if, for every $v\in H$, the real-valued random variable $\langle v, X\rangle_H$ is also Gaussian. If the $\langle v, X\rangle_H$ have mean zero, $X$ is \textit{centered}. The covariance operator of $X$ is the symmetric, positive-definite operator $C:H\to H$ defined through
\begin{equation}
\langle g,Ch\rangle_H = \textrm{Cov} \left ( \langle X,g\rangle_H, \langle X,h\rangle_H \right ) = \mathbb E_X[\langle X,g\rangle_H \langle X,h\rangle_H].  
\label{equ:definition_C}
\end{equation}
We denote the law of $X$ in this case by $\Normal(0,C)$. Since $X$ takes values in $H$, $C$ is guaranteed to be compact \citep{hairer2009spde}. Therefore, there exists an orthonormal basis $(e_i:i\ge 1)$ of eigenvectors of $C$ satisfying $Ce_i = c_i e_i$. Fixing this basis, the second moment of
$X$ is given by
\[ \mathbb{E} [\| X \|_H^2] =\mathbb{E} \left[ \sum_{i = 1}^{\infty} \langle
X, e_i \rangle_H^2 \right] = \sum_{i = 1}^{\infty} \mathbb{E} [\langle X,
e_i \rangle^2_H] = \sum_{i = 1}^{\infty} \langle e_i, Ce_i \rangle_H =
\sum_{i = 1}^{\infty} c_i . \]
Since a Gaussian measure is supported on $H$ if and only if its second moment
on $H$ is finite \citep{hairer2009spde}, and $X$ takes values in $H$, the trace of $C$, 
$\tmop{tr} (C) = \sum_{i = 1}^{\infty} c_i$, will be finite. We then also say that $C$ is of \tmtextit{trace class}. Note that this is not the case if one would choose $C = \text{Id}$, since its trace is infinite. However, one could always just consider a larger space  $H' \supset H$, such that $H'$ supports $\mu \coloneqq \Normal(0,C)$ and on which $C$ would then have finite trace. 

\subsubsection{The Cameron--Martin space}\label{section cameron martin space}

The covariance operator $C$ plays a special role in that it characterizes the `shape' of the Gaussian measure $\Normal(0, C)$. Indeed, one may define another canonical inner product space $U$ associated to $C$, which is a (compactly embedded) subspace $U\subseteq H$ called the \textit{Cameron--Martin space} of $\Normal(0,C)$. Intuitively speaking, with respect to the geometry of $U$, a random variable $X\sim \mathcal N(0,C)$ will have `identity' covariance. Assuming that $C$ is non-degenerate, the Cameron--Martin space is defined via the inner product
\[ \langle g,h \rangle_U = \langle g, C^{-1} h \rangle_H =\langle C^{-1/2}g, C^{-1/2} h \rangle_H. \] Since $C^{-1}$ is unbounded, $U$ is indeed a smaller space than $H$; more specifically, one can show that $U=C^{1/2}H$.
In order to generate a realization of $X\sim \Normal(0,C)$, one may simply draw i.i.d.\ coefficients $(\xi_i \sim \Normal(0,1) : i\ge 1)$ and set $X= \sum_{i=1}^N c_i^{1/2}\xi_i e_i$ 
where $(c_i,e_i)_{i= 1}^\infty$ are the eigenpairs of $C$.\footnote{This is also called the Karhunen--Lo\`eve expansion of $X$, and in finite dimensions relates to the simple fact that $C^{-1/2}X\sim \Normal(0,\Id)$.}

It is important to note that $X$ almost surely does \textit{not} take values in $U$. As an example, let $H=L^2([0,1])$, and consider a one-dimensional Brownian motion process $(B_t:t\in [0,1])$. Of course, $B\in H$ almost surely. The Cameron--Martin space of $B$, however, is given as the space $U=H^1([0,1])$ of weakly differentiable functions on $[0,1]$. Since the sample paths of $B$ are almost surely nowhere differentiable \citep{karatzas1991brownian}, we conclude that almost surely $B\notin U$.\footnote{Here, we have used that functions in $H^1$ are absolutely continuous, and therefore almost everywhere differentiable on $[0,1]$.} Nevertheless, $U$ does indicate the regularity of the Gaussian process at hand: the more regular $U$, the more regular the draws from the corresponding Gaussian measure.

\subsubsection{$C$-Wiener processes in Hilbert spaces} The standard Brownian motion in $\mathbb{R}^D$ has increments $W_{t+\Delta t} - W_t \sim \mathcal{N}(0, \Delta t \text{I}_D)$. However, for the case of a general infinite-dimensional Hilbert space $H$,  the meaning of an identity covariance matrix depends on the choice of the scalar product with respect to which the Gaussian measure has identity covariance.
Therefore, we will from now on fix two Hilbert spaces: the Cameron--Martin space $U$, with respect to which the increments of the Wiener process would have covariance $\Delta t I$, and a larger space $H$ on which $W_t^U$ takes values and has covariance operator $C$, i.e.
\[
    W_{t+\Delta t}^U - W_t^U \sim \mathcal{N}(0, \Delta t C).
\]
In general, we will pick $H$ large enough so that all of our objects take values in it (the target measure $\mudata$ as well as the $C$-Wiener process $W_t^U$). The choice of $U$ can then also be seen as being equivalent to choosing a covariance operator $C$ of $W_t^U$ on $H$.

\subsubsection{Interpretation in finite dimensions}
Given a Gaussian distribution $\mathcal{N}(0, C)$ on $\mathbb{R}^D$, its Cameron--Martin space will be again $\mathbb{R}^D$, but equipped with the scalar product
\[
    \langle x, y \rangle_U = \langle C^{-1/2} x, C^{-1/2} y \rangle_{\mathbb{R}^D} = x^T C^{-1} y.
\]
Plugging $U$ into definition \eqref{equ:definition_C}, one sees that $X$ has an identity covariance matrix with respect to $U$. If $X \sim \mathcal{N}(0, C)$, then it can also be represented as $\sqrt{C} Z$, for $Z\sim \mathcal{N}(0, \text{I}_D)$. Similarly, a $C$-Wiener process with increments $\mathcal{N}(0, C)$ in finite dimensions can be constructed by using a standard Brownian motion $W_t$ on $\mathbb{R}^D$ and multiplying it with $\sqrt{C}$. 

Therefore, in finite dimensions, most of the discussions above can be simplified to choosing covariance matrices and representing objects of interest in terms of standard Gaussians ($Z$) or Brownian motions ($W_t$). The main technical difficulties in infinite dimensions arise because one has to choose a Hilbert space $H$ on which $Z$ would have the standard normal distribution, 
and because $Z$ will not take values in $H$. 

However, in infinite dimensions, one can still understand most concepts that relate to the choice of Gaussian measures by simply thinking about some large Hilbert space $H'$ in which all quantities of interest take values and then identifying Gaussian random variables with their covariance operators on this space.


\section{The Infinite-Dimensional Forward and Reverse SDEs} \label{sec:inf_dimensional_forward_backward_sde}
We will now formulate the forward and reverse SDEs of our generative model in infinite dimensions, and show that the reverse SDE is, in fact, well-posed with the correct terminal distribution. 

To this end, let $\mudata$ be our target measure, supported on a separable Hilbert space $(H,\langle\cdot,\cdot\rangle)$. Our goal is to generate samples from $\mudata$, which is done by first adding noise to given samples from $\mudata$ using a forward SDE and then generating new samples using a learned reverse SDE \citep{song2021scorebased}. 

\subsection{Forward SDE}\label{section densities in infinite dimensions}\label{sec:density_forward_sde}
We now define the infinite-dimensional forward SDE used to `diffuse' the initial measure $\mudata$. As noted in Section \ref{section gaussian measures}, there is no natural Brownian motion process in infinite dimensions; instead there is one white noise process $W_t^U$ for each Hilbert space $U$. From now on, we fix some Cameron--Martin space $U$, together with its Gaussian measure $\mathcal{N}(0, C)$. Furthermore, let $H$ be large enough to not only support $\mudata$, but also $\mathcal{N}(0, C)$. In practice, an example  would be to choose a Gaussian process (GP) with a Mat\'ern covariance $\mathcal{N}(0,C)$ (which implicitly defines $U$). As the embedding space $H$, one could for example choose $L^2$. Then $W_t^U$ would have increments that are samples from a Mat\'ern GP.

We then define the forward SDE as
\begin{equation}
	\mathd X_t = - \frac{1}{2} X_t \mathd t + \mathd W_t^U = - \frac{1}{2} X_t \mathd t + \sqrt{C} \mathd W_t^H, \qquad X_0 \sim \mudata \, .  \label{inf-fwd}
\end{equation}
The marginal distributions of $X_t$ will converge to the stationary distribution $\Normal (0, C)$ as $t \to \infty$ \citep[Theorem 11.11]{da2014stochastic}. We will denote the marginal distributions of $X_t$ by $\mathbb{P}_t$.  

The choice of $U$, or equivalently $C$, can be guided by the theory that we will develop and strongly impacts empirical performance. We discuss these choices in Section \ref{sec:dms_in_infinite_dimensions}.

\subsection{Definition of the Score Function}\label{section score}
Analogously to score-based generative models in finite dimensions, we now wish to define the reverse SDE corresponding to \eqref{inf-fwd}; this SDE on $H$ should approximately transform $\Normal(0,C)$ to $\mudata$. This can be achieved by time-reversing the SDE \eqref{inf-fwd}. In the finite-dimensional case, the drift of the time reversal SDE involves the score function $\nabla \log p_t$ (see \eqref{fd_reverse_sde}), where $p_t$ is the density of $\mathbb{P}_t$ with respect to Lebesgue measure. More precisely, in the finite-dimensional case $H=\R^D$, the reverse SDE to the Ornstein--Uhlenbeck process
\begin{equation*}
	\mathd X_t = - \frac{1}{2} X_t \mathd t + \sqrt{C}\mathd W_t
\end{equation*}
is given by
\begin{equation*}
	\mathd Y_t = \frac{1}{2} Y_t \mathd t + C \nabla \log p_{T-t}(Y_t)\mathd t + \sqrt{C}\mathd W_t \, ;
\end{equation*}
see \citep{haussmann1986time}. In the infinite-dimensional case, the density $p_t$ is no longer well-defined, since there is no Lebesgue measure. Hence, we need another way to make sense of the score function. Interestingly, in finite dimensions, there is an alternative way to express $C \nabla_H \log p_t$ via conditional expectations which is amenable to generalization to infinite dimensions. 
\begin{lemma}
	\label{lemma score rewrite} Assume the finite-dimensional setting $H=\R^D$. Denote by $p_t$ the Lebesgue density of $X_t$, where $X_{[0, T]}$ is a solution to \eqref{inf-fwd}. Then, we can express the function $C \nabla \log p_t$ as
	\begin{align*}
		C \nabla \log p_t (x) & = -\frac{1}{1 - e^{- t}} \left(
		\mathbb{E} \left[ X_t - e^{- \frac{t}{2}} X_0  \mid  X_t = x \right] \right)\\
		&=
		- \frac{1}{1 - e^{- t}} \left( x - e^{- \frac{t}{2}} \mathbb{E} [X_0
		\mid X_t = x] \right)
	\end{align*}
	for $t > 0$, where $\mathbb{E} [f(X_\tau) \mid X_t = x]$ is the conditional expectation of the function $f(X_\tau)$ given $X_t = x$ and $\tau \in [0,T]$.
\end{lemma}
Conditional expectations are also \emph{well-defined in infinite dimensions}. Therefore, we will give the conditional expectation from Lemma \ref{lemma score rewrite} a name and make use of it as the drift of the reverse SDE on Hilbert space $H$:
\begin{definition}\label{def:reverse_drifts}
    Let $H$ be a possibly infinite-dimensional Hilbert space and $X_{[0, T]}$ a solution to \eqref{inf-fwd}. We define the reverse drift as a map $s: [0, T] \times H \to H$, 
	\begin{equation*}
		s(t, x) \coloneqq -\frac{1}{1 - e^{- t}}
		\left( x - e^{- \frac{t}{2}} \mathbb{E} [X_0 |X_t = x] \right).
	\end{equation*}
	
\end{definition}
\jp{
\begin{remark}
    For a definition of conditional expectations and measures for Hilbert-space valued random variables, see \citet[Section 1.3]{bogachev1997differentiable}.
\end{remark}
}
\begin{remark}
	Note that the above function is only defined up to $\mathbb{P}_t$-equivalence classes, where $\mathbb{P}_t$ is the distribution of the time-$t$ marginal of the forward SDE. However, the loss function for diffusion models is a $L^2$ loss, integrated over $\mathbb{P}_t$. Therefore, without restricting the function class that one optimizes over, the minimizer is also only defined up to $\mathbb{P}_t$-equivalence. Neural networks are normally contained in the class of continuous functions in $t$ and $x$. We will see that we can pick versions of $s(t,x)$ satisfying continuity properties, for example being locally Lipschitz continuous in $x$ (see Section \ref{sec:uniqueness}).
\end{remark}

We will also frequently use the fact that the drift of the reverse SDE is actually a rescaled martingale in reverse time. We will later show that this also holds in infinite dimensions, in Theorem \ref{theorem SDEs well defined}.

\begin{lemma}\label{lemma:reverse_martingale}
	\label{score is martingale}Assume the finite-dimensional setting $H=\R^D$. Then, the quantity $M_t = e^{- t/2} \nabla \log p_t (X_t)$ is a time-continuous	reverse time martingale, i.e.,
	\[ \nabla \log p_t (X_t) = e^{\frac{(t - \tau)}{2}} \mathbb{E} [\nabla \log p_\tau
	(X_\tau) \mid X_t] ~~~~  \text{for all}~ 0< \tau <t. \]
\end{lemma}
The proofs of both of these lemmas can be found in Appendix \ref{sec:proofs_score}. 
%

\subsection{Reverse SDE}\label{section the sdes}

We are now able to write down the infinite-dimensional forward SDE:
\begin{alignat}{4}
	&X_0 &&\sim \mudata, \qquad &&\mathd X_t && =  - \frac{1}{2} X_t \mathd t + \sqrt{C}\mathd W_t^H, \label{forward SDE inf}
\end{alignat}
where $W_t^U = \sqrt{C} W_t^H$ are $C$-Wiener processes. \jpp{Defining $Y_{s} := X_{T-t}$ as the time-reversal of a solution to \eqref{forward SDE inf}, we want to show that it satisfies the following stochastic differential equation:}
\begin{equation}
	Y_0 \sim \mathbb{P}_T, \mathd Y_t =  \frac{1}{2} Y_t \mathd t + s(T-t, Y_t) \mathd t + \sqrt{C}\mathd W_t^H.
	\label{reverse SDE inf abs}
\end{equation}
Here, the drift $s(t,x)$ of the reverse SDE is given by Definition \ref{def:reverse_drifts}. 

In finite dimensions, one could also rewrite the reverse SDE as
\begin{equation}
	\begin{aligned}
		\mathd Y_t &= \frac{1}{2}Y_t \mathd t + C \nabla \log p_{T-t}(X_t) \mathd t + \sqrt{C} \mathd W_t^H \\
		&= \frac{1}{2}Y_t \mathd t + C \nabla \log \frac{\mathd p_{T-t}}{\mathd \mathcal{N}(0, C)}(X_t) \mathd t + C \nabla \log \mathcal{N}(0, C)(X_t) + \sqrt{C} \mathd W_t^H \\
		&= -\frac{1}{2}Y_t \mathd t + C \nabla \log \frac{\mathd p_{T-t}}{\mathd \mathcal{N}(0, C)}(X_t) \mathd t + \sqrt{C} \mathd W_t^H, \\
	\end{aligned}\label{equ:rewriting_reverse_sde}
\end{equation}
where we denote by $\mathcal{N}(0, C)(x)$ the density of $\mathcal{N}(0, C)$ of evaluated at $x$.
In finite as well as in infinite dimensions, if $X_0 \sim \mudata$ has a density with respect to a Gaussian $\mathcal{N}(0, C)$, then so will the distribution of $X_t$ (see the proof of Theorem \ref{thm:uniqueness_gaussian}). Therefore, under that assumption, the SDE in the last line of \eqref{equ:rewriting_reverse_sde} can also be made sense of in infinite dimensions. Rewriting the SDE in this form is helpful in the proof of Theorem \ref{thm:uniqueness_gaussian}.
\begin{remark}
Another forward SDE with invariant measure $\nu =\Normal (0, C)$ is
	\begin{eqnarray}
		\mathd X_t & = & - \frac{1}{2} C^{- 1} X_t \mathd t + \mathd W_t^H,\label{equ:forward_spde}
	\end{eqnarray}
	with corresponding reverse SDE
	\begin{eqnarray}\label{equ:reverse_SDE_relative}
		\mathd Y_t & = & - \frac{1}{2} C^{- 1} Y_t \mathd t + \nabla_H \log \frac{\mathd \mathbb{P}_{T -
				t}}{\mathd \mathcal{N}(0, C)} (Y_t) \mathd t + \mathd W_t^H .  \label{reverse SDE SPDE}
	\end{eqnarray} 
	The operator $C^{- 1}$ can often
	be identified with a differential operator, turning \eqref{reverse SDE SPDE}
	into a \emph{stochastic partial differential equation} (SPDE). One can
	then use numerical tools for SPDEs to approximate the above. This constitutes an interesting direction for future work. 

    \jp{Note, however, that if $C^{-1}$ is an unbounded operator, for any fixed positive time $t > 0$, the high frequencies of $Y_t$ will already have been smoothed out by the process. Since the reverse SDE has to be discretized, care must be taken on how to train and evaluate the diffusion model; otherwise one might lose all high-frequency information. By `high frequencies,' here we mean the eigenvectors corresponding to large eigenvalues of $C^{-1}$.}
\end{remark}

\subsection{Training Loss}\label{section loss}

To simulate the reverse SDE, we need a way to approximately learn the drift function $s(t,x)$. For another function $\tilde{s} (t, x)$ (a candidate approximation to $s$), we measure the goodness of the fit
of $\tilde{s}$ using a score-matching objective, i.e.,
\begin{equation}
	\tmop{SM}_t (\tilde{s}) =\mathbb{E} [\| s (t, X_t) - \tilde{s} (t, X_t) \|^2_K].
	 \label{equ:score-matching-objective}
\end{equation}
On $\mathbb{R}^D$, the norm $\|\cdot\|_K$ to measure the misfit is typically the Euclidean norm. 
For training, this loss can be rewritten into the denoising score matching
objective,
\begin{equation}
	\tmop{DSM}_t (\tilde{s}) =\mathbb{E} [\| \tilde{s} (t, X_t) - (1 - e^{-t})^{-1/2}
	(X_t - e^{- t / 2} X_0) \|^2_K] = \tmop{SM}_t(\tilde{s}) + V_t.
	\label{equ:denoising-score-matching-objective}
\end{equation}
One can then show (see \citep{vincent2011connection}), that $\tmop{SM}$ and $\tmop{DSM}$ only
differ by a constant $V_t$ and therefore one can use $\tmop{DSM}$ as an optimization
objective to optimize SM. The DSM loss is normally optimized on a sequence of times $\{t_m\}_{m=1}^M$ on which the reverse SDE is discretized (since the score will only be evaluated at these $t_i$ values), i.e.,
\begin{equation}
\begin{aligned}
	\text{Loss}(\tilde{s}) 
	~=~\sum_{m=1}^M \text{SM}_{t_m} (\tilde{s})
	~=&~\sum_{m=1}^M \text{DSM}_{t_m} (\tilde{s}) - V_{t_m}
    \\
	=&~\mathbb{E}_{t_m, X}[\| \tilde{s} ({t_m}, X_{t_m}) - \sigma_{t_m}^{- 1} (X_{t_m} - e^{- {t_m} / 2} X_0) \|^2_K] - V,
	\label{equ:full_loss}
    \end{aligned}
\end{equation}
where the last expectation is taken over $t_m \in \text{Unif}(\{t_1, \ldots, t_M\})$ and $V = \sum_{m=1}^M V_{t_m}$.

We will see that the equivalence of $\tmop{SM}_t$ and $\tmop{DSM}_t$ does not
hold in general in infinite dimensions. Furthermore, we will study the choice of the norm $\|\cdot\|_K$. Two natural choices that come to mind are the norm of the
embedding Hilbert space $H$ and of the Cameron--Martin space $U$ of
$C$. In the following lemma, we study conditions under which we can rewrite the loss
into the denoising score matching objective.
\begin{lemma}
	\label{lemma score matching}Let $(K, \langle \cdot, \cdot \rangle_K)$ be a
	separable Hilbert space. Furthermore, denote by $\tilde{s}$ an approximation
	to $s$, such that the score matching objective \eqref{equ:score-matching-objective} is finite. Then,
	\[ \tmop{SM}_t (\tilde{s}) = \tmop{DSM}_t (\tilde{s}) - V_t, \]
	where $\tmop{DSM}_t$ is defined in
	\eqref{equ:denoising-score-matching-objective} and $V_t$ is given by the conditional variance of $X_0$,
	\[ V_t = \frac{e^{- t}}{1 - e^{- t}} \mathbb{E} [\| X_0 -\mathbb{E} [X_0
	|X_t] \|^2_K]. \]
	Furthermore, $\tmop{DSM}_t$ is infinite if $V_t$ is.
\end{lemma}
Lemma \ref{lemma score matching} shows that in infinite dimensions there is the possibility that the true objective $\tmop{SM}$, which we are trying to optimize,
might be finite, while DSM is not. One might argue that this is not relevant
since in practice one always has to discretize and then both will be finite.
However, as we will argue in the following paragraph, the discretization level will impact the variance of the gradients.
In practice, we do not evaluate the full expectation values in $\tmop{SM}$ or
$\tmop{DSM}$, but take Monte Carlo estimates in the form of mini-batches. Assuming that we have already reached the optimum, i.e., $s = \tilde{s}$, then the SM objective would be zero and also the gradient of any mini-batch taken to approximate it would be zero. However, derivatives of Monte Carlo
estimates of the DSM objective will have the form
\[ \partial_{\theta_i} \tmop{DSM} (\tilde{s}_\theta) = \frac{1}{M} \sum_{i = 1}^M
\langle \partial_{\theta_i} \tilde{s} (t, x_t^M), \tilde{s} (t, X_t) -
\sigma_t^{- 1} (X_t - e^{- t / 2} X_0) \rangle,  \]
where me made the parameters $\theta$ (typically, the weights of a neural network) of $\tilde{s}_\theta$ explicit.
The random variable $\tilde{s} (t, X_t) - \sigma_t^{- 1} (X_t - e^{- t /
	2} X_0)$ has infinite variance, and therefore we can expect the above gradient
estimates to have infinite variance too. Hence, if $V_t$ is not finite and
therefore $\tmop{DSM}_t$ is not finite in infinite dimensions, one can expect
variance of the the gradient of the DSM to get arbitrarily large as the discretization
gets finer, despite the fact that the true gradient should be zero.
In the following lemma,
we study some cases in which we can expect $V$ to be finite.

\begin{lemma}\label{lemma:special_cases_score_matching}
	The denoising score matching objective \eqref{equ:denoising-score-matching-objective} is finite in infinite dimensions if
	one of the following two conditions holds:
	\begin{enumerate}
		\item We use the Cameron--Martin norm $\| \cdot \|_K = \| \cdot \|_U$ in the objective, and $\mu_{\tmop{data}}$ is supported on the
		Cameron--Martin space $U$ of $\mathcal{N} (0, C)$ and has finite second
		moment, i.e.,
		\[ \mathbb{E} [\| X_0 -\mathbb{E} [X_0] \|^2_U] < \infty . \]
		\item Both $\mudata$ and $\mathcal{N} (0, C)$ are supported on $K$.
	\end{enumerate}
\end{lemma}
A consequence of point 2 of Lemma~\ref{lemma:special_cases_score_matching} is that the norm of the embedding Hilbert space $K = H$ is always a valid choice. The proof of both lemmas above can be found in Appendix \ref{sec:proofs_loss}. 


\section{Well-Posedness of the Reverse SDE}
\label{sec:wellposedness}
We need to show that the reverse SDE possesses solutions and that they are unique in order to prove that the reverse SDE samples from the target distribution in infinite dimensions. In Section \ref{sec:sde_existence}, we show that the time-reversal of the forward SDE satisfies the reverse SDE in infinite dimensions and therefore samples the right final distribution $\mudata$ at its final time. In Section \ref{sec:uniqueness}, we will show strong uniqueness and existence of the reverse SDE for general initial conditions.


\subsection{The Time Reversal Satisfies the Reverse SDE}\label{sec:sde_existence}
Thus far, we have \textit{formally} formulated the reverse SDE \eqref{reverse SDE inf abs} without showing that it actually constitutes a time reversal of the stochastic dynamics from the forward equation. In the following theorem, we show that $Y_t$ indeed constitutes a time reversal of $X_t$ and that it recovers the correct target distribution at terminal time $T$.
\begin{theorem}
	\label{theorem SDEs well defined} Assume $X_t$ is a solution to \eqref{inf-fwd}. 
	Then, the time reversal $Y_t \coloneqq X_{T-t}$ solves the SDE \eqref{reverse SDE inf abs}. Furthermore, if $H$ is a Hilbert space such that $\mudata$ and $\mathcal{N}(0, C)$ are both supported on $H$, we can choose $s$ such that $M_t = s(t, X_t)$ is almost surely continuous in $t$ with respect to the $H$-norm.
\end{theorem}
\begin{proof}(Sketch)
	\jp{We approximate the infinite-dimensional forward-SDE in finite dimensions using a spectral approximation in the eigenbasis of the covariance operator $C$. The finite-dimensional approximations are denoted by $X_t^D$.}

    \jp{Next we show that the finite-dimensional time-reversals $Y_t^D := X_{T-t}^D$ satisfy an equation analogous to \eqref{reverse SDE inf abs}:
	\[ Y_t^D - Y_0^D  - \frac{1}{2} \int_0^t Y_r^D \mathd r - \int_0^t s_{T-r}^D \mathd r = \sqrt{C^D} B_t^D. \]
    We then show that all of those terms converge to their counterparts in \eqref{reverse SDE inf abs}, and therefore $Y_t := X_{T-t}$ satisfies the same equation. The convergence of $Y_t^D$ to $Y_t$ is trivial: the $Y_t^D$ are spectral approximations. The convergence of the other terms is a bit more involved. Unfortunately, for $L > D$, the conditional expectations $s_t^D$ are not the projections of $s_t^L$ to a lower-dimensional space, and the same holds for the Brownian motions $B_t^D$. }
    
    \jp{We can, however, show that $s_t^D$ is a martingale in $D$. Combining this with the fact that $e^{-t/2} s_t^D$ is also a martingale in time (see Lemma \ref{lemma:reverse_martingale}), we obtain uniform-in-time convergence of $s_t^D$ to $s_t$. The convergence of $C^D B_t^D$ also follows, and we can identify the limit as a $C$-Wiener process. }

	The full proof can be found in Appendix \ref{section:proof_theorem_well_defined}.
\end{proof}
\begin{remark}
The work \citet{follmer1986time} studies time-reversal of more general forward SDEs. Due to the more general setting, the resulting SDE is only expressed coordinate-wise, and the SDE as well as the assumptions are more technical. Using our approach and the reverse drift $s(t, x)$, we prove that we can still use the common denoising score matching loss to approximate $s(t,x)$; see Lemma \ref{lemma score matching}. Another related concept is vector logarithmic derivatives, as discussed in \citet{bogachev1997differentiable}.
\end{remark}

Due to Theorem \ref{theorem SDEs well defined}, we know that there is a solution to \eqref{reverse SDE inf abs} that will sample $\mudata$ at the final time. To motivate approximating \eqref{reverse SDE inf abs} for sampling from $\mudata$, we also need to show that these solutions are unique; otherwise there could be other solutions that have different terminal conditions. We will achieve this in the following section.

\subsection{Uniqueness and Existence of Solutions}\label{sec:uniqueness}
\jp{We now study strong uniqueness and existence of the solutions to the reverse SDE. We say an SDE satisfies \emph{strong existence} if we can construct a solution to the SDE for any driving Brownian motion and that solution will be adapted to the filtration of the Brownian motion. 
We say that an SDE satisfies \emph{strong uniqueness} if, for any two solutions $Y_t$ and $\tilde{Y}_t$ of that SDE, with the same driving Brownian motion, it holds that $\mathbb{P}[Y_t = \tilde{Y}_t \text{ for all } t] = 1$.} 

\jp{
\begin{remark}
Here we will prove strong uniqueness (also called pathwise uniqueness) of solutions to the reverse SDE. For sampling purposes, uniqueness in law of the reverse SDE would suffice and is generally easier to prove. 
However, for the Wasserstein distance bounds which we will prove later (see Theorem \ref{theorem distance}) we will employ coupling arguments. These arguments implicitly rely on strong existence of solutions to the reverse SDE and therefore we will prove strong existence. Strong existence together with uniqueness in law already imply strong uniqueness; see \citet[Section 5.3]{karatzas1991brownian} (the result also holds here since $H$ is separable).
Therefore, in our case we can obtain strong uniqueness no matter which uniqueness we prove.
\end{remark}
}

\jp{We will treat two different settings. The first setting is tailored to distributions supported on substructures of the full space. The main motivation for this setting are measures which are supported on a manifold-like structure $\mathcal{M}$.} Since many distributions that diffusion models are applied to satisfy the manifold hypothesis, understanding how diffusion models interact with manifolds has been an active area of research \citep{pidstrigach2022score, de2022convergence, batzolis2022your}. 
\begin{theorem}
	Fix a covariance operator $C$ in the forward SDE \eqref{forward SDE inf} together with its Cameron--Martin space $U$. 
	Assume that the support of $\mudata$ is contained in a ball $B_R$ in $U$ of radius $R \ge 0$:
	\[
	B_R = \{x~:~\|x\|_U \le R\}.
	\]
	Then there is a version of $s$ which is Lipschitz continuous with respect to the Cameron--Martin norm, i.e.,
	\begin{equation}
		\|s(t, x) - s(t, y)\|_U \le L_t \|x - y\|_U,
		\label{equ:local_lipschitz_U}
	\end{equation}
	where $L_t \in \mathbb{R}^+$ is a time-dependent Lipschitz constant. Moreover, the reverse SDE with the Lipschitz continuous version of $s(t,x)$ has a unique strong solution.
	\label{thm:uniqueness_manifold}
\end{theorem}
\begin{proof}(Sketch)
	The transition kernel of the forward SDE is given by
	\[
            p_t(x_0, \cdot) \sim \mathcal{N}(e^{-t}x_0, v_t C),
	\]
	where we used the shorthand notation $v_t = 1 - e^{-t}$.
	If $x_0$ is an element of the Cameron--Martin space $U$ of $C$, then the transition kernel is absolutely continuous with respect to $\mathcal{N}(0, (1 - e^{-t})C)$. The explicit formula for the density is 
	\[
	n_t(x_0, x) 
	= \frac{\mathd \mathcal{N}(e^{-t}x_0, v_t C)}{\mathd \mathcal{N}(0, v_t C)}(x)
	\]
	by the Cameron--Martin theorem. Since $\mudata$ almost surely takes values in $U$, one can use the above to derive an explicit expression for the conditional expectation $\mathbb{E}[X_0 | X_t = x]$ in terms of these densities: 
	\[
	   \mathbb{E}[X_0 | X_t = x] = \frac{\int x_0 n_t(x_0, x) \mathrm{d}\mudata(x_0)}
	{\int n_t(x_0, x) \mathrm{d}\mudata(x_0)}.
	\]
	This formula can be used to derive local Lipschitzness of
	\begin{equation*}
		s(t, x) = -\frac{1}{1 - e^{- t}} x + \frac{e^{-\frac{t}{2}}}{1 - e^{-t}} \mathbb{E} [X_0 |X_t = x].
	\end{equation*}	
Interestingly, the local Lipschitzness is in terms of the norm of $U$. Even if $x$ and $y$ themselves are not in $U$, if their difference in is $U$, the $U$-norm of the difference $s(t, x) - s(t, y)$ will be bounded by \eqref{equ:local_lipschitz_U}.  
Taking some care, one can still use a fixed point argument to obtain existence, but not uniqueness. One can then apply Grönwall's lemma to obtain uniqueness.
	
	Note that obtaining \emph{weak} uniqueness would be easier, since under our assumptions the drift $s(t, x)$ will always map to the Cameron--Martin space of the $C$-Wiener process and one could apply a Girsanov-type argument.
	
	The full proof can be found in Appendix \ref{sec:uniqueness_manifold}.
\end{proof}

The other case of interest is applying diffusion models to Bayesian inverse problems or simulation-based inference. In this case, we assume that the true measure is given as a density with respect to a Gaussian reference measure. We treat it in the theorem below:
\begin{theorem}
	Fix a covariance operator $C$ in the forward SDE \eqref{forward SDE inf}. Assume $\mu_{\tmop{data}}$ is given as
	\[ \mu_{\tmop{data}} \propto \exp (- \Phi (x)) \mathd \mathcal{N} (0,
	C_{\mu}) . \]
	Let $(H, \langle \cdot,\cdot \rangle_H )$ be a Hilbert space on which $\mathcal{N}(0, C_\mu)$ is supported and $C$ is bounded.
	For the potential $\Phi \in C^1(H)$ we assume,
	\begin{itemize}
		\item $\Phi (x) \geqslant E_0$,
		
		\item $\Phi (x) \leqslant E_1 + E_2 \| x \|^2$, and
		
		\item $\| \nabla \Phi (x) - \nabla \Phi (y) \| \leqslant L \| x - y \|$,
	\end{itemize}
	where the gradient is the $H$-gradient. Then there is a version of $s(t,x)$ that is locally Lipschitz continuous with respect to the $H$-norm for each $t$, i.e., for $\|x\|, \|y\| \le r$ there is a $L_{t, r} < \infty$ such that
	\[
	\|s(t, x) - s(t, y)\| \le L_{t, r} \|x - y\|,
	\]
and the reverse SDE with the locally Lipschitz continuous version of $s(t,x)$ has a unique strong solution.
	\label{thm:uniqueness_gaussian}
\end{theorem}
\begin{proof}(Sketch)
	The proof holds for any $C$ that is diagonalizable with respect to the same eigenbasis as $C_\mu$ (in particular, also for $C = \text{Id}$, i.e., $H$-white noise), but we will only treat the less technical case $C = C_\mu$ here.
	
	In the case of $C = C_\mu$, the distribution $\mathbb{P}_t$ of $X_t$ is absolutely continuous with respect to $\mathcal{N}(0, C)$. One can rewrite the reverse SDE as in \eqref{equ:rewriting_reverse_sde}. 
	It will then hold that
	\[
	\nabla_{x_t} \log \frac{\mathd p_t}{\mathd \mathcal{N} (0, C)} (x_t) = \mathbb{E}[C \nabla \Phi(X_0) | X_t = x_t],
	\]
	and the proof will mainly translate the Lipschitzness properties of $\nabla \Phi(x)$ to $\mathbb{E}[\nabla \Phi(X_0) | X_t = x]$. The global Lipschitzness of $\nabla \Phi(x)$ only induces local Lipschitzness of $\mathbb{E}[\nabla \Phi(X_0) | X_t = x]$, but that is enough to apply a Grönwall argument and deduce strong uniqueness. 

    Furthermore, one can obtain weak existence to the reverse SDE. By Theorem \ref{theorem SDEs well defined}, the time reversal will be a weak solution with initial condition $\mathbb{P}_T$. However, under the assumptions of the theorem, $\mathcal{N}(0, C)$ will be absolutely continuous with respect to $p_T$. Therefore, one can obtain a weak solution with initial conditions $\mathcal{N}(0, C)$ by reweighting the time reversal. However, weak existence together with strong uniqueness already imply strong uniqueness; see \citet[Section 5.3]{karatzas1991brownian}.
    
    The full proof can be found in Appendix \ref{sec:uniqueness_gaussian}.
\end{proof}

\section{Algorithms and Discretizations}\label{sec:algorithm}
We state simplified versions of our proposed algorithms in Algorithms \ref{alg:training} and \ref{alg:sampling}. There are many potential modifications one might make to the above algorithms, as for example discussed in \citet{song2021scorebased, song2020improved, ho2020denoising}; we do not include these here since they are not the focus of the current work. To implement any algorithm on a computer, the functions have to be discretized in some way. Discretization also interacts with the covariance matrix $C$, as the same covariance matrix has different meanings in different discretizations. We discuss this briefly now.

If the functions are discretized on a grid, i.e., if the samples are of the form $\{f(x_d)\}_{d=1}^D$ for a fixed grid $\{x_d\}$, choosing an identity covariance matrix corresponds to adding independent noise at each grid point $x_d$. The limiting object of the noise as the grid gets finer is space-time white noise (recall Section~\ref{sec:finite_dim_diff_models}). Furthermore, the Euclidean norm on $\mathbb{R}^D$ in the loss function \eqref{equ:denoising-score-matching-objective} will correspond to using the $L^2$ loss in the limit---i.e., the Cameron--Martin norm of the noising process.

In $\mathbb{R}^D$ the choice of the white noise process is equivalent to choosing a covariance matrix $C$ and adding $\sqrt{C} \mathd W_t$ with a standard $\mathbb{R}^D$-valued Brownian motion $W_t$. Any correlated Wiener noise process $W_t^U$ can be represented in this way on $\mathbb{R}^D$. 
If one wants the limit of $\sqrt{C} \mathd W_t$ to be a Gaussian process, one needs to plug in for $C$ the kernel matrix of that Gaussian process on the grid $\{x_d\}_{d=1}^D$. 
Alternatively, one can also use one of many available libraries to generate Gaussian process realizations for common kernels (such as Mat\'ern or squared exponential).

Note that the meaning of $C$ depends on the discretization. If $f$ is discretized with respect to some basis $e_i$ of a space $U$, then using the identity covariance matrix corresponds to using the white noise process with Cameron--Martin space $U$. Therefore, discretizing the functions in a wavelet or Fourier basis will also result in space-time white noise as these both form an orthonormal basis of $L^2$ (under the common scaling of the basis vectors). However, if one does not want to work in the spatial domain, one can also just discretize the functions in an orthonormal basis of the Cameron--Martin space $U$ of the noise one is targeting. Therefore, we can translate the approaches in \citet{guth2022wavelet,phillips2022spectral,phung2022wavelet} into our setting. 

\begin{figure}
	\begin{minipage}{0.48\textwidth}
		\begin{algorithm}[H]
			\centering
			\caption{Training}\label{alg:training}
			\begin{algorithmic}[1]
				\Require Covariance operator $C$
				\Require Training data $\{X^n\}_{n=1}^N$
				\Require Loss Norm $\|\cdot\|$
				\Require Batch size $B$
                    \Require Discretization grid $\{t_1, \ldots, t_M\}$
				\While{Metrics not good enough}
				\State Sample $\{\xi^i\}_{i=1}^B \sim \mathcal{N}(0, C)$ i.i.d.
				\State Subsample $\{x^i_0\}_{i=1}^B$ from $\{X^n\}_{n=1}^N$
				\State Sample $t^i \in \text{Unif}(\{t_1, \ldots, t_M\})$
				\State $x_t^i \gets e^{-t^i} x_0^i + \sqrt{1 - e^{-t^i}} \xi^i$
				\State Loss$(\theta)$ = $\sum_{i=1}^B \|\tilde{s}_\theta(t^i, x_t^i) - \frac{1}{\sqrt{1 - e^{-t^i}}} \xi^i\|^2$
				\State Perform gradient step on Loss.
				\EndWhile
			\end{algorithmic}
		\end{algorithm}
	\end{minipage}
	\hfill
	\begin{minipage}{0.48\textwidth}
		\begin{algorithm}[H]
			\centering
			\caption{Sampling}\label{alg:sampling}
			\begin{algorithmic}[1]
				\Require Covariance operator $C$
				\Require Discretization grid $\{t_1, \ldots, t_M\}$
				\Require Number of samples to generate $L$
				\State $\{x_M^i\}_{i=1}^L \sim \mathcal{N}(0, C)$
				\For{$m \gets M, \ldots, 1$}
				\State $\Delta t \gets t_m - t_{m-1}$
				\State Sample $\{\xi^i\}_{i=1}^L \sim \mathcal{N}(0, C)$ i.i.d.
				\State $x_{m-1}^i \gets x_m^i + \Delta t ~ \tilde{s}_\theta(t_m, x_m^i) + \sqrt{\Delta t} ~ \xi^i$
				\EndFor
				\State \textbf{return }{$\{x_M^i\}_{i=1}^M$}
			\end{algorithmic}
		\end{algorithm}
	\end{minipage}
\end{figure}

\section{Bounding the Distance to the Target Measure}\label{sec:distance_target_measure}

We now study how far the samples generated by the diffusion model algorithm lie from the true target measure $\mudata$. We do this in the Wasserstein-2-distance,
\[ \mathcal{W}_2 (\mu, \nu) = \left(\inf_{\kappa \in Q (\mu, \nu)} \int \| x - y
\|_H^2 \mathd \kappa (x, y)\right)^{1/2},  \]
where $\kappa$ runs over all measures on $H \times H$ which have marginals $\mu$ and
$\nu$. The Wasserstein-2 distance in some sense ``lifts'' the distance induced
by $\| \cdot \|_H$ to the space of measures. In the following
theorem, we give an upper bound for the Wasserstein distance between the sample
measure and the true data-generating measure. The bound holds
irrespective of $\| \cdot \|_H$, giving us the freedom to study how different
choices of $\| \cdot \|_H$ affect the distance bound. 

\begin{theorem}
	\label{theorem distance} We denote the covariance of the forward noising process by $C$. Let $(H, \|\cdot\|_H)$ be any Hilbert space such that the support of $\mudata$ and $\mathcal{N}(0, C)$ are contained in $H$. Assume that $\|\cdot\|_H$ is at least as strong as the norm $\|\cdot\|_K$ used in the training of the diffusion model (see \eqref{equ:full_loss}), i.e., 
	\[
		\|x\|_H \le a \|x\|_K
	\]
	for some constant $a$.
	Further, assume that $s(t, x)$ is Lipschitz on $H$ with constant $L$, i.e.,
	\[
	\|s(t, x) - s(t, y)\|_H \le L \|x - y\|_H
	\]
    and that the reverse SDE has a strong solution (see Theorem \ref{thm:uniqueness_manifold} or \ref{thm:uniqueness_gaussian} for the requirements).
	Let the reverse SDE \eqref{reverse SDE inf abs} be discretized using an exponential integrator (see Appendix \ref{sec:exponential_integrator}).  Then,
	\begin{equation}
		{\mathcal{W}_2}  (\mu_{\tmop{data}}, \mu_{\tmop{sample}})  \leqslant 
		\left( \exp(-T/2)~{\mathcal{W}_2}  (\mudata, \Normal (0, C))  +
		\varepsilon_{\tmop{Num}}^{1/2} + a \varepsilon_{\tmop{Loss}}^{1/2} \right) \exp \left(
		\frac{1}{4} L^2 T \right),
        \label{equ:w2_distance_bound}
	\end{equation}
	where $\varepsilon_{\tmop{Loss}}$ is the value of the loss objective \eqref{equ:full_loss} and $\varepsilon_{\tmop{Num}}$ denotes the error due to the numerical
	integration procedure,
	\begin{eqnarray*}
		\varepsilon_{\tmop{Num}} & = & O (\Delta t) \sup_{0 < t \leq T} \mathbb{E}_{X_t \sim p_t}[\|s(t, X_t)\|_H^2]. 
	\end{eqnarray*}
\end{theorem}

\begin{proof}(Sketch)
	We define two strong SDE solutions: $Y_t$, which is a solution to \eqref{reverse SDE inf abs} with the correct drift $s(t,x)$ and started in
	$\mathbb{P}_T$; and $\tilde{Y}_t$, which uses the approximate drift
	$\tilde{s}$ and is started in $\Normal (0, I)$. 
	
	Both solutions are run to time $T$. We couple them by using the same Brownian motion process for both and starting them in $\mathcal{W}_2$-optimally coupled initial conditions.
	
	We then obtain a bound on $\mathbb{E}[\|Y_T - \tilde{Y}_T\|_H^2]$. Since we know that $Y_T \sim \mudata$, $\tilde{Y}_T \sim \mu_{\tmop{sample}}$ by definition, this gives us a coupling between $\mudata$ and $\mu_\text{sample}$ and therefore upper bounds the Wasserstein-2 distance between those two.
	
	We make use of the fact that the score is a martingale to obtain an upper bound for the numerical integration error, depending only on the quantity $\sup_{0 < t \leq T} \mathbb{E}_{X_t \sim p_t}[\|s(t, X_t)\|_H^2]$.
	
	The full proof can be found in Appendix \ref{sec:distance_proof}. 
\end{proof}

Since the choice of the embedding space $(H, \|\cdot\|_H)$ is left open in Theorem \ref{theorem distance}, we briefly discuss the implications of that choice. Controlling the Wasserstein distance with respect to a stronger underlying norm always implies the same Wasserstein-bound w.r.t.~any weaker underlying norm. Of course, there is the possibility to obtain a better bound by directly applying the theorem for a weaker norm. 

Picking stronger norms for $H$ will in general result in the Wasserstein distance also factoring in differences in sample smoothness as well as deviations in function values. For example, picking $H = L^2$ means that the bound only implies closeness of the function evaluations while, for example, samples being too rough is not factored in. Picking positive Sobolev spaces will punish deviations in function values and deviations in the derivatives of the samples from the true samples. Picking negative Sobolev spaces for $H$ (as we will need to for the common implementation of diffusion models; see Section \ref{sec:guidance_wndm}) means that the samples are only close in a distributional sense. See Section \ref{sec:negative_sobolev_spaces} in the appendix for further discussion.

\section{Implementing Infinite-Dimensional Diffusion Models}
\label{sec:dms_in_infinite_dimensions}


Our theory yields several suggestions on how one should design infinite-dimensional diffusion model algorithms. Section \ref{sec:guidance} gives a concise summary of those design principles, while Section \ref{sec:guidance_wndm} discusses the extent to which those design principles align with common implementations of diffusion models. In Section \ref{sec:numerics}, we then show how to implement the guidelines in some explicit examples. 


The two main design choices we will discuss are
\begin{enumerate}
	\item The choice of the forward noising process $W_t^U$ in \eqref{inf-fwd}.
	\item The choice of the norm $\|\cdot\|_K$ in the denoising score matching objective in \eqref{equ:full_loss}.
\end{enumerate}
The first choice (of $U$ and hence $W_t^U$) 
is equivalent to the choice of an invariant Gaussian distribution $\mathcal{N}(0, C)$, and we will use these choices interchangeably. In general, picking a $C$ such that $\mathcal{N}(0, C)$ produces smoother samples corresponds to picking a smaller space $U$, while rougher samples correspond to larger Cameron--Martin spaces $U$; see also the examples in Section \ref{sec:family_gaussian_measures} and Figure \ref{fig:gaussians_different_alphas}. Furthermore, after discretizing the problem to finite dimensions, the choice of $U$ or $C$ corresponds to nothing else than specifying a
covariance matrix $C$,
i.e., $W_t$ is replaced by $\sqrt{C} W_t$ in the diffusion processes.  See Section \ref{sec:algorithm} for more details.

After discretization, the second design choice of $\|\cdot\|_K$ 
corresponds to specifying a loss norm of the form $\|K^{-1/2} \cdot\|$ in place of the typical Euclidean norm $\|\cdot\|$.

\subsection{Practical Implementation Guidance}\label{sec:guidance}

\subsubsection{Match $C$ to $\mudata$} First, we begin by pointing out the implications of Theorem \ref{theorem distance} on the choice of $C$, or equivalently, $U$. The term $\mathcal{W}_2(\mudata, \mathcal{N}(0, C))$ appearing in the error bound \eqref{equ:w2_distance_bound} clearly indicates choosing $C$ such that $\mathcal{N}(0, C)$ is as close as possible to $\mudata$.



\subsubsection{Choosing $C$ such that we can pick a strong $H$-norm} Second, as discussed at the end of Section \ref{sec:distance_target_measure}, we would like to choose as strong an $\|\cdot\|_H$-norm as possible in Theorem \ref{theorem distance}.
This suggests not picking $C$ too rough ($U$ too large) as the norm space $H$ in Theorem \ref{theorem distance} has to support $\mathcal{N}(0, C)$.

This last points seems to suggest that we would want to pick $C$ as smooth as possible to allow for stronger $H$-norms. However, since $H$ has to support $\mudata$, there is a restriction on how strong an $H$-norm can be chosen. 
Therefore, this suggests matching $\mathcal{N}(0, C)$ to $\mudata$ so that they are supported on the same space $H$, similarly to our first observation.

\subsubsection{Choosing the loss-norm $\|\cdot\|_K$} The $H$-norm in Theorem \ref{theorem distance} also has to be stronger than the loss norm $\|\cdot\|_K$, again suggesting choosing $K$ as small as possible. However, besides numerical issues, also here there are lower bounds on how strong we can choose $K$.

To that end, we take another look at Lemma \ref{lemma:special_cases_score_matching} in which we study two separate cases. In the first case, if $C$ is rough enough such that $U$ contains the support of $\mudata$, we can choose the Cameron--Martin norm of $\mathcal{N}(0, C)$ as the loss norm, i.e., $K = U$. In the second case, $K$ has to support both $\mathcal{N}(0, C)$ and $\mudata$, which is the same condition as for the space $H$ chosen in Theorem \ref{theorem distance}. 

Hence, there are predominantly two natural ways to design the algorithm:
\begin{enumerate}
    \item Choose $\mathcal{N}(0, C)$ as smooth as possible / $U$ as small as possible, but large enough such that the support of $\mudata$ is contained in its Cameron--Martin space $U$. Then choose the loss norm $\|\cdot\|_K$ in \eqref{equ:full_loss} equal to the Cameron--Martin norm, $K = U$. This algorithm design is called \emph{Infinite-Dimensional Diffusion Model 1 (IDDM1)}
    \item Match $C$ to the data, i.e., choose $\mathcal{N}(0, C)$ such that its samples are as similar to the samples from $\mudata$ as possible. Then choose the loss norm $\|\cdot\|_K$ in \eqref{equ:full_loss} such that it supports both $\mudata$ and $\mathcal{N}(0, C)$. Let us call this algorithm design \emph{Infinite-Dimensional Diffusion Model 2 (IDDM2)}.
\end{enumerate}
Note that by Theorem \ref{thm:uniqueness_manifold}, if not much is known about the distribution, and if in particular it might be supported on manifold-like structures, we must pick $U$ large enough to contain the support of $\mudata$ anyway. We will therefore use IDDM1 in these cases; see Section \ref{sec:numerics_sphere}. If one has more structural information, for example the knowledge that $\mudata$ has density with respect to a Gaussian measure, we will use IDDM2; see Section \ref{sec:volatility_estimation}.

\subsection{Image Distributions and White Noise Diffusion Models}\label{sec:guidance_wndm}
The common implementation of the diffusion model algorithm will converge as $D \to \infty$ to $U = L^2$, i.e., use space-time white noise in the forward noising process. 
Furthermore, the loss function will also approach the $L^2$ loss, which means we are in the setting where we use the Cameron--Martin norm in the loss. For more details, see Section \ref{sec:algorithm}. We will call this algorithm \emph{White Noise Diffusion Model (WNDM)}. 

If $\mudata$ is an image distribution, we can expect it to lie on a manifold, or more generally some lower-dimensional substructure. Furthermore, since the function values of an image are bounded on $[0,1]$, the image samples are all contained in $L^2$. Therefore, we are in the setting of Theorem \ref{thm:uniqueness_manifold}, where the data are contained in the Cameron--Martin space $U$ of the noise. Furthermore, we can apply Lemma \ref{lemma:special_cases_score_matching} (bullet point $1$) to see that we can use the Cameron--Martin norm, i.e., the $L^2$ norm, and obtain a well-defined denoising score-matching objective. Therefore, under these assumptions, we have shown 
that applying WNDM to image distributions \emph{has a well-defined infinite-dimensional limit}.

Coming back to the discussion in Section \ref{sec:guidance} (in particular the design guidance for IDDM1), however, our theory suggests that we should try to pick $U$ as small as possible, while still containing the typical image distribution. However, this $U$ cannot be too regular, since images are quite irregular---for example, they can be discontinuous. If we identify images with two-dimensional functions, then already the $L^2$-Sobolev spaces $H^\alpha$ of order $\alpha > 1$ only contain continuous functions. Therefore, on the Sobolev scale $(H^\alpha:-\infty<\alpha<\infty)$, the `optimal' Cameron--Martin space would possess regularity of at most $\alpha=1$. In light of this, setting $U = H^0 =  L^2$ indeed seems like a natural choice that is close to matching the maximal possible regularity.
Strikingly, this is in line with the huge empirical success of the WNDM algorithm for image distributions. To further refine the optimal choice of the space $U$ beyond $L^2$ is an interesting avenue, both for theoretical and empirical future study. 

\section{Numerical Illustrations}\label{sec:numerics}

In this section, we illustrate our results through numerical experiments. 
We sample functions defined on $[0,1]$, and we discretize this spatial domain into a uniform grid with $D = 256$ evenly spaced points. 
For other discretization schemes, see the discussion in Section \ref{sec:algorithm}. We employ a grid-based spatial discretization since it allows us to use the popular U-Net architecture. Other common discretization schemes `whiten' the data, rendering the convolutional layers of the U-Net unnecessary. For implementation details, see Appendix~\ref{sec:numerics_details}.

Section \ref{sec:family_gaussian_measures} introduces some common preliminaries needed for both of the subsequent numerical examples. 
Section \ref{sec:numerics_sphere} then compares various diffusion model constructions in the setting of distributions that are not defined via Gaussian reference distributions, but rather supported on submanifolds of the infinite-dimensional space. Section \ref{sec:volatility_estimation}  demonstrates the use of infinite-dimensional diffusion models for solving Bayesian inverse problems via a simulation-based (i.e., conditional sampling) approach.

\subsection{Families of Gaussian Measures}\label{sec:family_gaussian_measures}

\begin{figure}
	\centering
	\begin{subfigure}[b]{0.32\textwidth}
		\centering
		\includegraphics[width=\textwidth]{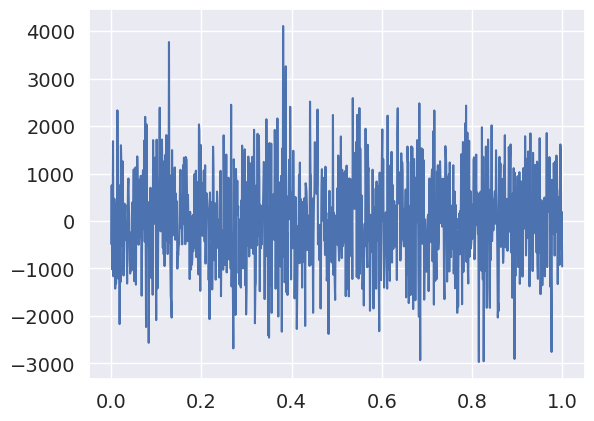}
		\caption{$\alpha = 0$, one sample}
            \label{fig:space_time_white_noise}
	\end{subfigure}
	\hfill
	\begin{subfigure}[b]{0.32\textwidth}
		\centering
		\includegraphics[width=\textwidth]{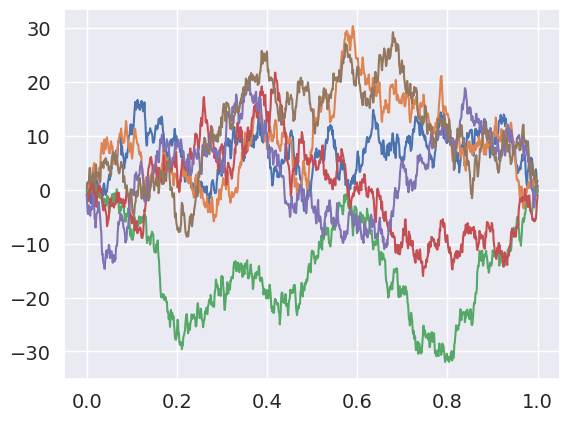}
		\caption{$\alpha = 1$, six samples}
	\end{subfigure}
	\hfill
	\begin{subfigure}[b]{0.32\textwidth}
		\centering
		\includegraphics[width=\textwidth]{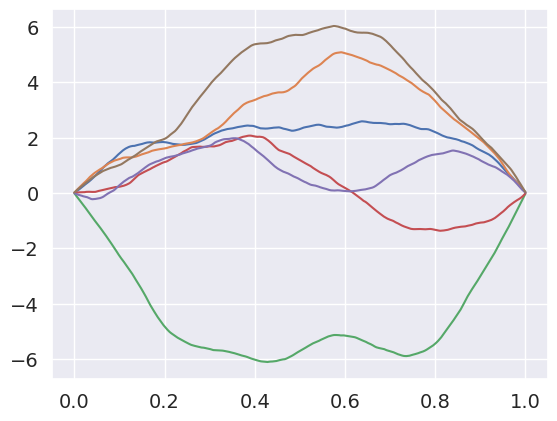}
		\caption{$\alpha = 2$, six samples}
	\end{subfigure}
	
	\caption{In each panel, we plot samples from $\pi^\alpha$ for different values of $\alpha$, where $\pi^\alpha$ is defined in Section \ref{sec:family_gaussian_measures}. We chose $e_k(\cdot) = \sqrt{2}\sin(2 \pi k \, \cdot \, )$ as an orthonormal basis of $L^2$. For $\alpha=0$ we see a sample of space-time white noise, where no function value is correlated to any of its neighboring function values. For $\alpha = 1$ and our specific choice of $e_k$, the sampled measure is the Brownian bridge measure.}
	\label{fig:gaussians_different_alphas}
\end{figure}

We first construct a family $(\pi^\alpha, H^\alpha : -\infty < \alpha < \infty)$ of Gaussian measures $\pi^\alpha$ and their Cameron--Martin spaces $H^\alpha$. This construction allows us to interpolate between measures with different sample smoothness and compare between the algorithms described in Section \ref{sec:guidance} and the canonical implementation of diffusion models described in Section \ref{sec:guidance_wndm}.

To that end, we fix an orthonormal basis $e_k$ of $L^2([0,1])$. Then, we construct a family $(\pi^\alpha :  -\infty < \alpha < \infty)$ of Gaussian measures as the distributions of
\begin{equation}
\sum_{k=1}^\infty k^{-\alpha} Z_k e_k \sim \pi^\alpha,
\label{equ:gaussian_family_distribution}
\end{equation}
where $Z_k \sim \mathcal{N}(0,1)$ i.i.d. The Cameron--Martin space of $\pi^\alpha$ is denoted by $H^\alpha$ and has norm
\begin{equation}
\|x\|^2_{\alpha} = \sum_{k=1}^\infty  k^{2 \alpha} \langle x, e_k \rangle^2_{L^2([0,1])}.
\label{equ:gaussian_family_norm}
\end{equation}
Note that $H^0 = L^2([0,1])$ and therefore $\pi^0$ is space-time white noise. Furthermore, $H^\alpha \subset H^\beta$ for $\alpha > \beta$. 
As we have discussed before, samples of $\pi^\alpha$ will (almost surely) not be elements of the corresponding Cameron--Martin space $H^\alpha$. Nevertheless, the distribution $\pi^\alpha$ is supported on $H^{\alpha - \kappa}$ as long as $\kappa > \frac{1}{2}$; see \citet[Proposition 3.1]{beskos2011hybrid}. 

The exact form of samples of $\pi^\alpha$ depends on the chosen basis $e_k$ in \eqref{equ:gaussian_family_distribution}. In our examples $H^\alpha$ will be $L^2$-Sobolev spaces with either zero or periodic boundary conditions. For the case of zero boundary conditions, we have visualized samples of $\pi^\alpha$ for different values of $\alpha$ in Figure \ref{fig:gaussians_different_alphas}. 

As discussed in Section \ref{sec:guidance}, we want to study two main modeling choices: 
First, we must select a Gaussian measure $\mathcal{N}(0, C)$, or equivalently its Cameron--Martin space $U$, for the noising process. We do that by fixing an $\alpha_\text{noise}$ and setting
\[
    U = H^{\alpha_\text{noise}}.
\]
Second, a loss norm $\|\cdot\|_K$ must be chosen. We do so by fixing an $\alpha_\text{loss}$ and setting
\[
    K = H^{\alpha_\text{loss}}.
\]
As recommended in Section \ref{sec:guidance}, these choices should depend on the structure of $\mudata$. Therefore, we choose an $\alpha_\text{data}$ and define $\mudata$ through a nonlinear transformation of $\pi^{\alpha_\text{data}}$. This way, $\mudata$ will be \emph{non-Gaussian}, but we still have perfect knowledge about where its samples are supported. In particular, we will have 
\[
    \text{support}(\mudata) \approx H^{\alpha_\text{data} - \frac{1}{2}}.
\] 
This gives us the possibility to match $U$ and $K$ to $\mudata$ in different ways. In realistic examples, knowledge about $\mudata$ could come from prior information or by studying the training samples---the empirical covariance matrix is, for example, a natural candidate for specifying $C$ in IDDM2. 

Lastly, we will be able to make explicit statements about the norm of $H$ for which the distance bounds in Theorem \ref{theorem distance} hold, which we will also quantify by choosing an $\alpha_\text{dist}$. The larger $\alpha_\text{dist}$, the better, since the underlying norm for the distance measurement gets stronger (see also the discussion in Section \ref{sec:guidance}).

Note that the limit of the common implementation of diffusion models, which we called WNDM (see Section \ref{sec:guidance_wndm}) will use white noise for the noising process as well as the loss, i.e., $\alpha_\text{noise} = \alpha_\text{loss} = 0$. In that case, $\mathcal{N}(0, C)$ will only be supported on any $H^\alpha$ with $\alpha < -\frac{1}{2}$. Therefore, so that we can apply Theorem \ref{theorem distance} with $H^{\alpha_\text{dist}}$ we have to choose $\alpha_\text{dist} < -\frac{1}{2}$, i.e., use the norm of a negative Sobolev space.

For the two numerical experiments in Sections \ref{sec:numerics_sphere} and \ref{sec:volatility_estimation}, we will proceed as follows: 
\begin{enumerate}
\item We have information about the sample smoothness and support of $\mudata$, in this case in the form of an $\alpha_\text{data}$. 
\item Based on Section \ref{sec:guidance}, we then choose $U$ and $K$, which boils down to the choice of  $\alpha_\text{noise}$ and $\alpha_\text{loss}$.
\item We then know for which norms $\|\cdot\|_H$ our Wasserstein bound in Theorem \ref{theorem distance} holds. In our interpolation family, this boils down to an upper bound for $\alpha_\text{dist}$. Therefore, we can make statements about which properties of $\mudata$ the diffusion model should successfully approximate.
\end{enumerate}

\subsection{Manifold Distribution on a Cameron--Martin Sphere}\label{sec:numerics_sphere}
In this section, we will study a distribution which lies on an infinite-dimensional submanifold of $L^2([0,1])$, namely the unit sphere of some Cameron--Martin space. 
To that end, choose $e_k (\cdot)  = \sqrt{2}\sin(k \pi \, \cdot)$ in the construction of Section \ref{sec:family_gaussian_measures}. For this choice, the Cameron--Martin spaces $H^\alpha$ will be the Sobolev spaces $W^{\alpha, 2}_0$ of functions vanishing at the boundary, and $\pi^1$ is proportional to the distribution of a Brownian bridge. Here we see that we can not only capture smoothness but also structural information, such as boundary conditions, through the choice of an appropriate Gaussian measure. For a more in-depth study of this, see \citet{mathieu2023geometric}.

We draw $N = 50\,000$ samples from $\pi^{\alpha_\text{data}}$, where the data-generating $\alpha_\text{data}$ was set to
\[
	\alpha_\text{data} = 2.
\]
By our discussion in Section \ref{sec:family_gaussian_measures}, these samples are supported on any $H^\alpha$ with $\alpha < \frac{3}{2}$, in particular $H^1$. Now define 
$\alpha_\text{supp} = 1 < \frac{3}{2}$.
The target distribution $\mudata$ is created by projecting $\pi^{\alpha_\text{data}}$ onto the $10$-sphere in $H^{\alpha_{\text{supp}}}$, i.e., applying the map
\[
	H^{\alpha_\text{supp}} \to H^{\alpha_\text{supp}}, \quad x \mapsto 10 \frac{x}{\|x\|_{H^\alpha_\text{supp}}}
\]
to all samples. We depict some of the training samples and a heatmap of their marginal densities in Figure \ref{fig:sphere_256D_train}.
\begin{figure}
	\centering
	\begin{subfigure}[b]{0.6\textwidth}
		\centering
		\includegraphics[width=\textwidth]{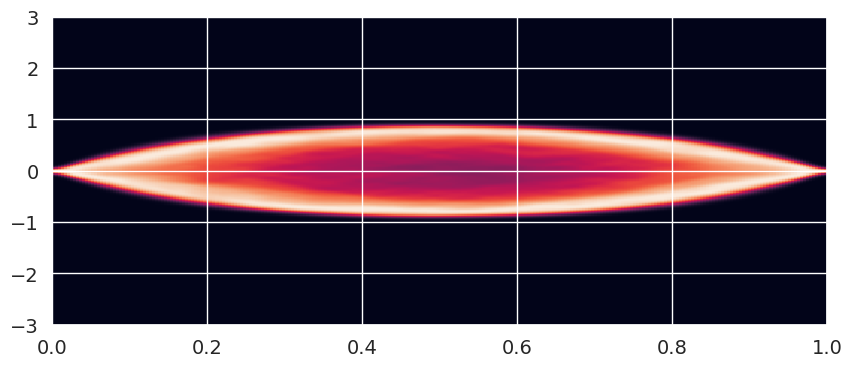}
	\end{subfigure}
	\hfill
	\begin{subfigure}[b]{0.38\textwidth}
		\centering
		\includegraphics[width=\textwidth]{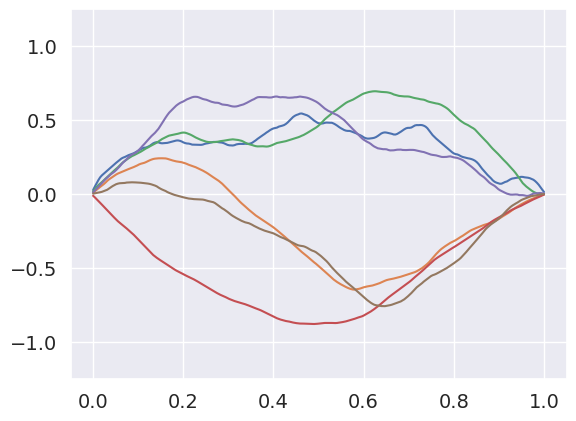}		
	\end{subfigure}	
	\caption{We generated $50~000$ training examples from the distribution described in Section \ref{sec:numerics_sphere}. On the left, we show a heatmap of the resulting marginal densities of function values at each point in the domain $[0,1]$. On the right, we plot a few training samples.}
	\label{fig:sphere_256D_train}
\end{figure} 

As described in Section \ref{sec:guidance}, we use the theory to guide the choices for $\alpha_\text{noise}$ and $\alpha_\text{loss}$. The data distribution was not absolutely continuous with respect to a Gaussian, since we projected it to a submanifold (the sphere). Therefore, we must apply Theorem \ref{thm:uniqueness_manifold} to obtain uniqueness. To satisfy the assumptions of Theorem \ref{thm:uniqueness_manifold}, however, the Cameron--Martin space $U$ has to contain $\mudata$. Hence, we will apply the IDDM1 from Section \ref{sec:guidance}. 

To apply IDDM1, we choose $U$ so that it contains the support of $\mudata$ and $\mathcal{N}(0, C)$. This is accomplished by setting $\alpha_\text{noise} = 1$. Then, following the design principles of IDDM1, we pick the loss norm to be $K = U$, i.e., $\alpha_\text{loss} = \alpha_\text{noise} = 1$, and learn the score by using the Cameron--Martin norm in the loss. 

Note that the bound in Theorem \ref{theorem distance} holds for any $\alpha_\text{dist}$ smaller than
\[
    \alpha_\text{dist} 
    < \min\left \{\alpha_\text{data} - \frac{1}{2}, \alpha_\text{noise} - \frac{1}{2}, \alpha_\text{loss} \right \} 
    = \min \left \{\frac{3}{2}, \frac{1}{2}, 1 \right \} = \frac{1}{2}.
\]
For WNDM, i.e., the canonical implementation of diffusion models described in Section \ref{sec:guidance_wndm}  with $\alpha_\text{loss} = \alpha_\text{noise} = 0$, the upper bound is
$-\frac{1}{2}$. Therefore, while we do not expect the samples of WNDM to match the smoothness class of $\mudata$, we expect the samples of IDDM1 to at least partially retain the smoothness.


\begin{figure}
	\centering
	\begin{subfigure}[b]{0.32\textwidth}
		\centering
		\includegraphics[width=\textwidth]{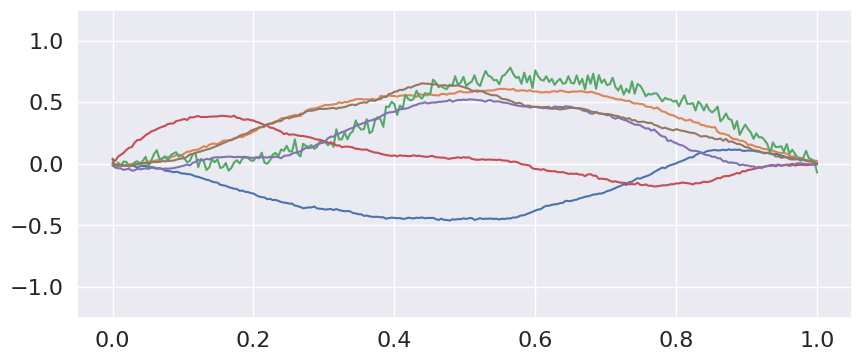}
	\end{subfigure}
	\hfill
	\begin{subfigure}[b]{0.32\textwidth}
		\centering
		\includegraphics[width=\textwidth]{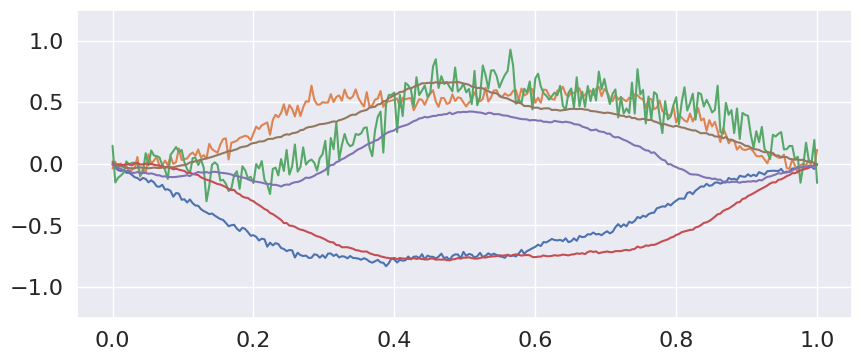}
	\end{subfigure}
	\hfill
	\begin{subfigure}[b]{0.32\textwidth}
		\centering
		\includegraphics[width=\textwidth]{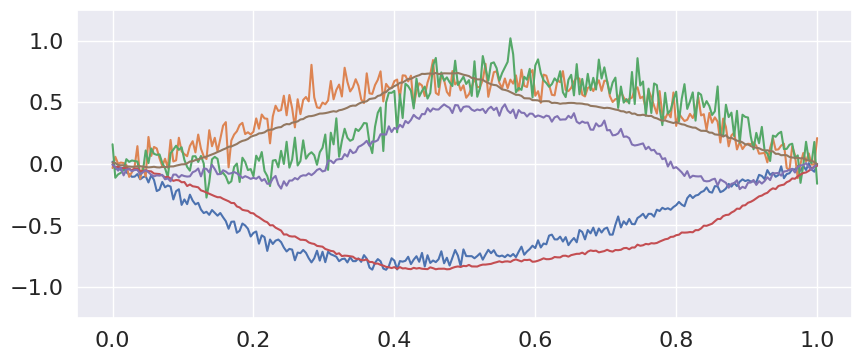}
	\end{subfigure}

	\begin{subfigure}[b]{0.32\textwidth}
	\centering
	\includegraphics[width=\textwidth]{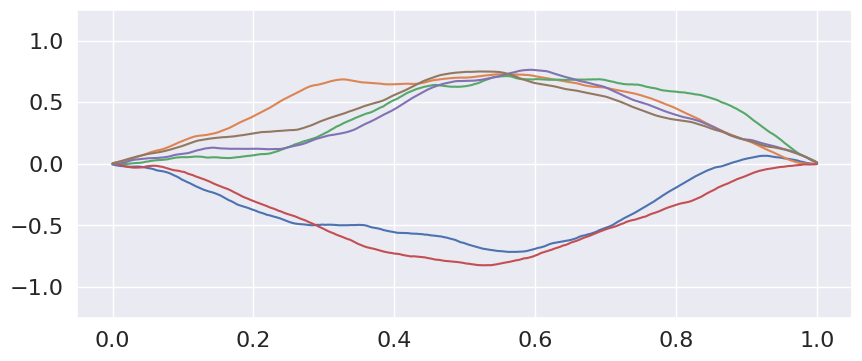}
        \caption{Epoch 10}
	\end{subfigure}
	\hfill
	\begin{subfigure}[b]{0.32\textwidth}
		\centering
		\includegraphics[width=\textwidth]{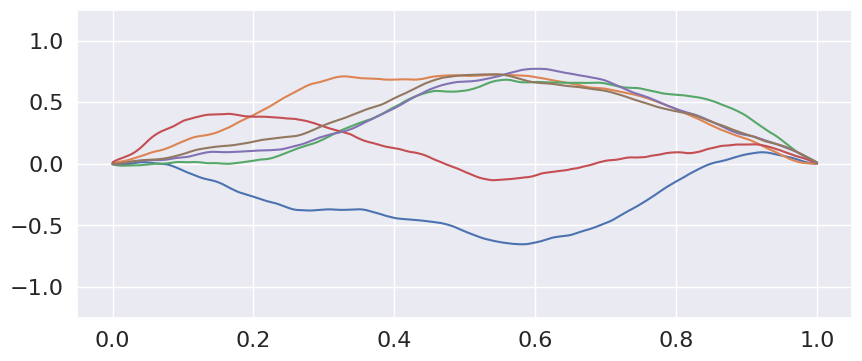}
            \caption{Epoch 30}
	\end{subfigure}
	\hfill
	\begin{subfigure}[b]{0.32\textwidth}
		\centering
		\includegraphics[width=\textwidth]{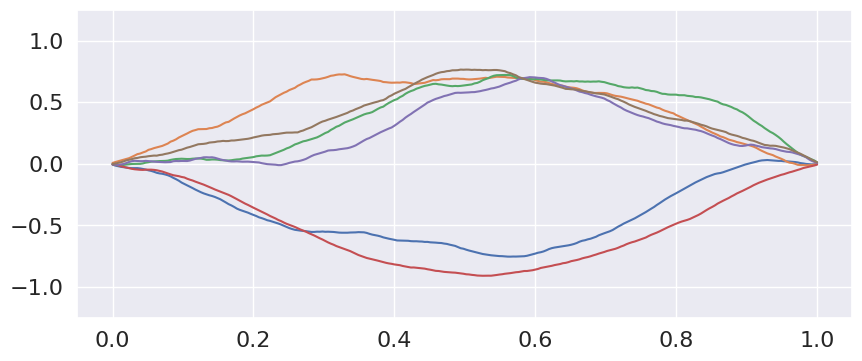}
            \caption{Epoch 60}
	\end{subfigure}


	\caption{Example of Section \ref{sec:numerics_sphere}: samples generated by WNDM (row 1) and IDDM1 (row 2) after increasing numbers of training epochs. Samples from the true measure can be compared in Figure \ref{fig:sphere_256D_train}.}
	\label{fig:samples_gaussian_sphere}
\end{figure}

Figure \ref{fig:samples_gaussian_sphere} shows samples generated by the two models. We see that our theoretical findings are confirmed: WNDM fails to learn the smoothness or correlation structure of the samples. 
Solely at training epoch 10 the WNDM algorithm generated some samples that seemed to have the right smoothness, but even those actually contain jitter if one looks closely. Overall, the training process was very unstable regarding the data smoothness, and minimizing the loss did not seem to correlate with also matching the derivatives of the functions. 
On the other hand, IDDM1 produces samples from the correct smoothness class, from the start of training onwards.

Note that both algorithms matched the marginals quite well, as can be seen in the heatmap plots of Figure \ref{fig:heatmaps_gaussian_sphere}, which is also suggested by the theory: even if Theorem \ref{theorem distance} only holds for an underlying negative Sobolev norm, the overall distribution and in particular its marginals should still match the true marginals (see Section \ref{sec:negative_sobolev_spaces}).

\begin{figure}
	\centering
	\begin{subfigure}[b]{0.48\textwidth}
		\centering
		\includegraphics[width=\textwidth]{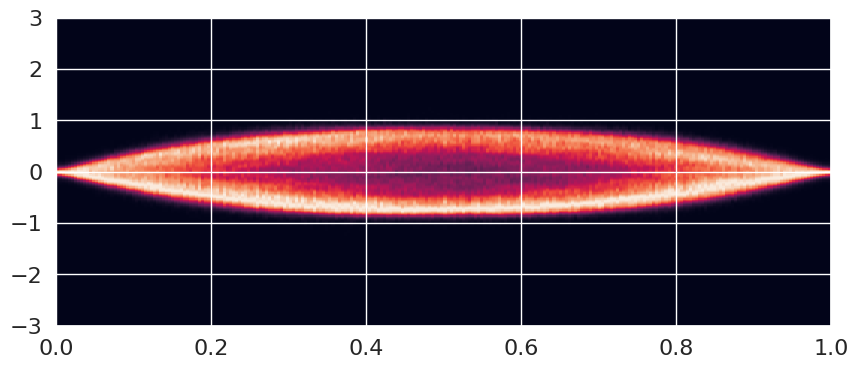}
		\caption{WNDM}
	\end{subfigure}
	\hfill
	\begin{subfigure}[b]{0.48\textwidth}
		\centering
		\includegraphics[width=\textwidth]{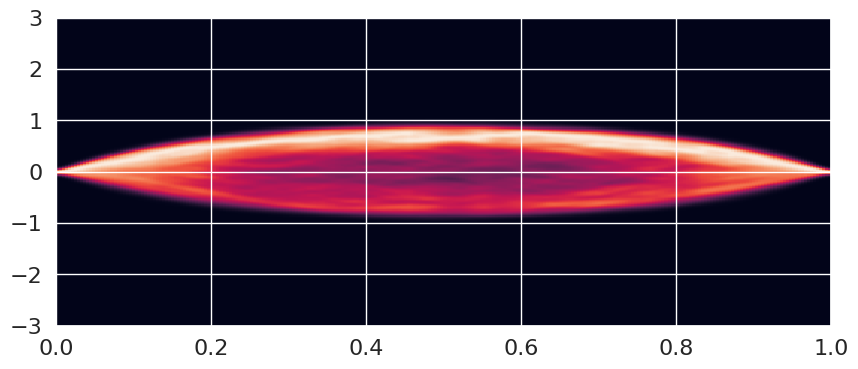}
		\caption{IDDM1}
	\end{subfigure}
	\hfill
	\caption{Example of Section \ref{sec:numerics_sphere}: each vertical slice shows a heatmap of the marginal density estimated from 2048 samples generated by each of the diffusion models, after 60 epochs of training. For comparison, the heatmap of the 50\,000 training examples is plotted in Figure \ref{fig:sphere_256D_train}. The one-dimensional marginals are matched well by both algorithms.}
	\label{fig:heatmaps_gaussian_sphere}
\end{figure}

\subsection{\jp{Conditional Sampling and} Infinite-Dimensional Bayesian Inverse Problems: Volatility Estimation}
\label{sec:volatility_estimation}

The following numerical experiment is inspired by Bayesian inverse problems (BIPs) \citep{stuart2010inverse}, which here we approach via the paradigm of \emph{simulation-based inference}. In this setting, we will use our infinite-dimensional diffusion models for conditional sampling. 

We assume that we have some knowledge about an unknown random variable $X\in\mathcal{H}$ in the form of a  measurement $y \in \mathbb{R}^l$ drawn from
\[
Y \sim q(X, \cdot),
\]
where $q$ is an observation kernel. Furthermore, we have some prior information on $X$, formalized through a prior probability distribution:
\[
X \sim \pi \coloneqq \mathcal{N}(0, C_\mu).
\]
By Bayes' theorem, the posterior distribution $\nu$ of $X$ given some observations $Y=y$ is given by
\begin{equation}
	d \pi( \, \cdot \,  | Y = y) \propto q(\cdot, y) ~\mathd \pi.
	\label{equ:posterior_distribution}
\end{equation}
Gold-standard methods for asymptotically exact sampling of distributions like \eqref{equ:posterior_distribution} involve Markov chain Monte Carlo (MCMC), e.g., Hamiltonian Monte Carlo \citep{duane1987hybrid} or, in the infinite-dimensional case, Hilbert space Hamiltonian Monte Carlo (HSHMC) \citep{beskos2011hybrid}, or other geometry-exploiting infinite-dimensional MCMC methods \citep{cotter2013statsci, cui2016dimension, kim2023hippylib}.

These MCMC methods, however, rely on having an explicit formula for the density of $\nu$ (up to a normalizing constant). In many cases this is not possible---for example, if $q$ or any of its components is given as a black-box model. To train a conditional diffusion model, on the other hand, we only need \textit{samples} from the joint distribution of $(X, Y)$. These can be generated by sampling $X^i$ from the prior measure and sampling $Y \sim q(X^i, \cdot)$. We then train a conditional diffusion model to generate samples from $X | Y = y$ for any $y$. This is done by making the score model $s$ not only depend on $X_t$, but also on $Y$, i.e., we have a model $s(t, X_t, Y)$, which predicts $X_0$ given $Y=y$. \jp{The only modification to Algorithm \ref{alg:training} is that one sub-samples paired states and observations $(x^i, y^i)$ in line 3 from the training data, and then inputs $y^i$ into the diffusion model on line 6.}
During generation, one can then input the observation value $y$ that one wants to condition on during simulation of the reverse SDE \citep{batzolis2021conditional}. \jp{In Algorithm \ref{alg:sampling}, this would correspond to inputting a fixed value of $y$ for all times $t$ in line 5. Hence the entire procedure is sample-driven, and an example of simulation-based inference \citep{cranmer2020frontier}.}

We now proceed to a specific instance of a Bayesian inverse problem. The experiment is inspired by volatility estimation. We assume that we observe a path of a time series, for example a stock price, modeled as
\[
	\mathd S_\tau = \sigma_\tau S_\tau \mathd B_\tau,
\]
with no drift and a time-dependent volatility $\sigma_\tau$. The solution to the above equation is given by a geometric Brownian motion, i.e.,
\[
	S_\tau = S_0 \exp \left (\int_0^\tau \sigma_r \mathd B_r - \frac{1}{2} \int_0^\tau \sigma_r^2 \mathd r \right).
\]
We simulate paths of the above and observe $S_\tau$ at discrete times $\tau_1=\frac{1}{4}, \tau_2 = \frac{2}{4}, \tau_3 = \frac{3}{4}, \tau_4=1$. Then, we apply a log-transformation and define $r_i$ as the log-returns: 
\begin{equation*}
	r_i \coloneqq \log S_{\tau_{i}} - \log S_{\tau_{i-1}} = \int_{\tau_{i-1}}^{\tau_{i}} \sigma_r \mathd B_r - \frac{1}{2} \int_{\tau_{i-1}}^{\tau_{i}}\sigma_r^2 \mathd r \sim \mathcal{N}\left (-\frac{1}{2} v_{i}, v_{i} \right ),~ \text{with } v_i \coloneqq \int_{\tau_{i-1}}^{\tau_i} \sigma_r^2 \mathd r.
\end{equation*} 
Here, we set $\tau_0 = 0$ for notational convenience.
Since $\sigma$ should be positive, we model it as
\[
\sigma_\tau = \exp(a_\tau),
\]
and seek to infer the log-volatility $a: [0, 1] \to \mathbb{R}$.

Again, we define a family of Gaussian measures as in Section \ref{sec:family_gaussian_measures}. This time we use a different orthonormal basis of $L^2([0,1])$, given by
\[
	e_k(\tau) = \begin{cases}
			\sqrt{2} \cos(k \pi \tau), & \text{if } k \text{ even}\\
			\sqrt{2} \sin((k+1) \pi \tau), & \text{otherwise}
		  \end{cases}.
\]
This leads to Gaussian measures whose samples have periodic boundary conditions.
Since $e_k$ and $e_{k+1}$ (for $k$ uneven) should have the same `magnitude,' we slightly modify \eqref{equ:gaussian_family_distribution} and \eqref{equ:gaussian_family_norm}: for $k$ uneven, we replace $(k+1)^{-\alpha}$ by $k^{-\alpha}$. All the discussed properties of the family $\pi^\alpha$ are not affected by this change, since the decay of the eigenvalues is asymptotically the same. We put a prior on $a$. It's covariance is given by $\frac{1}{2} C_\text{prior}$, where $C_\text{prior}$ is the covariance of $\pi^{\alpha_\text{data}}$, with $\alpha_\text{data} = 4$:
\begin{equation*}
	a \sim \mathcal{N}(0, C_{\text{prior}}).
\end{equation*}
The goal of a conditional diffusion model is to generate samples from the posterior
\begin{equation}
d\pi^{\alpha_\text{data}}(a_\tau | r_1, r_2, r_3, r_4) 
\propto \prod_{i=1}^4 \mathcal{N} \left (r_i; -\frac{1}{2}v_i, v_i \right) \mathd \pi^{\alpha_\text{data}},
\label{equ:posterior_volatility}
\end{equation} 
for a fixed observation $r = (r_1, r_2, r_3, r_4)$. Via the model defined above, each $v_i$ is a functional of $\sigma_\tau$ and thus $a_\tau$. 

For training, we generate $N = 50\,000$ samples from the prior $\{a^n\}_{n=1}^{N}$ together with simulated observations $\{r^n\}_{n=1}^N$. The trained diffusion models should, for any input $r \in \mathbb{R}^4$, generate samples from \eqref{equ:posterior_volatility}. 

To assess the performance of the trained models, we drew a random $\tilde{a}_\tau$ and corresponding observations $\tilde{r} = (\tilde{r}_1, \tilde{r}_2, \tilde{r}_3, \tilde{r}_4)$. We used the HSHMC algorithm to generate $50\,000$ ``reference'' posterior samples from \eqref{equ:posterior_volatility} for this fixed observation value $\tilde{r}$. We plot these posterior samples and their heatmap, as well as the data-generating value $\tilde{a}$ of the log-volatility, in Figure \ref{fig:256D_heatmap_truth}. After training, we input $\tilde{r}$ (which the diffusion models have not seen before) to the conditional diffusion models and compare the generated samples to those from HSHMC.

As in Section \ref{sec:numerics_sphere}, we again compare the canonical diffusion model implementation WNDM against an implementation motivated by the infinite-dimensional theory. 
In this case, since we are sampling from a Bayesian inverse problem with a Gaussian prior, we are in the setting of Theorem \ref{thm:uniqueness_gaussian}. Therefore, we will implement the IDDM2 algorithm from Section \ref{sec:guidance}. 
We match the noise structure to the data by setting $\alpha_\text{noise} = \alpha_\text{data} = 4$, which is justified by the form of \eqref{equ:posterior_volatility} of $\mudata$.
Then we set $\alpha_\text{loss} = 2$, such that $K$ supports $\mudata$ and $\pi^{\alpha_\text{noise}}$. Note that any $\alpha_\text{loss} < \frac{7}{2}$ would also have been a valid choice. The choice $\alpha_\text{loss} = 3$ worked comparably well in our numerical experiments.

We compare samples generated by the two diffusion models in Figure \ref{fig:256D_samples_trained}. Again, the WNDM algorithm did not match the smoothness class of $\mudata$ in a stable way. 
While during training, there were times at which the network generated smooth samples, it later unlearned to do so. 
The IDDM2 algorithm outputs samples of the correct class at every point during training. 
Both algorithms are able to match the marginal distributions, although IDDM2 does slightly better, as seen in Figure \ref{fig:256D_heatmap_trained}. Therefore, as in Section \ref{sec:numerics_sphere}, these numerical experiments confirm the theoretical predictions made in Section \ref{sec:dms_in_infinite_dimensions}.

\begin{remark}
    Note that for our choice of $\alpha_\text{noise}$, the reverse SDE now starts with initial condition $\pi^{\alpha_\text{data}}(a_\tau) = \mathcal{N}(0, C_\text{prior})$ and ends in the posterior $\pi^{\alpha_\text{data}}(a_\tau | r)$. Therefore, it has learned to transport the prior to the posterior. Furthermore, in this case, we can interpret the reverse SDE as a smoothed version of the forward process, i.e., an Ornstein--Uhlenbeck process conditioned on its terminal values. This also opens up the way to interpret the training of the reverse SDE as an infinite-dimensional control problem. 
\end{remark}

\begin{figure}
	\centering
	\begin{subfigure}[b]{0.49\textwidth}
		\centering
		\includegraphics[width=\textwidth]{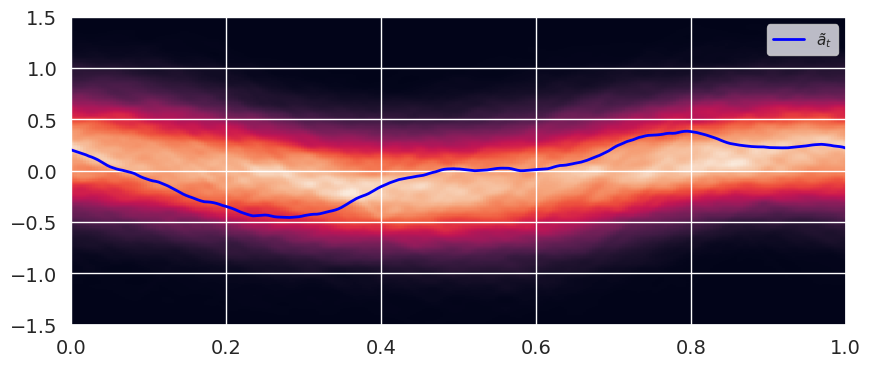}
	\end{subfigure}
	\hfill
	\begin{subfigure}[b]{0.49\textwidth}
		\centering
		\includegraphics[width=\textwidth]{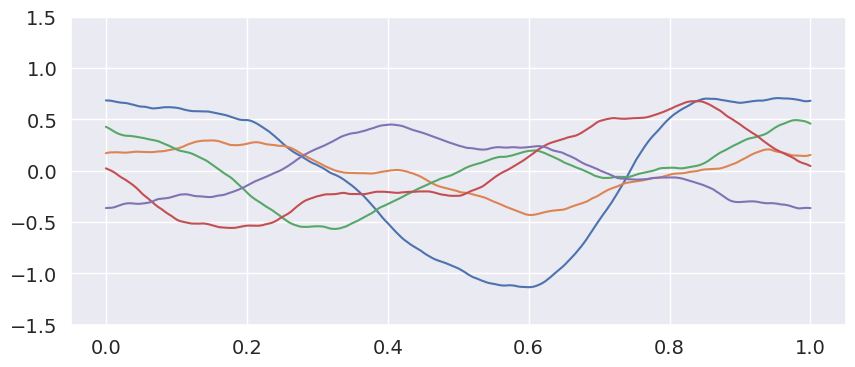}		
	\end{subfigure}	
	\caption{Example of Section~\ref{sec:volatility_estimation}. As a reference/comparison, we generate $50\,000$ high-quality posterior samples from $d \pi^{\alpha_\text{data}}(a_\tau \vert \tilde{r})$ using the Hilbert space Hamiltonian Monte Carlo algorithm. On the left is a heatmap of posterior marginal densities of $a_\tau$, at each point in the domain $\tau \in [0,1]$. On the right, we plot a few example posterior samples.}
	\label{fig:256D_heatmap_truth}
\end{figure}

\begin{figure}
	\centering
	\begin{subfigure}{0.32\textwidth}
		\centering
		\includegraphics[width=\textwidth]{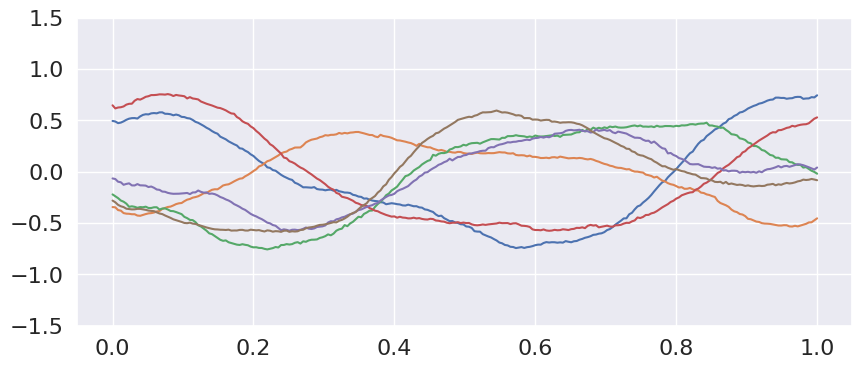}
	\end{subfigure}
	\hfill
	\begin{subfigure}[b]{0.32\textwidth}
		\centering
		\includegraphics[width=\textwidth]{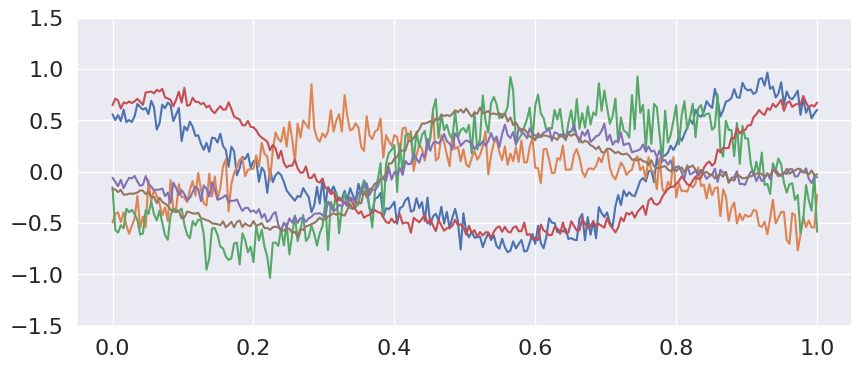}
	\end{subfigure}
	\hfill
	\begin{subfigure}[b]{0.32\textwidth}
		\centering
		\includegraphics[width=\textwidth]{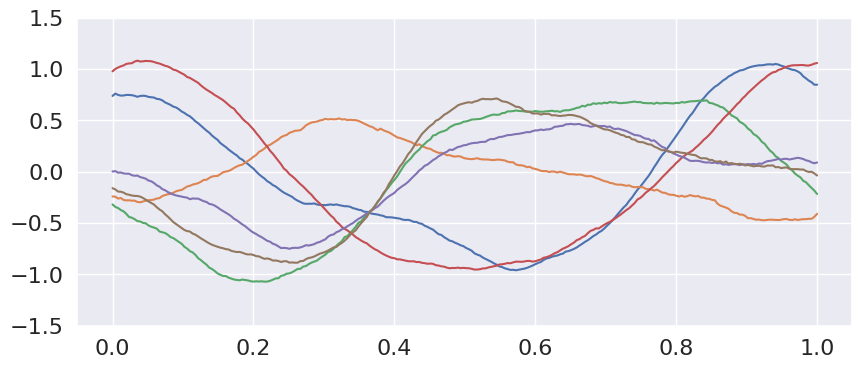}
	\end{subfigure}

	\centering
	\begin{subfigure}[b]{0.32\textwidth}
		\centering
		\includegraphics[width=\textwidth]{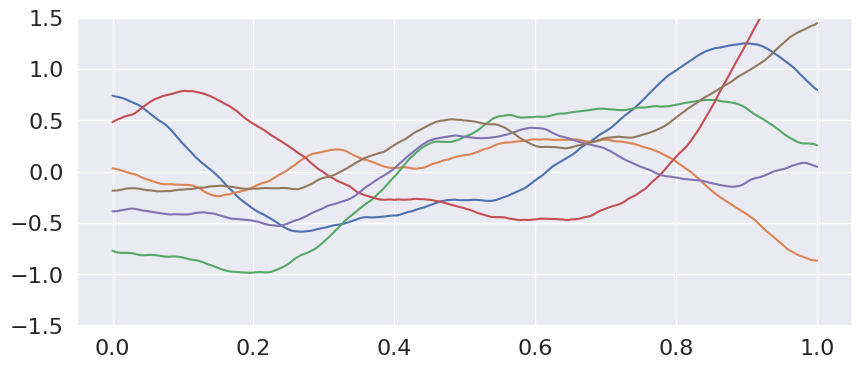}
		\caption{Epoch 10}
	\end{subfigure}
	\hfill
	\begin{subfigure}[b]{0.32\textwidth}
		\centering
		\includegraphics[width=\textwidth]{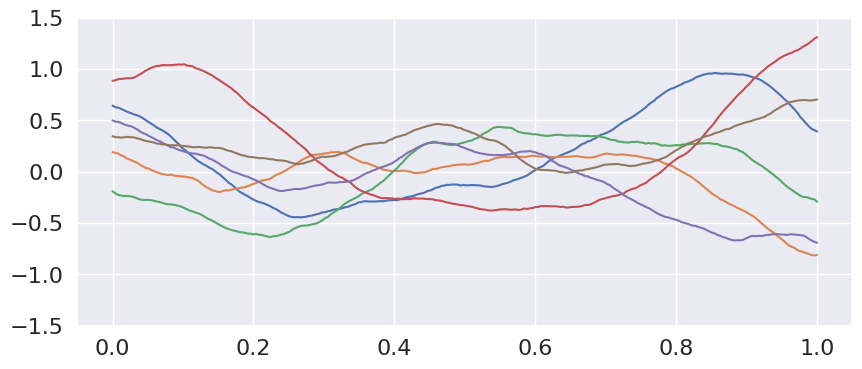}
		\caption{Epoch 30}
	\end{subfigure}
	\hfill
	\begin{subfigure}[b]{0.32\textwidth}
		\centering
		\includegraphics[width=\textwidth]{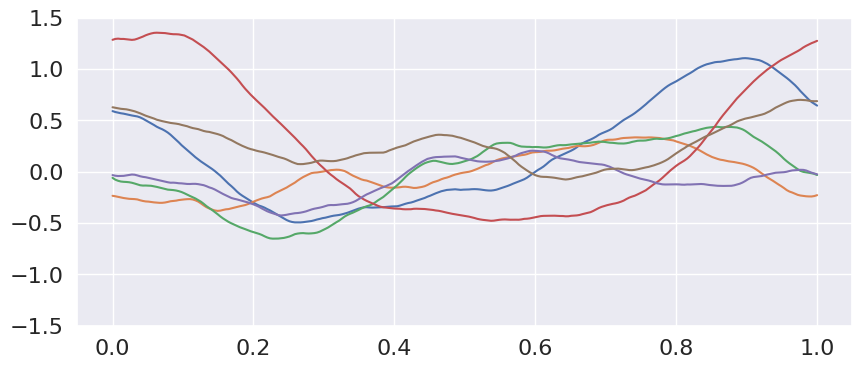}
		\caption{Epoch 60}
	\end{subfigure}
	\caption{Example of Section~\ref{sec:volatility_estimation}. Conditional samples from WNDM (upper row) and IDDM2 (lower row) after varyings number of training epochs. Compare to the high-quality posterior samples generated using Hamiltonian Monte Carlo in Figure \ref{fig:256D_heatmap_truth}.}
	\label{fig:256D_samples_trained}
\end{figure}

\begin{figure}
	\centering
	\begin{subfigure}[b]{0.32\textwidth}
		\centering
		\includegraphics[width=\textwidth]{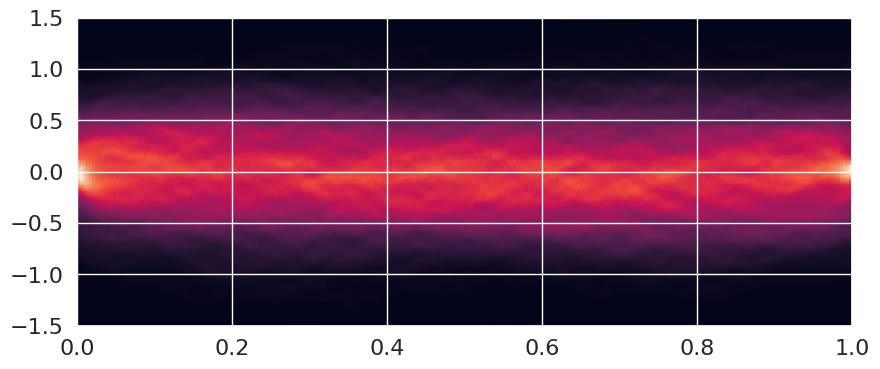}
	\end{subfigure}
	\hfill
	\begin{subfigure}[b]{0.32\textwidth}
		\centering
		\includegraphics[width=\textwidth]{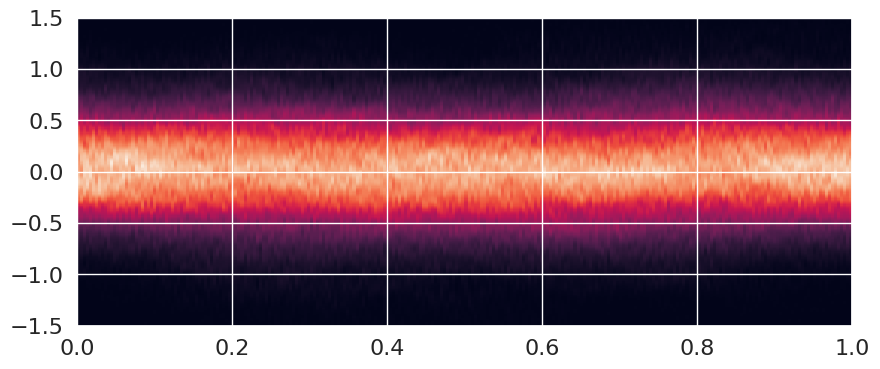}
	\end{subfigure}
	\hfill
	\begin{subfigure}[b]{0.32\textwidth}
		\centering
		\includegraphics[width=\textwidth]{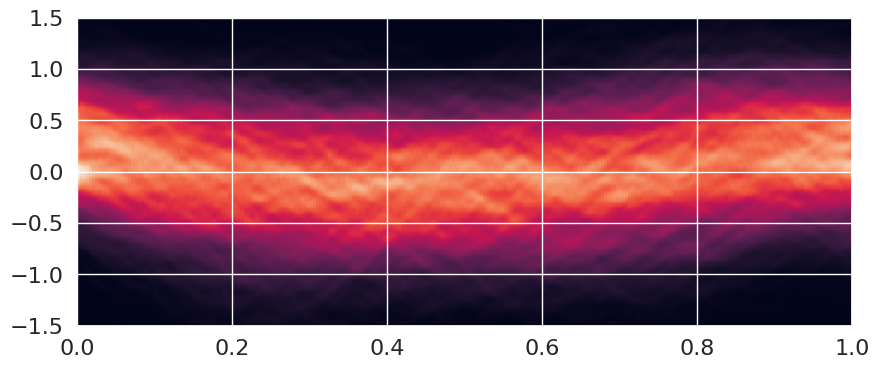}
	\end{subfigure}

	\centering
	\begin{subfigure}[b]{0.32\textwidth}
		\centering
		\includegraphics[width=\textwidth]{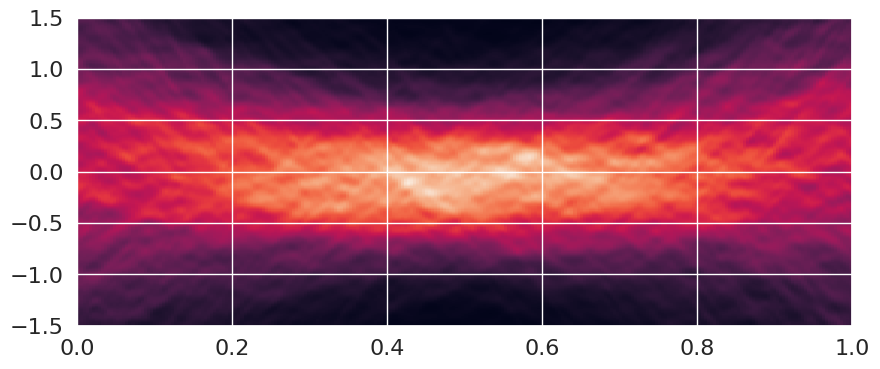}
		\caption{Epoch 10}
	\end{subfigure}
	\hfill
	\begin{subfigure}[b]{0.32\textwidth}
		\centering
		\includegraphics[width=\textwidth]{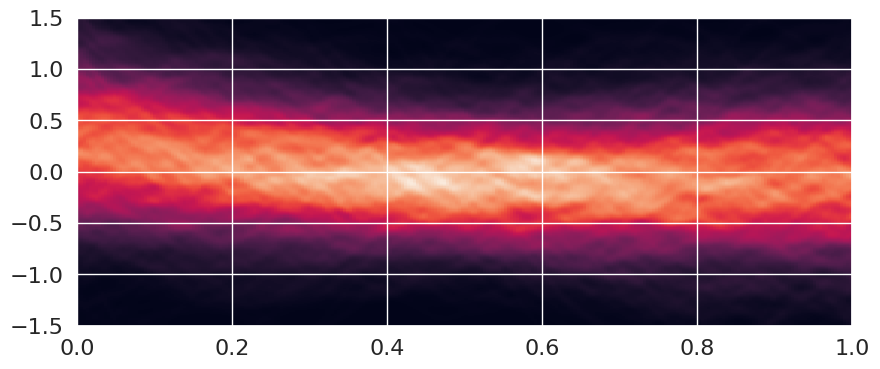}
		\caption{Epoch 30}
	\end{subfigure}
	\hfill
	\begin{subfigure}[b]{0.32\textwidth}
		\centering
		\includegraphics[width=\textwidth]{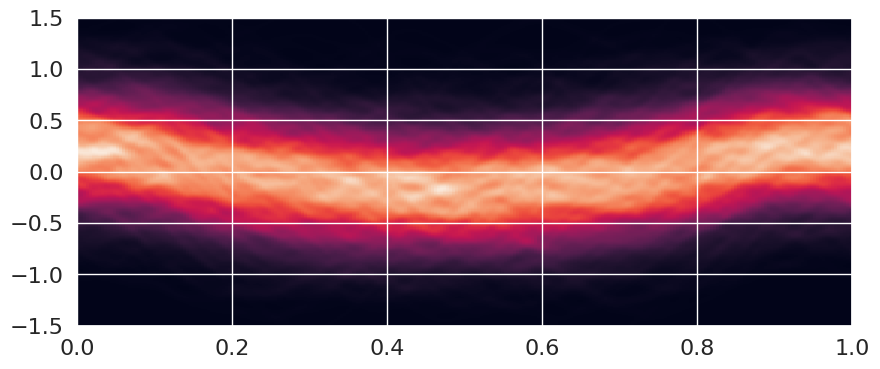}
		\caption{Epoch 60}
	\end{subfigure}
	\caption{Example of Section~\ref{sec:volatility_estimation}. Heatmaps of the $2048$ conditional samples generated by WNDM (upper row) and IDDM2 (lower row), for increasing numbers of training epochs. Reference heatmaps generated via Hilbert space Hamiltonian Monte Carlo are in Figure \ref{fig:256D_heatmap_truth} for comparison.}
	\label{fig:256D_heatmap_trained}
\end{figure}

\jp{
\subsection{Limitations}
It is important to point out that, generally, the white noise diffusion models were able to generate samples from the appropriate smoothness class after several rounds of retraining. However, the sample smoothness was not robust. While the models typically fitted the moments and marginals of the distributions quite well, the smoothness of the samples proved inconsistent---varying with the initial conditions and the duration of the training. Specifically, depending on the network's initial training parameters, the samples could either be smooth, become less smooth over time, or exhibit initial smoothness that diminished as the training progressed and the marginals were better fitted.} 

\jp{
We also note that the numerics here are intended as an illustration of the preceding theory and the resulting guidelines. A more comprehensive numerical study could train multiple models with different random initial conditions on the same training data set, perform ablation studies over individual design choices, and compare numerical measures of the smoothness of paths. It would also be of interest to evaluate the impact of different neural network architectures. Such studies are outside the scope of this article, however.
}

\section{Summary}
We have formulated the diffusion-based generative modeling approach directly on infinite-dimensional Hilbert spaces. 
Our formulation involves specifying infinite-dimensional forward and reverse SDEs and an associated denoising score matching objective. We prove that our formulation is well-posed. To that end, we show that the reverse SDE we wish to approximate has a unique solution; furthermore, we show under which conditions the denoising score matching objective generalizes to an infinite-dimensional setting. Building on these results, we are able to prove dimension-independent convergence bounds for diffusion models, which hold in the infinite-dimensional case.

These theoretical developments reveal an intricate relationship between the properties of the target/data-generating measure $\mudata$ and the choices of the Wiener process $W_t^U$ and the loss norm in the denoising score matching objective $\|\cdot\|_K$. We utilize this knowledge to develop guidelines on how to make such choices for a given $\mudata$. For image distributions, these guidelines are in line with the canonical choices made in practice. For other target distributions $\mudata$, however, the algorithm design should be modified. We apply these modifications to two generative modeling tasks that are discretizations of infinite-dimensional problems, and the numerical results confirm our theoretical findings.


\section*{Acknowledgements}
JP and SR have been partially supported by Deutsche Forschungsgemeinschaft (DFG), Project ID 318763901, SFB-1294. 
SW and YM acknowledge support from Air Force Office of Scientific Research (AFOSR) MURI Analysis and Synthesis of Rare Events, award number FA9550-20-1-0397. JP and SR would also like to thank the Isaac Newton Institute for Mathematical Sciences, Cambridge, for support and hospitality during the programme  {\it The Mathematical and Statistical Foundation of Future Data-Driven Engineering} where work on this paper was also undertaken. This work was supported by EPSRC grant no EP/R014604/1.

\appendix

\section{Numerical Details}\label{sec:numerics_details}
We list some implementation details:
\begin{enumerate}
\item Instead of running the forward SDE with a uniform speed, we instead ran it using a speed function $\alpha(t)$,
\[
    \mathd X_t = -\frac{1}{2}\beta(t) X_t \mathd t + \sqrt{\beta(t) C} \mathd W_t.
\]
The SDE was then run on the interval $[0,1]$. This corresponds to a time-change and is common practice for diffusion models; see for example \citet{song2021scorebased}. We used the time-change function $\beta(t)$ as in \citet{song2021scorebased}, i.e.,
\[
      \beta(t) = 0.001 + t(20 - 0.001).
\]

\item We discretized the unit interval $[0,1]$ into $M = 1000$ evenly spaced points for training and generation. 

\item We added a last denoising step, as is common practice. That means, that in the last step of the Euler-Maruyama integrator, we did not add any extra noise any more, but just evaluated the drift and took a step in that direction. This is even more important in our case for comparison than normally, since the added noise has correlation structure $\mathcal{N}(0, C)$ which is close to $\mudata$, while the added noise in WNDM has structure $\mathcal{N}(0, I)$. Therefore, adding this noise to all samples right before comparing them would have given an unfair disadvantage to WNDM.

\item Furthermore, we added $\varepsilon_\text{reg} \text{Id}$ onto the covariance matrices for numerical stability, where $\varepsilon_\text{reg} = 0.0001$. 

\item Our experiments were implemented in JAX, and we used the U-Net architecture from \citet{song2021scorebased} for the neural network.
\end{enumerate}


\section{Negative Sobolev Wasserstein Distances}\label{sec:negative_sobolev_spaces}
We briefly and heuristically explore what it means if $\mu$ and $\pi$ are close in $\mathcal{W}_2$ when the underlying norm is a negative Sobolev norm. Denote by $H^\alpha$ the spaces defined in Section \ref{sec:family_gaussian_measures}. Now let $f \in H^{\alpha}$. Assume that $X, Y$ form a $\mathcal{W}_2^{-\alpha}$-optimal coupling, i.e., $X \sim \mu$ and $Y \sim \pi$ and that
\[
    \mathbb{E}[\|X - Y\|_{-\alpha}] \le \mathcal{W}_2^{-\alpha}(\pi, \mu).
\]
Now,  $H^{-\alpha}$ can be viewed as the dual of $H^\alpha$ and therefore we can evaluate $X$ or $Y$ on $f$. Then,
\[
    \mathbb{E}[|X(f) - Y(f)|] \le \|f\|_{\alpha} \mathbb{E}[\|X - Y\|_{-\alpha}] \le \|f\|_{\alpha} \mathcal{W}^{-\alpha}_2(\mu, \pi)
\]
Therefore, we can expect the evaluations of $X$ and $Y$ on test-functions from $H^\alpha$ to be close. The larger $\alpha$ gets, the fewer test functions are in $H^\alpha$. Note that for any $\alpha\ge 0$ (and in particular for our typical case $H^0 = L^2$), $H^\alpha$ will not contain point evaluations, and therefore we cannot expect point evaluations of $X$ and $Y$ to be close (in case they are  well-defined).

\section{Exponential Integrator}\label{sec:exponential_integrator}
The exponential integrator \citep{certaine1960solution} is derived by splitting the following SDE
\begin{equation}
	\mathd Y_t = \frac{1}{2} Y_t + s(t, Y_t) \mathd t + \sqrt{C}\mathd W_t = {{\frac{1}{2} Y_t}} + {s(t, Y_t) \mathd t} + \sqrt{C}\mathd W_t
	\label{equ:sde_linear_nonlinear}
\end{equation}
into the linear and {nonlinear} part.
The exact solution is then given by
\begin{align*}
	Y_{t + \Delta t} &= e^{t/2} Y_t 
	= e^{t/2} Y_t + (e^{\Delta t / 2} - 1) s(t, Y_t) + \sqrt{e^{\Delta t} - 1} \xi.
\end{align*}
The exponential integrator was first applied to diffusion models in \citet{zhang2023fast}. 

\section{Proofs for Section \ref{sec:inf_dimensional_forward_backward_sde}}\label{appendix technical lemmas}
\subsection{Proofs for Section \ref{section score}}\label{sec:proofs_score}
We first prove Lemma \ref{lemma score rewrite}

\begin{proof}
    In finite dimensions we can explicitly write $p_t$ as
    \begin{equation}\label{equ:pt_finite_dim}
        p_t(x) =   \int p_{t | 0}(x | x_0) \mathd \mudata (x_0)
    \end{equation}
    where $p_{t|0}$ is the time $t$-transition kernel of the forward SDE, given by
    \[
p_{t | 0}(x | x_0) = \frac{1}{\sqrt{(2\pi v_t)^D\text{det}(C)}} \exp \left( - \frac{1}{2 (1 - e^{- t})} \left\langle \left( x - e^{- \frac{t}{2}} x_0 \right), C^{- 1} \left( x - e^{- \frac{t}{2}} x_0 \right)
    \right\rangle_H^2 \right).
    \]
    We can exchange the derivative with the integral by Leibniz rule since the derivative of the integrand is bounded.
    Therefore, we have that
  \begin{eqnarray*}
    \nabla \log p_t(x)
    & = & \frac{1}{p_t (x)} \nabla_x \int p_{t | 0}(x | x_0) \mathd \mudata (x_0)\\
    & = & -\frac{1}{(1 - e^{- t})}  \int C^{- 1} \left( x - e^{- \frac{t}{2}}
    x_0 \right)  \frac{p_{t| 0} (x|x_0)}{p_t (x)} \mathd \mudata
    (x_0)\\
    & = & -\frac{1}{(1 - e^{- t})} C^{- 1} \left( \mathbb{E} \left[ X_t - e^{-
    \frac{t}{2}} X_0 |X_t = x \right] \right).
  \end{eqnarray*}
  In the last equation we
  used the formula for the conditional density; see \citet[Section 4.1.c]{Durrett2005probability}.
\end{proof}

We now prove Lemma \ref{lemma:reverse_martingale}.

\begin{proof}
	We first treat continuity. Since $p_t$ can be written as the convolution of $\mudata$ with a Gaussian kernel, we know that it is smooth in space and time on $(0, \infty]$. Furthermore, it holds that $p_t > 0$ everywhere. Due to that, we can deduce that $\nabla \log p_t = \frac{\nabla p_t}{p_t}$ is continuous in $t$. Since $X_t$ is also continuous in time, we get that $\nabla \log p_t(X_t)$ is time-continuous. 
	
	Now we prove the reverse-time martingale property. Since we can write
	\begin{eqnarray*}
		p_t (x_t) & = & \int p_s (x_s) p_{t|s} (x_t |x_s) \mathd x_s
	\end{eqnarray*}
	and by using since the transition kernel $p_{t|s}(x_t|x_s)$ is given by $\mathcal{N}(e^{-(t-s)/2}x_s, \frac{1}{1 - e^{-(t-s)}}C)$,
	\begin{eqnarray*}
    & & \nabla p_t (x_t)\\
		 & = & \nabla_{x_t} \int \exp \left( - \frac{1}{2 (1 -
			e^{- (t - s)})} \left\langle x_t - e^{- \frac{(t - s)}{2}} x_s, x_t - e^{-
			\frac{(t - s)}{2}} x_s \right\rangle \right) p_s (x_s) \mathd x_s\\
		& = & \int e^{\frac{(t - s)}{2}} \nabla_{x_s} \exp \left( - \frac{1}{2 (1
			- e^{- (t - s)})} \left\langle x_t - e^{- \frac{(t - s)}{2}} x_s, x_t -
		e^{- \frac{(t - s)}{2}} x_s \right\rangle \right) p_s (x_s) \mathd x_s\\
		& = & \int e^{\frac{(t - s)}{2}} \exp \left( - \frac{1}{2 (1 - e^{- (t -
				s)})} \left\langle x_t - e^{- \frac{(t - s)}{2}} x_s, x_t - e^{- \frac{(t
				- s)}{2}} x_s \right\rangle \right) \nabla_{x_s} p_s (x_s) \mathd x_s\\
		& = & \int e^{\frac{(t - s)}{2}} p_{t|s} (x_t |x_s) p_s (x_s)
		\nabla_{x_s} \log p_s (x_s) \mathd x_s \\
        & = & e^{\frac{(t - s)}{2}} \int p_{s,
			t} (x_s, x_t) \nabla_{x_s} \log p_s (x_s) \mathd x_s .
	\end{eqnarray*}
	Since $\nabla \log p_t (x_t) = \frac{\nabla p_t (x_t)}{p_t (x_t)}$ and
	$p_{s|t} (x_s |x_t) = \frac{p_{s, t} (x_s, x_t)}{p_t (x_t)}$, we get that
	\begin{eqnarray*}
    & & \nabla \log p_t (x_t) \\
		 & = & e^{\frac{(t - s)}{2}} \int \frac{p_{s, t}
			(x_s, x_t)}{p_t (x_t)} \nabla_{x_s} \log p_s (x_s) \mathd x_s =
		e^{\frac{(t - s)}{2}} \int p_{s|t} (x_s |x_t) \nabla_{x_s} \log p_s (x_s)
		\mathd x\\
		& = & e^{\frac{(t - s)}{2}} \mathbb{E} [\nabla \log p_s (X_s) |X_t = x_t] .
	\end{eqnarray*}
	
\end{proof}

The above calculations have already been done in \citet{chen2022improved} to
bound the difference $\mathbb{E} [\| \nabla \log p_t (x_t) - \nabla \log p_s
(x_s) \|^2]$.

\subsection{Proofs for Section \ref{section loss}}\label{sec:proofs_loss}
We start by proving Lemma \ref{lemma score matching}. 

\begin{proof}
	Let $e_i$ be a basis of $K$ and $K^D = \tmop{span} \langle e_1, \ldots, e_D
	\rangle$. We denote by $P^D$ the projection onto $K^D$ and by $X_t^D = P^D
	X_t$ the projection of $X_t$ onto $K^D$. We will denote by
	\[ \| \cdot \| = \| \cdot \|_K \]
	throughout this proof. Let
	\[ s^D = P^D \mathbb{E} [s (t, X_t) |X_t^D] =\mathbb{E} [\sigma_t^{- 1}
	(X_t^D - e^{- t / 2} X_t^D) |X_t^D], \]
	where $\sigma_t^{- 1} = \frac{1}{\sqrt{1 - e^{- t}}}$. We have that
	\begin{eqnarray*}
		\mathbb{E} [\| s^D \|^2 ] & \leqslant & \mathbb{E} [\| \sigma_t^{- 1}
		(X_t^D - e^{- t / 2} X_0^D) \|^2 ] = \sigma_t^{- 2} \mathbb{E} [\|
		\mathcal{N} (0, P^D CP^D) \|^2 ] < \infty
	\end{eqnarray*}
	since the right hand side is the expectation of the norm of a finite
	dimensional Gaussian, which is finite. Let $\tilde{s}^D (t, X_t) = P^D
	\mathbb{E} [\tilde{s} (t, X_t ) |X_t^D]$.
	
	Now,
	\begin{eqnarray*}
		\mathbb{E} [\| s^D - \tilde{s}^D \|_K^2] & = & \mathbb{E} [\| P^D
		(\mathbb{E} [\tilde{s} (t, X_t ) |X_t^D] -\mathbb{E} [s (t, X_t) |X_t^D])
		\|^2 ]\\
		& \leqslant & \mathbb{E} [\| \mathbb{E} [\tilde{s} (t, X_t ) - s (t, X_t)
		|X_t^D] \|^2 ]
		  \leqslant  \mathbb{E} [\| s (t, X_t) - \tilde{s} (t, X_t) \|^2 ] <
		\infty,
	\end{eqnarray*}
	and
	\begin{eqnarray*}
		\mathbb{E} [\| \tilde{s}^D \|^2] & \leqslant & 2 (\mathbb{E} [\| s^D -
		\tilde{s}^D \|^2] +\mathbb{E} [\| s^D \|^2]) < \infty .
	\end{eqnarray*}
	Therefore, we get that
	\begin{eqnarray*}
		\mathbb{E} [\| s^D - \tilde{s}^D \|^2] & = & \mathbb{E} [\| s^D \|^2 + \|
		\tilde{s}^D \| - \langle s^D, \tilde{s}^D \rangle] =\mathbb{E} [\| s^D
		\|^2] +\mathbb{E} [\| \tilde{s}^D \|] - 2\mathbb{E} [\langle s^D,
		\tilde{s}^D \rangle],
	\end{eqnarray*}
	where we used that all the terms are finite in the last equality, so that we
	are no adding and subtracting infinities. Now, for
	\begin{eqnarray*}
		\mathbb{E} [\langle s^D, \tilde{s}^D \rangle]  =  \mathbb{E} [\langle
		\mathbb{E} [\sigma_t^{- 1} (X_t^D - e^{- t / 2} X_0^D) |X_t^D],
		\tilde{s}^D \rangle]
		  =  \mathbb{E} [\langle \sigma_t^{- 1} (X_t^D - e^{- t / 2} X_0^D),
		\tilde{s}^D \rangle],
	\end{eqnarray*}
	and therefore, since $\mathbb{E} [\| X_t^D - e^{- t / 2} X_0^D \|^2]$ is
	finite we can do a zero-addition of $\mathbb{E} [\| X_t^D - e^{- t / 2}
	X_0^D \|^2]$ and get that
	\begin{equation*}
		\mathbb{E} [\| s^D - \tilde{s}^D \|^2] = \mathbb{E} [\| \tilde{s}^D -
		\sigma_t^{- 1} (X_t^D - e^{- t / 2} X_0^D) \|^2] +\mathbb{E} [\| s^D \|^2]
		-\mathbb{E} [\| \sigma_t^{- 1} (X_t^D - e^{- t / 2} X_0^D) \|^2].
	\end{equation*}
	By Lemma \ref{lemma:approximations converge} we see that the left hand
	side converges to $\mathbb{E} [\| s - \tilde{s} \|^2]$. Therefore we get
	that $\mathbb{E} [\| \tilde{s}^D - \sigma_t^{- 1} (X_t^D - e^{- t / 2}
	X_0^D) \|^2]$ converges to something finite, if and only if
	\[ \mathbb{E} [\| s^D \|^2] -\mathbb{E} [\| \sigma_t^{- 1} (X_t^D - e^{- t
		/ 2} X_0^D) \|^2] \]
	does. For this term we get that since
	\begin{eqnarray*}
		\mathbb{E} [\langle s^D, \sigma_t^{- 1} (X_t^D - e^{- t / 2} X_0^D)
		\rangle] & = & \mathbb{E} [\langle \mathbb{E} [\sigma_t^{- 1} (X_t^D -
		e^{- t / 2} X_0^D) |X_t^D], \sigma_t^{- 1} (X_t^D - e^{- t / 2} X_0^D)
		\rangle]\\
		& = & \mathbb{E} [\| \mathbb{E} [\sigma_t^{- 1} (X_t^D - e^{- t / 2}
		X_0^D) |X_t^D] \|^2]\\
		& = & \mathbb{E} [\| s^D \|^2],
	\end{eqnarray*}
	we can deduce
	\begin{eqnarray*}
		&  & \mathbb{E} [\| s^D \|^2] -\mathbb{E} [\| \sigma_t^{- 1} (X_t^D -
		e^{- t / 2} X_0^D) \|^2]\\
		& = & -\mathbb{E} [\| s^D \|^2] + 2\mathbb{E} [\langle s^D, \sigma_t^{-
			1} (X_t^D - e^{- t / 2} X_0^D) \rangle] -\mathbb{E} [\| \sigma_t^{- 1}
		(X_t^D - e^{- t / 2} X_0^D) \|^2]\\
		& = & -\mathbb{E} [\| s^D - \sigma_t^{- 1} (X_t^D - e^{- t / 2} X_0^D)
		\|^2] \rightarrow_{L^2} -\mathbb{E} [\| s - \sigma_t^{- 1} (X_t - e^{- t /
			2} X_0) \|^2]
	\end{eqnarray*}
	where the last convergence is implied by Proposition
	\ref{lemma:approximations converge}. The last result now follows by rewriting
	\[ V_t =\mathbb{E} [\| \mathbb{E} [\sigma_t^{- 1} (X_t - e^{- t / 2} X_0)
	|X_t] - \sigma_t^{- 1} (X_t - e^{- t / 2} X_0) \|^2] \]
	and using that $X_t$ can be pulled out of the conditional expectation.
	
	\ 
\end{proof}

Now we prove Lemma \ref{lemma:special_cases_score_matching}:

\begin{proof}
	\tmtextbf{Item 1:} We know from Lemma \ref{lemma score matching} that
	$\tmop{DSM}$ is finite if and only if $V$ is finite. Therefore, we will
	prove that $V$ is finite:
	\[ V_t = \frac{e^{- t}}{1 - e^{- t}} \mathbb{E} [\| X_0 -\mathbb{E} [X_0
	|X_t] \|^2_U] \]
	Now,
	\begin{eqnarray*}
		&  & \mathbb{E} [\| X_0 -\mathbb{E} [X_0] \|^2_U]\\
		& = & \mathbb{E} [\| X_0 -\mathbb{E} [X_0 |X_t] +\mathbb{E} [X_0 |X_t]
		-\mathbb{E} [X_0] \|^2_U]\\
		& = & \mathbb{E} [\| X_0 -\mathbb{E} [X_0 |X_t] \|^2_U] +\mathbb{E} [\|
		\mathbb{E} [X_0 |X_t] -\mathbb{E} [X_0] \|^2_U] \\
        & &+ \mathbb{E} [\langle X_0
		-\mathbb{E} [X_0 |X_t], \mathbb{E} [X_0 |X_t] -\mathbb{E} [X_0]
		\rangle_U]\\
		& = & \mathbb{E} [\| X_0 -\mathbb{E} [X_0 |X_t] \|^2_U] 
        +\mathbb{E} [\|\mathbb{E} [X_0 |X_t] -\mathbb{E} [X_0] \|^2_U],
	\end{eqnarray*}
	where the last term drops by taking the conditional expectation with respect
	to $X_t$. Therefore,
	\[ \mathbb{E} [\| X_0 -\mathbb{E} [X_0 |X_t] \|_U^2] \leqslant \mathbb{E}
	[\| X_0 -\mathbb{E} [X_0] \|_U^2] < \infty \]
	and so $V_t$ is finite.
	
	\tmtextbf{Item 2:} In this case we use that we can also write $V_t$ as
	\begin{eqnarray*}
		V_t & = & \frac{1}{1 - e^{- t}} \mathbb{E} [\| X_t - e^{- t / 2} X_0
		-\mathbb{E} [X_t - e^{- t / 2} X_0 |X_t] \|_H^2]
	\end{eqnarray*}
	Similarly as above,
	\begin{eqnarray*}
		\mathbb{E} [\| X_t - e^{- t / 2} X_0 -\mathbb{E} [X_t - e^{- t / 2} X_0
		|X_t] \|_H^2] & \leqslant & \mathbb{E} [\| X_t - e^{- t / 2} X_0 \|_H^2]
	\end{eqnarray*}
	where we used that $\mathbb{E} [X_t - e^{- t / 2} X_0] = 0$. Now, since $X_t
	- e^{- t / 2} X_0 \sim \mathcal{N} (0, (1 - e^{- t}) C)$, which is supported
	on $H$, the above expectation is finite.
\end{proof}

We used the following lemma in the proofs of the two above lemmas:

\begin{lemma}
	\label{lemma:approximations converge}Let $(K, \langle \cdot, \cdot \rangle)$
	be a separable Hilbert space, and $Z, \tilde{Z}$ random variables taking values in         $K$.  
	Let $e_i$ be an orthonormal basis of $K$. Denote by $K^D = \tmop{span}
	\langle e_1, \ldots, e_D \rangle$ and by $P^D$ the projection onto $K^D$.
	Furthermore, let $Z^D$ be given by $Z^D = P^D \mathbb{E} [Z|P^D \tilde{Z}]$.
	Then, if $\mathbb{E} [\| \mathbb{E} [Z| \tilde{Z}] \|^2_K] < \infty$, $Z^D
	\rightarrow \mathbb{E} [Z| \tilde{Z}]$ in $L^2$ and almost surely.
\end{lemma}

\begin{proof}
	We have that
	\begin{eqnarray}
		\mathbb{E} [\| Z^D -\mathbb{E} [Z| \tilde{Z}] \|^2_K] & = & \mathbb{E} [\|
		P^D (\mathbb{E} [Z| \tilde{Z}^D] -\mathbb{E} [Z| \tilde{Z}]) \|^2_K]
		+\mathbb{E} [\| (I - P^D) \mathbb{E} [Z| \tilde{Z}] \|^2_K] \nonumber\\
		& \leqslant & \mathbb{E} [\| \mathbb{E} [Z| \tilde{Z}^D] -\mathbb{E} [Z|
		\tilde{Z}] \|^2_K] +\mathbb{E} [\| (I - P^D) \mathbb{E} [Z| \tilde{Z}]
		\|^2_K]  \label{equ:K converges upper bound}
	\end{eqnarray}
	The cross term in the first equality is $0$ since $P^D$ is the orthogonal
	projection. The first term in the \eqref{equ:K converges upper bound}
	converges to $0$, since $\mathbb{E} [Z| \tilde{Z}^D] =\mathbb{E} [\mathbb{E}
	[Z| \tilde{Z}] | \tilde{Z}^D]$ is a family of conditional expectations of
	the $L^2$-random variable $\mathbb{E} [Z| \tilde{Z}]$. The result follows by
	the $L^2$-martingale convergence theorem. The second term converges to
	$0$ since
	\[ \mathbb{E} [\| \mathbb{E} [Z| \tilde{Z}] \|^2_K] = \sum_{d = 1}^{\infty}
	\mathbb{E} [\langle e_i, \mathbb{E} [Z| \tilde{Z}] \rangle_K^2] < \infty
	. \]
	But $\mathbb{E} [\| (I - P^D) \mathbb{E} [Z| \tilde{Z}] \|^2_K]$ is equal to
	$\sum_{d = D}^{\infty} \mathbb{E} [\langle e_i, \mathbb{E} [Z| \tilde{Z}]
	\rangle^2_K]$, which converges to $0$ since the full sum is finite. 
 
        Furthermore, we can write
        \begin{align*}
            \| Z^D -\mathbb{E} [Z| \tilde{Z}] \|_K
            &\leqslant \|P^D (\mathbb{E} [Z| \tilde{Z}^D] -\mathbb{E} [Z| \tilde{Z}]) \|^2_K
		+ \| (I - P^D) \mathbb{E} [Z| \tilde{Z}] \|^2_K \\
            &\leqslant \|\mathbb{E} [Z| \tilde{Z}^D] -\mathbb{E} [Z| \tilde{Z}]\|^2_K
		+ \| (I - P^D) \mathbb{E} [Z| \tilde{Z}] \|^2_K
        \end{align*}
        The second term on the right-hand vanishes as $D \to \infty$ since $\mathbb{E}[Z | \tilde{Z}] \in K$. The first term almost surely converges to $0$ due to the almost sure martingale convergence theorem.
\end{proof}

\section{Existence Proof}\label{section:proof_theorem_well_defined}
\subsection{Spectral Approximation of $C$}\label{sec:spectral_approx}
Let $\nu =\Normal (0, C)$ be a Gaussian measure with values in $(H,
\langle \cdot, \cdot \rangle_H)$. $C$ has an orthonormal basis $e_i$ of
eigenvectors and corresponding non-negative eigenvalues $c_i \geqslant 0$,
i.e.,
\begin{eqnarray*}
	Ce_i & = & c_i e_i .
\end{eqnarray*}
We define the linear span of the first $D$ eigenvectors as
\[ H^D = \left\{ \sum_{i = 1}^D f_i e_i |f_1, \ldots, f_D \in \mathbb{R}
\right\} \subset H \]
Let $P^D : H \rightarrow H^D$ be the orthogonal projection onto $H^D$. If we write an
element $f$ of $H$ as
\[ f = \sum_{i = 1}^{\infty} \langle f, \varphi_i \rangle_H \varphi_i, \]
$P^D$ is equivalent to restricting $f$ to its first $D$ coefficients:
\begin{eqnarray*}
	P^D : H \rightarrow H^D, 
	\quad f \mapsto \sum_{i = 1}^D \langle f, \varphi_i \rangle_H \varphi_i .
\end{eqnarray*}
The push-forwards $(P^D)_{\ast} \nu$ of $\nu$ under $P^D$ are denoted by
\begin{eqnarray*}
	\nu^D & : = & (P^D)_{\ast} \nu,
    \quad \text{ where } \quad (P^D)_{\ast} \nu (A) = \nu ((P^D)^{- 1} (A)) . 
\end{eqnarray*}
It is a Gaussian measure with covariance operator $P^D CP^D$.

By sending $v \in \mathbb{R}^D$ to $\hat{v} = v_1 e_1 + \cdots + v_D e_D$ we
can identify $H^D$ with $\mathbb{R}^D$. Under these identifications, $\nu^D$
would have distribution $\Normal (0, C^D)$ on $\mathbb{R}^D$, where
C\tmrsup{D} is a diagonal matrix with entries $c_1, \ldots, c_D$.

\subsection{Spectral Approximation of the SDEs}\label{sec:spectral_approximation_sdes}

We define the finite-dimensional approximations of $\mudata$ by
$\mudata^D = (P^D)_{\ast} \mudata$.
We discretize the forward SDE \eqref{forward SDE inf} by
\begin{eqnarray}
	\mathd X_t^D = - \frac{1}{2} X_t^D \mathd t + \sqrt{P^D CP^D} \mathd W_t
	\label{forward SDE approx},
	\quad X_{0 }^D \sim \mudata^D
\end{eqnarray}
Since $\mudata$ is supported on $H^D, P^D CP^D$ projects the noise
down to $H^D$, and the operation $X_t \rightarrow - \frac{1}{2} X_t$ keeps
$H^D$ invariant, $X_t^D$ will stay in $H^D$ for all times. Therefore we can
view $X_t^D$ as process on $\mathbb{R}^D$ and define the Lebesgue densities $p_t^D$ of $X_t^D$ there.

\subsection{Proof of Theorem \ref{theorem SDEs well defined}}
We can now prove Theorem \ref{theorem SDEs well defined}:

\begin{proof}
	The forward SDE is just a standard Ornstein--Uhlenbeck process and existence
	and uniqueness of that is standard; see, for example, \citet[Theorem 7.4]{da2014stochastic}. We will now
	show that the time reversal
	\[
	Y_t \coloneqq X_{T-t}
	\]
	is a solution to \eqref{reverse SDE inf abs}.
	
	The solution to the forward SDE is given as the stochastic convolution,
	\begin{eqnarray*}
		X_t  & = & e^{- t} X_0  + \int_0^t e^{- (t - s)} \sqrt{C} \mathd W_s.
	\end{eqnarray*}
	The processes $X_t^D \coloneqq P^D (X_t)$ (see Section \ref{sec:spectral_approx}) are solutions to \eqref{forward SDE approx},
	since the SDE coefficients are decoupled. We now show that they
	converge to $X_t$ almost surely in the supremum norm. We define
	\[ X_t^{D : \infty} = X_t - X_t^D . \]
	Then
	\begin{eqnarray*}
		X_t^{D : \infty} & = & e^{- t} X_0^{D : \infty} + \int_0^t e^{- (t - s)} \sqrt{C}
		\mathd W_s^{D : \infty},
	\end{eqnarray*}
	where $W_s^{D : \infty}$ is the projection of $W_s$ onto $\tmop{span} \{ e_D, e_{D+1}
	\ldots  \}$. It holds that
	\begin{eqnarray*}
		\mathbb{E} [\sup_{t \leqslant T} \| X_t^{D : \infty} \|_H^2] & \leqslant & 4
		e^{- 2 t} \mathbb{E}[\| X_0^{D : \infty} \|^2_H] + 4 (1 - e^{- t}) \sum_{i = D}^\infty c_i \to 0
	\end{eqnarray*}
	for $D \rightarrow \infty$, where we used Doob's $L^2$ inequality to bound
	the stochastic integral. The first term will converge to $0$ almost surely, since $X_0$ is $H$-valued and therefore the sum $\|X_0\|_H^2 = \sum_{i=1}^\infty \langle X_0, \varphi_i \rangle^2$ is almost surely finite, where the $\varphi_i$ are defined in Section \ref{sec:spectral_approx}. Therefore, $\| X_0^{D : \infty} \|^2_H = \sum_{i=D}^\infty \langle X_0, \varphi_i \rangle^2$ will almost surely converge to zero. An analogous argumentation holds for the second term since $\sum_{i = 1}^{\infty} c_i$ is finite because $C$ is trace-class on $H$. 
	
	Denote by $E$ the Banach space of continuous, $H$-valued
	paths with the supremum norm, $E = C([0, T], H)$. Then we can view $X_\cdot$ as an
	$E$-valued random variable; see \citet[Theorem 4.12]{da2014stochastic}. Then we have
	just proven that $X^N$ converges to $X$ in $L^2(\Omega, E)$.
	
	We denote the time-reversals of $X_t^N$ by $Y_t^N \coloneqq X_{T-t}^N$. These 
	converge to their infinite-dimensional 
	counterpart $Y_t = X_{T-t}$ in the same way as $X_t^N$ converge to $X_t$. The main difficulty is
	to show $Y_t$ solves the SDE \eqref{reverse SDE inf abs}. We do this using an approximation argument.
	
	We define $p_t^D$ as in Section \ref{sec:spectral_approximation_sdes}. Furthermore, we define the finite-dimensional time reversals of $X_t^D$ as $Y_t^D = X_{T-t}^D$. \jp{By Proposition \ref{prop:finite_dim_sde_satisfies_reverse}, we know that they satisfy 
	\[ Y_t^D - Y_0^D  - \frac{1}{2} \int_0^t Y_r^D \mathd r - \int_0^t s_{T-r}^D \mathd r = \sqrt{P^D C P^D} B_t^D, \]
    for a $H^D$-Brownian motion $B_{D, t}$. 
    It is important to note, that $B_D$ will not be equal to the projection of $B_{D+L}$ onto $H^D$ in general, and the same hold for the $\nabla \log p_t^D$. However, as we will see now, we can prove some martingale-like properties for them to obtain convergence to their infinite-dimensional counterpart.}
        
        By, Lemma \ref{lemma score rewrite}, we know that we can replace the $C\nabla \log p_t$ term by a conditional expectation $s^D$,
	\[
	s^D(t, x^D) = \frac{1}{1 - e^{-t}} \mathbb{E}[X_t^D - e^{-\frac{t}{2}}X_0^D | X_t^D = x^D].
	\]
	However, since the forward SDE decouples, we can also write the above conditional expectation in terms of the infinite-dimensional process $X_t$:
	\[ s^D(t, x^D)
	= \frac{1}{1 - e^{- t}}  P^D \mathbb{E} \left[ X_t - e^{-
		\frac{t}{2}} X_0 | P^D X_t = x^D \right], \]
	where we use that the projections $P^D X_t$ are solutions to \eqref{forward SDE approx}.
	In particular, due to the tower property of conditional expectations,
	\[ s^D_t = s^D(t, X_t^D) = P^D \mathbb{E} [s (t,
	X_t) |X_t^{D}]. \]
        By Lemma \ref{lemma:approximations converge}, which makes use of the fact that $\mathbb{E}[s(t, X_t)| P^D X_t]$ is a martingale in $D$, the $s^D_t$ converge to $s(t, X_t)$ in $L^2$. Furthermore, by Lemma \ref{lemma:reverse_martingale}, we know that the $e^{-t/2} s_t^D$ form a reverse-time martingale in $t$. However, due to Doob's $L^2$-inequality, we get that for any $\epsilon > 0$,
        \begin{equation*}
            \mathbb{E}[\sup_{\varepsilon \leqslant t \leqslant T} \|s_t^D - s_t^L\|^2] 
            \leqslant e^T \mathbb{E}[\|s_\varepsilon^D - s_\varepsilon^L\|^2].
        \end{equation*}
        The right hand side is Cauchy and therefore is the left-hand side is too. Therefore, the convergence of $s_t^D$ to $s_t$ in $L^2$ is uniform on $[\varepsilon, T]$, i.e., continuous martingales $s_t^D$ form a Cauchy sequence in the norm 
	\[
	\|N\|_{I, \epsilon} = \mathbb{E}[\sup_{t \ge \epsilon} |N_t|^2].
	\]
	The continuous martingales are closed with respect to that norm (see \citet[Section 1.3]{karatzas1991brownian}), and $s_t$ is also a continuous martingale on $[\varepsilon, T]$. Since $\epsilon$ was arbitrary, we have shown that $s(t, X_t)$ is a continuous local martingale (in reverse time) up to $t = 0$. 
	
	Furthermore, since all the terms on the left-hand side converge in $L^2$, uniformly in $t$,
	so does the right-hand side. The right-hand side is $P^D CP^D$ Brownian
	motion for each $D$. Using again the that the spaces of martingales is closed
	and furthermore the Levy characterization of Brownian motion, we find that
	the $B_t^D$ have to converge to a $C$-Brownian motion $B_t$. Therefore,
	\begin{equation}
		Y_t  =  Y_0 + \frac{1}{2} \int_0^t Y_r \mathd r + \int_0^t s_{T-r} \mathd r + B_t 
	\end{equation}
    is indeed a weak solution to \eqref{reverse SDE inf abs}.
    \jpp{It is a \emph{weak} solution since $Y_s$ is not necessarily measurable with respect to the filtration generated by $B_s$. In general, it can even be the other way around, see Proposition \ref{prop:finite_dim_sde_satisfies_reverse}.}
\end{proof}

\begin{proposition}
    Let $X_t$ be a solution to \eqref{inf-fwd}. Assume $H = \mathbb{R}^D$. Then, the time-reversal $Y_t = X_{T-t}$ of the SDE  satisfies the SDE
	\begin{equation}
		\label{eq:time-reversal-fd}\mathd Y_t = \frac{1}{2} Y_t \mathd t + C \nabla \log p_{T-t}(Y_t) \mathd t + \sqrt{C} \mathd B_t.
	\end{equation}
 Here, $B_t$ is a different Brownian motion $B_t$ to $W_t$. \jp{If $C$ has full rank, $B_t$ can be defined on the same probability space as $X_t$ itself.}
	\label{prop:finite_dim_sde_satisfies_reverse}
\end{proposition}
\begin{proof}
	The above is the usual time-reversal formula. All we need to show is that for the special case of the forward SDE \eqref{inf-fwd} the conditions from \citet{haussmann1986time} are always satisfied. 
	Assumption $(A)(i)$ in \citet{haussmann1986time} is satisfied since $b$ and $\sigma$ are linear. Assumption $(A)(ii)$ is that for each $t_0 > 0$, it holds that
	\begin{enumerate}
		\item $\int_{t_0}^T \int_{B_R} | p (t, x) |^2 \mathd x < \infty$, and
		
		\item $\int_{t_0}^T \int_{B_R} | \partial_{x_i} p (t, x) |^2 \mathd x <
		\infty$ for all $i$m
	\end{enumerate}
	where $B_R$ is the ball of Radius $R$ on $\mathbb{R}^D$ and $p(t,x)$ is the Lebesgue-density of $\mathbb{P}_t$. We now prove that both of these conditions hold.  
	We have the explicit formula
	\[ p (t, x) = \frac{1}{\sqrt{2 \pi (v_t \det C)^D}} \int e^{- \frac{\| x -
			x_0 \|_U^2}{2}} \mathd \mu_{\tmop{data}} (x_0) \]
	and therefore, in particular
	$| p (t, x) | \leqslant \frac{}{} \frac{1}{\sqrt{2 \pi (v_t \det C)^D}}$.
	Since $\frac{1}{v_t} = \frac{1}{1 - e^{- t}}$ is integrable on $[t_0, T]$ for
	$t_0 > 0$, this implies $1.$	
	Furthermore, we get that
	\[ \nabla p (t, x) = \frac{1}{\sqrt{2 \pi (v_t \det C)^D}} \int C^{- 1} x 
	e^{- \frac{\| x - x_0 \|_U^2}{2}} \mathd \mu_{\tmop{data}} (x_0), \]
	where we used the Leibniz rule to exchange differentiation and integration,
	since the integrand is bounded in $x_0$.	
	On $B_R$ this can be upper bounded by $\| \nabla p (t, x) \| \leqslant  \frac{\| C^{- 1} \| R}{\sqrt{2 \pi (v_t \det C)^D}}$
	which is again integrable on $[t_{0,} T]$. 

    \jp{By \citet{haussmann1986time} we know that the time reversal $Y_t$ will have the same generator as the SDE \eqref{eq:time-reversal-fd}. In particular 
    $M_t = Y_t - Y_0 - \int_0^t \frac{1}{2} Y_r + C\nabla \log p_{T-r}(Y_r) \mathd r$ is a continuous martingale with quadratic variation $C$ with respect to the canonical filtration of $Y_s$}. 
    
    \jp{We want to apply the Martingale representation theorem to express this martingale in terms of a Brownian motion. In general, one might need to extend the probability space to do so. However, since $C$ has full rank, we can express the Brownian motion as $B_t = C^{-1/2} M_t$, which is defined on the same probability space as $Y_t$.}
\end{proof}

\section{Uniqueness Proofs}\label{sec:uniqueness_proofs}

\subsection{Proof of Theorem \ref{thm:uniqueness_manifold}}\label{sec:uniqueness_manifold}
First we prove Theorem \ref{thm:uniqueness_manifold}:

\begin{proof}
	\paragraph{Step 1: Prove that $s$ is locally Lipschitz with respect to the Cameron-Martin Norm}
	Recall that
	\begin{equation}
		s(t, x) = -\frac{1}{1 - e^{- t}} x + \frac{e^{-\frac{t}{2}}}{1 - e^{-t}} \mathbb{E} [X_0 |X_t = x].
		\label{equ:s_as_sum_with_relative_densitites}
	\end{equation}
	The $-\frac{1}{1 - e^{- t}} x$ term is Lipschitz for any $t \in [\epsilon, T]$. Therefore, our goal is to show that $\mathbb{E}[X_0 | X_t = x]$ is Lipschitz in $x$ too. We will frequently use that for $u \in U$, we can write the Radon-Nikodym derivative of $\mathcal{N}(u, v_t C)$ with respect to $\mathcal{N}(0, v_t C)$ as
	\[
	\frac{\mathd \mathcal{N}(u, v_t C)}{\mathd \mathcal{N}(0, v_t C)}(x_t) = \exp\left(\frac{\langle u, x_t\rangle_U - \|u\|^2_U}{v_t}\right),
	\]
	by the Cameron-Martin theorem (see, for example, \citet{hairer2009spde}).
	Here $v_t$ is a shorthand notation for
	\[
	v_t = 1 - e^{-t}.
	\]
	To simplify notation, we will define 
	\[
	n(x_0, x_t) \coloneqq \frac{\mathd \mathcal{N}(e^{-t} x_0, v_t C)}{\mathd \mathcal{N}(0, v_t C)}(x_t).
	\]
	Then, the joint distribution of $X_0$ and $X_t$ is given by
	\[
	\mathd n(x_0, x_t) = \mathd (\mathcal{N}(0, v_t C)(x_t) \otimes \mu_\text{data}(x_0)).
	\]
	This can be seen by the following calculation:
	\begin{align*}
		&\int_A \int _B n(x_0, x_t) ~ \mathd \mathcal{N}(0, v_tC)(x_t) \mathd \mu_\text{data}(x_0) \\
		=& \int_A \int _B \frac{\mathd \mathcal{N}(e^{-t} x_0, v_t C)}{\mathd \mathcal{N}(0, v_t C)}(x_t) \mathd \mathcal{N}(0, v_tC)(x_t) \mathd \mu_\text{data}(x_0) \\
		=& \int_A \int _B \mathd \mathcal{N}(e^{-t} x_0, v_tC)(x_t) \mathd \mu_\text{data}(x_0) \\
		=& \mathbb{P}[X_0 \in A, X_t \in B],
	\end{align*}
	where we used that $\mathcal{N}(e^{-t}x_0, v_t C)$ is the transition kernel of the forward SDE \eqref{forward SDE inf}. We show that
	\[
	f(x_t) = \frac{\int x_0 n(x_0, x_t) \mathrm{d}\mu_\text{data}(x_0)}
	{\int n(x_0, x_t) \mathrm{d}\mu_\text{data}(x_0)}
	\]
	is a version of the conditional expectation $\mathbb{E}[X_0 | X_t = x]$.
	The function $f$ is $\sigma(X_t)$ measurable by Fubini's theorem. Furthermore, for $A \in \sigma(X_t)$,
	\begin{align*}
		\mathbb{E}_{X_t}[1_A f(X_t)]
		=& \int_{A} \frac{\int_H x_0 n(x_0, x_t) \mathrm{d}\mu_\text{data}(x_0)}
		{\int_H n(x_0, x_t) \mathrm{d}\mu_\text{data}(x_0)} ~\mathrm{d} \mathbb{P}_t(x_t) \\
		=& \int_{H} \int_{A} \frac{\int_H x_0 n(x_0, x_t) \mathrm{d}\mu_\text{data}(x_0)}
		{\int_H n(x_0, x_t) \mathrm{d}\mu_\text{data}(x_0)} n(\tilde{x}_0, x_t) ~ \mathrm{d}\mathcal{N}(0, v_t C)(x_t)~ \mathrm{d} \mu_\text{data}(\tilde{x}_0) \\
		=& \int_{A} \int_H x_0 n(x_0, x_t) \mathrm{d}\mu_\text{data}(x_0)
		\frac{\int_H n(\tilde{x}_0, x_t) \mathrm{d}\mu_\text{data}(\tilde{x}_0)}{\int_H n(x_0, x_t) \mathrm{d}\mu_\text{data}(x_0)} 
		~ \mathrm{d}\mathcal{N}(0, v_t C)(x_t) \\
		=& \int_{A} \int_H x_0 n(x_0, x_t) \mathrm{d}\mu_\text{data}(x_0)
		~ \mathrm{d}\mathcal{N}(0, v_t C)(x_t) 
		= \mathbb{E}[1_A X_0].
	\end{align*}
	Since these two properties define the conditional expectation, we have shown that
	\[
	\mathbb{E}[X_0 | X_t = x] = f(x)
	\]
	almost surely. We will now proceed to show that $f$ is Lipschitz with respect to the Cameron-Martin norm $\|\cdot\|_U$. For notational convenience, we will define
	\begin{equation}
		\pi_t(x_t)
		= \int n(x_0, x_t) \mathd \mu_\text{data}(x_0) 
		= \int \exp\left(\frac{2\langle e^{-t/2} x_0, x_t \rangle_U - e^{-t}\|x_0\|^2_U}{2v_t}\right) \mathd \mu_\text{data}(x_0).
		\label{equ:pt_N0C_density}
	\end{equation}
	We see that
	\begin{equation}
		\begin{aligned}
			\pi_t(x_t + z) 
			=& \int \exp\left(\frac{2\langle e^{-t/2} x_0, x_t + z \rangle_U - e^{-t}\|x_0\|^2_U}{2v_t}\right) \mathd \mu_\text{data}(x_0) \\
			=& \int \exp\left(\frac{\langle e^{-t/2} x_0, z\rangle_U}{v_t} \right )
			n(x_0, x_t) ~\mathd \mu_\text{data}(x_0),
		\end{aligned}\label{equ:pt_N0C_density_shifted}
	\end{equation}
	which differs from \eqref{equ:pt_N0C_density} only by $\exp\left(\frac{\langle e^{-t/2} x_0, z\rangle_U}{v_t} \right)$. By our assumption that the support of $\mu_\text{data}$ is contained in a Cameron-Martin ball of size $R$. Therefore, 
	\begin{equation}
		\frac{\langle z, e^{-t/2} x_0 \rangle_U}{v_t} \le \frac{e^{-t/2}}{v_t}\|z\|_U R,
		\label{equ:upper_bound_potential}
	\end{equation}
	and
	\begin{equation}
		\exp\left(-R\|z\|_U \frac{e^{-t/2}}{v_t}\right) 
		\le \frac{\pi_t(x_t + z)}{\pi_t(x_t)} 
		\le \exp\left(R\|z\|_U \frac{e^{-t/2}}{v_t}\right).
		\label{equ:pt_N0C_ratio_bound}
	\end{equation}
	
	With these estimates out of the way, let us show local Lipschitz continuity:
	\begin{align*}
		\left \|f(x_t + z) - f(x_t) \right \|_{U}
		~=&~ \left \|\frac{\int x_0 ~ n(x_0, x_t + z) \mathd \mu_\text{data}(x_0)}{\pi_t(x_t + z)}
		- \frac{\int x_0 ~ n(x_0, x_t) \mathd \mu_\text{data}(x_0)}{\pi_t(x_t)} \right \|_U \\
		\eqqcolon &~ \left \|\frac{A'}{\pi_t(x_t+z)} - \frac{A}{\pi_t(x_t)} \right \|_U.
	\end{align*}
	We rewrite the above as
	\begin{align*}
		 \left \|\frac{A'}{\pi_t(x_t + z)} - \frac{A}{\pi_t(x_t)} \right \|_U 
		\le \left |1 - \frac{\pi_t(x_t+z)}{\pi_t(x_t)} \right | \left \|\frac{A'}{\pi_t(x_t+z)}\right \|_U + \frac{1}{\pi_t(x_t)}\|A' - A\|_U.
	\end{align*}
	We see that
	\begin{align*}
    	\left |1 - \frac{\pi_t(x_t)}{\pi_t(x_t)} \right | \left \|\frac{A'}{\pi_t(x_t + z)} \right \|_U \le \left(\exp\left(\frac{e^{-t/2}}{v_t}\|z\|_U R\right) - 1\right)  \left \|\mathbb{E}[X_0 | X_t = x_t + z] \right \|_U,
	\end{align*}
	where we used \eqref{equ:pt_N0C_ratio_bound} for the first term.
	We also get that
	\begin{align*}
		\frac{\|A - A'\|_U}{\pi_t(x_t)} 
		&\le \left(\exp\left(\frac{e^{-t/2}}{v_t}\|z\|_U R\right) - 1\right) \frac{1}{p_t(x_t)}\int \|x_0\|_U \frac{\mathd \mathcal{N}(e^{-t/2} x_0, v_t C)}{\mathd \mathcal{N}(0, v_t C)}(x_t) \mathd \mu_\text{data}(x_0) \\
		&\le \left(\exp\left(\frac{e^{-t/2}}{v_t}\|z\|_U R\right) - 1\right) R
	\end{align*}
	where we used \eqref{equ:upper_bound_potential} and \eqref{equ:pt_N0C_density_shifted} and our assumption that $\|x_0\|_U \le R$.
	Putting it all together, we get that
	\begin{align*}
		\|f(x_t + z) - f(x_t)\| 
		\le 2\left(\exp\left(\frac{e^{-t/2}}{v_t}\|z\|_U R\right) - 1\right) ~ R,
	\end{align*}
	where we again used that $\|x_0\| \le R$ to bound $f(x_t + z)$. 
	However, we can strengthen this bound. For any $N$, it holds that
	\begin{align*}
		\|f(x_t + z) - f(x_t)\|_U 
        &\le \sum_{i=1}^N \|f(x_t + z\frac{i}{N}) - f(x_t + z\frac{i-1}{N})\|_U\\ 
        &\le N~2\left(\exp\left(\frac{e^{-t/2}}{v_t}\frac{\|z\|_U}{N} R\right) - 1\right) ~ R.
	\end{align*}
	In particular, we can take the limit $N \to \infty$ and get that
	\begin{align*}
		\|f(x_t + z) - f(x_t)\|_U \le \frac{\mathrm{d}}{\mathrm{d}h}|_{h=0} 2\left(\exp\left(\frac{e^{-t/2}}{v_t}h\|z\|_U R\right) - 1\right) ~ R = 2R^2 \frac{e^{-t/2}}{v_t} \|z\|_U.
	\end{align*}
	From this we can conclude that $f$ has the global Lipschitz constant $2R^2 \frac{e^{-t/2}}{v_t}$.
	From \eqref{equ:s_as_sum_with_relative_densitites} we see that there is a version of $s(t, \cdot)$, such that for any $x_t, y_t \in H$
	\[
	\|s(t, x_t) - s(t, y_t)\|_U \le L_t \|x_t - y_t\|_U, \quad 
    L_t = \frac{1}{(1 - e^{-t})^2} \max\{1, 2R^2 e^{-t}\}.
	\]

    \paragraph{Step 2: Existence of solutions}
    Fix a $C$-Wiener process $W_t$. Denote by  $M: \mathcal{C}([0, T-u], H) \to \mathcal{C}([0, T-u])$ the map
    \[
        (M y)(t) = \int_0^t s(T-r, y_r) \mathd r + W_t.
    \]
    Then we have if $u < T - \varepsilon$,
    \begin{equation}
        \sup_{t \leqslant u} \| My(t) - M\tilde{y}(t)\|_U
        \leqslant \int_0^{u} \|s(T-r, y_r) - s(T-r, \tilde{y}_r)\|_U \mathd r
        \leqslant u L_{T-\varepsilon} \sup_{t \leqslant u} \|y_t - \tilde{y}_t\|_U.
        \label{equ:fixpoint_map_U}
    \end{equation}
    Here we used that we can apply Jensen's inequality because $\|\cdot\|_U: H \to [0, \infty]$ is lower-semicontinuous on $H$; see Proposition \ref{prop:U_norm_lower_semicontinuous}. We now choose $u$ smaller than $\frac{1}{L_{T - \varepsilon}}$. Starting with any $y^0$, we can now define the sequence $y^{n+1} = My^n$. Applying \eqref{equ:fixpoint_map_U} and noting that $\frac{u}{L_{T-\varepsilon}} < 1$, we see that $y_n$ is Cauchy with respect to $\sup \|\cdot\|_U$ and therefore also with respect to $\sup \|\cdot\|_H$ and has a limit. We then have a strong solution w.r.t. to the fixed Wiener process $W_t$ on $[0,u]$. We can extend the solution to $[0, T - \varepsilon]$ by repeating this process and gluing the solutions together. However, we cannot apply Banach's fix point theorem to get uniqueness, since $\|y - \tilde{y}\|_U$ might be infinite for two solutions $y$ and $\tilde{y}$. 
    
    \paragraph{Step 3: Strong uniqueness of solutions} We now assume we have two solutions to the reverse SDE, $Y_t$ and $\tilde{Y}_t$ solving the reverse SDE \eqref{reverse SDE inf abs} with respect to the same Wiener process $B_t$ and $Y_0 = \tilde{Y_0}$. Since $Y_t$ and $\tilde{Y}_t$ are not necessarily in $U$, we have to be a bit careful before directly applying Grönwall. Again, in Proposition \ref{prop:U_norm_lower_semicontinuous} we have proven that $\|\cdot\|_U: H \to [0, \infty]$ is lower-semicontinuous on $H$. We define the seminorms \[
        \|x\|_{U^D} = \|P^D x \|_{U}
    \]
    where $P^D$ is the projection operator defined in Section \label{section:proof_theorem_well_defined}. Those are smooth functions on $H$ and therefore the map $t \mapsto \|Y_t - \tilde{Y}_t\|_{U^D}$ is continuous and we can apply Grönwall:
\begin{align*}
    \frac{\mathd }{\mathd t}\|Y_t - \tilde{Y}_t\|_{U^D}
    \le \frac{e^{-\frac{t}{2}}}{1 - e^{-t}} \|\mathbb{E} [X_0 |X_t = Y_t] - \mathbb{E} [X_0 |X_t = \tilde{Y}_t]\|_{U^D} \le \frac{e^{-\frac{T-t}{2}}}{1 - e^{-(T-t)}} 2R.
\end{align*}
which in particular shows that 
    \[
        \|Y_t - \tilde{Y}_t - (Y_s - \tilde{Y}_s)\|_{U^D} \le \frac{e^{-\frac{{T-t}}{2}}}{1 - e^{-{T-t}}} 2R (t - s)
    \]
    for $t \ge s$. By taking the limit, we see that $t \mapsto \|Y_t - \tilde{Y}_t\|_{U}$ is continuous. Therefore, we can apply Grönwall to that quantity (continuity is a requirement for Grönwall).
    
    Furthermore, by using the above calculation for $\tau = 0$, we find that $\|Y_t - \tilde{Y}_t\|_U$ will be finite for all $t$, if $\|Y_0 - \tilde{Y}_0\|_U$ is finite (even if $Y_t$ and $\tilde{Y}_t$ are almost surely not in $U$). Now, since $Y_0 = \tilde{Y}_0$ and therefore $\|Y_0 - \tilde{Y}_0\|_U = 0$,
    \[
        \|Y_t - \tilde{Y}_t\|_{U}
        \leqslant \int_0^t \|s(T-r, Y_r) - s(T-r, \tilde{Y}_r)\|_{U^d} \mathd r
        \leqslant \int_0^t L_{T-r}^2 \|Y_t - \tilde{Y}_t\|_U \mathd r.
    \]
    If $t < T$, the Lipschitz constant $L_{T - t}$ is finite. Therefore, we can apply Grönwall to see that $\|Y_t - \tilde{Y}_t\|_U = 0$ to prove uniqueness on $[0, t]$ for any $t < T$. Hence, we have shown that there is a unique strong solution on $[0, T-\varepsilon]$. Since $\varepsilon$ was arbitrary, we have shown strong uniqueness on $[0, T)$. 
\end{proof}

\subsection{Proof of Theorem \ref{thm:uniqueness_gaussian}}\label{sec:uniqueness_gaussian}
Now we prove Theorem \ref{thm:uniqueness_gaussian}:

\begin{proof}
	\paragraph{Step 0: A priori bounds}
	Let $H$ be any Hilbert space on which $\mathcal{N}(0, C_\mu)$ is supported. Let $e_i$ be an eigenbasis von $C_\mu$ and $C$, which exists by our assumptions. Let $c_i$ and $\mu_i$ be the eigenvalues associated with $C$ and $C_\mu$ respectively, i.e.,
	\[
	C e_i = c_i e_i, \quad C_\mu e_i = \mu_i e_i.
	\]
	Furthermore, we define $C_t$ by
	\begin{equation}
		C_t = (e^{-t} C_\mu + (1 - e^{-t}) C),\label{equ:def_ct}
	\end{equation}
	which would be the covariance of $X_t$ at time $t$ in case $\Phi = 0$. The following operators are all bounded: $C_\mu C_t^{-1}$, $C C_t^{-1}$ and $C C_t C_\mu^{-1}$.
	We will show it for the first operator, and the others follow by similar arguments. The first operator will have eigenvalues
	\[
	\lambda_i = \frac{\mu_i}{e^{-t} c_i + (1 - e^{-t}) \mu_i}.
	\]
	We know that $\mu_i \to 0$ since $C_\mu$ is trace class. For $c_i \to 0$, the eigenvalues converge to $\lambda_i \to \frac{1}{(1 - e^{-t})}$,	while as $c_i \to \infty$, we get that $\lambda_i \to 0$. Furthermore, the derivative with respect to $c_i$ is negative, which shows that the $\lambda_i$ are bounded by $\frac{1}{1 - e^{-t}}$. A similar calculation can be done for the other operators listed.
	
	\paragraph{Step 1: Rewrite the reverse SDE} 
	The goal of this section is to show that we can write the drift in infinite dimensions as
	\[
	s(t, x_t) = -e^{\frac{t}{2}} \mathbb{E}[C (C_{\mu} C_t^{- 1})^{-1} \nabla\Phi(X_0) | X_t = x_t] + C C_t^{-1} x_t
	\]
	without assuming any more conditions on $\Phi$ or $\nabla \Phi$. To that end, we will argue in finite dimensions and take the limit in the end. Henceforth, unless we say otherwise, everything will be in finite dimensions for this step. The projection of $\mu_\text{data}$ to $H^D$ will be defined by
	\[
	\mu_\text{data}^D = \exp(-\Phi^D) \mathcal{N}(0, C_\mu^D),
	\]
	where $C_\mu^D$ is defined as in Section \ref{sec:spectral_approx} and
	\begin{align*}
	\exp(-\Phi^D(x^D)) 
    &= \mathbb{E}_{\mathcal{N}(0, C)}[\exp(-\Phi(X)) | X^D = x^D] \\
	&= \int \exp(-\Phi(x^D, x^{D+1:\infty})) ~ \mathd \mathcal{N}(0, C^{D+1:\infty})(x^{D+1:\infty})
	\end{align*}
	For the gradient it will then hold that
	\begin{align*}
		\nabla \Phi^D(x^D)
		&= \nabla \log \exp (\Phi^D(x^D)) \\
		&= \frac
		{\int \nabla_{x^D} \exp(-\Phi(x^D, x^{D+1:\infty})) ~ \mathd \mathcal{N}(0, C^{D+1:\infty})(x^{D+1:\infty})}
		{\int \exp(-\Phi(x^D, x^{D+1:\infty})) ~ \mathd \mathcal{N}(0, C^{D+1:\infty})(x^{D+1:\infty})} \\
		&= \frac
		{-\int \nabla_{x^D} \Phi(x^D, x^{D+1:\infty}) \exp(-\Phi(x^D, x^{D+1:\infty})) ~ \mathd \mathcal{N}(0, C^{D+1:\infty})(x^{D+1:\infty})}
		{\int \exp(-\Phi(x^D, x^{D+1:\infty})) ~ \mathd \mathcal{N}(0, C^{D+1:\infty})(x^{D+1:\infty})} \\
		&= \mathbb{E}_{\mu_\text{data}}[\nabla \Phi(x) | X^D = x^D].
	\end{align*}
	We denote by $(\tilde{X}_t: 0 \le t \le T)$ the Gaussian solution, started in $\tilde{X}_0 \sim \mathcal{N}(0, C)$ and as always by $(X_t: 0 \le t \le T)$ the solution to the forward process started in $X_0 \sim \mu_\text{data}$. To be precise,
	\[ \left(\begin{array}{c}
		\tilde{X}_0\\
		\tilde{X}_t
	\end{array}\right) \sim \mathcal{N} \left( 0, \left(\begin{array}{cc}
		C_{\mu} & e^{- \frac{t}{2}} C_{\mu}\\
		e^{- \frac{t}{2}} C_{\mu} & C_t
	\end{array}\right) \right) , \]
	where $C_t$ is defined in \eqref{equ:def_ct}.
	Therefore,
	\[ \tilde{X}_0 | \tilde{X}_t \sim \mathcal{N} \left( e^{- \frac{t}{2}} C_{\mu} C_t^{- 1} \tilde{X}_t,
	C_{\mu} - e^{- t} C_{\mu} C_t^{- 1} C_{\mu} \right). \]
	Then we have that
	\begin{equation}
		\frac{\mathd p_t}{\mathd \mathcal{N} (0, C_t)} (x) =\mathbb{E} [\exp (-
		\Phi (\tilde{X}_0)) |\tilde{X}_t = x] =\mathbb{E}_{\mathcal{N} \left( e^{- \frac{t}{2}}
			C_{\mu} C_t^{- 1} x, C_{\mu} - e^{- t} C_{\mu} C_t^{- 1} C_{\mu}
			\right)} [\exp (- \Phi (\tilde{X}_0))].
		\label{equ:density_pt_N0C}
	\end{equation}
	We denote by
	\begin{equation*}
		A_t = e^{- \frac{t}{2}} C_{\mu} C_t^{- 1}, \quad 
		Q_t = C_{\mu} - e^{- t} C_{\mu} C_t^{- 1} C_{\mu}.
	\end{equation*}
	Note that for $C = C_\mu$ all of the definitions simplify to easier terms.
	Taking $\nabla \log$ of the above leads to
	\begin{align*}
		\nabla_{x_t} \frac{\mathd p_t}{\mathd \mathcal{N} (0, C_t)} (x_t) 
		&= \frac{1}{Z} \int \nabla_{x_t} \exp(-\|A_t x_0 - x_t\|^2_{Q_t}) \exp(-\Phi(x_0)) \mathd x_0 \\
		&= \frac{1}{Z} \int -A_t^{-1} \nabla_{x_0} \exp(-\|A_t x_0 - x_t\|^2_{Q_t}) \exp(-\Phi(x_0)) \mathd x_0 \\
		&= \frac{1}{Z} \int A_t^{-1} \exp(-\|A_t x_0 - x_t\|^2_{Q_t}) \nabla_{x_0} \exp(-\Phi(x_0)) \mathd x_0  \\
		&= -\frac{1}{Z} \int A_t^{-1} \exp(-\|A_t x_0 - x_t\|^2_{Q_t}) \nabla_{x_0}\Phi(x_0) \exp(-\Phi(x_0)) \mathd x_0,
	\end{align*}
	where $B$ is the normalizing constant.
	Therefore,
	\begin{align*}
		C\nabla_{x_t} \log \frac{\mathd p_t}{\mathd \mathcal{N} (0, C_t)} (x_t)
		&= \frac{- \int C A_t^{-1} \exp(-\|A_t x_0 - x_t\|^2_{Q_t}) \nabla_{x_0}\Phi(x_0) \exp(-\Phi(x_0)) \mathd x_0 }
		{\int \exp(-\|A_t x_0 - x_t\|^2_{Q_t}) \exp(-\Phi(x_0)) \mathd x_0} \\
		&= \frac{\int C A_t^{-1}\nabla \Phi (x) \exp (- \Phi (x)) \mathd
			\mathcal{N} (A_t x_t, Q_t)}{\int \exp (\Phi (x)) \mathd \mathcal{N} (A_t
			x_t, Q_t)} \\
		&= \mathbb{E}[-CA_t^{-1}\nabla\Phi(X_0) | X_t = x_t].
	\end{align*}
	We make the dimension dependence explicit and get that
	\begin{align*}
		C\nabla_{x_t^D} \log \frac{\mathd p_t^D}{\mathd \mathcal{N} (0, C_t^D)} (x_t^D)
		&= \mathbb{E}[-C^D (A_t^{-1})^D \nabla\Phi^D(X_0^D) | X_t^D = x_t^D] \\
		&= \mathbb{E}[-C^D (A_t^{-1})^D \mathbb{E}[P^D \nabla\Phi(X_0)| X_0^D] | X_t^D = x_t^D] \\
		&= \mathbb{E}[-C^D (A_t^{-1})^D P^D \nabla\Phi(X_0) | X_t^D = x_t^D] \\
		&= P^D \mathbb{E}[-C A_t^{-1} \nabla\Phi(X_0) | X_t^D = x_t^D].
	\end{align*}
	Furthermore, 
	\begin{align}
		s^D(t, X_t^D) = C\nabla \log p^D_t(x_t^D) 
		=& C\nabla \log \frac{p^D_t(x_t^D)}{\mathcal{N}(0, C_t)} + C \nabla \log \mathcal{N}(0, C_t)(x^D_t) \nonumber \\
		=& P^D \mathbb{E}[-C A_t^{-1} \nabla\Phi(X_0) | X_t^D = x_t^D] + C^D (C_t^{-1})^D x_t^D. \label{equ:sD_gaussian_case_rewrite}
	\end{align}
	By Lemma \ref{lemma:approximations converge} $s^D$ converges almost surely to $s$. Furthermore, since $C A_t^{-1}$ is bounded and $\nabla \Phi(X_0)$ is Lipschitz, we get that
    \[
        \mathbb{E}[\|\mathbb{E}[C A_t^{-1} \nabla \log \Phi(X_0) | X_t]\|^2] 
        \lesssim \mathbb{E}[\|\nabla \log \Phi(X_0)\|^2] 
        \lesssim \mathbb{E}[\|X_0\|^2] < \infty.
    \]
    Therefore, also the first term in \eqref{equ:sD_gaussian_case_rewrite} converges to its infinite dimensional counterpart by Lemma \ref{lemma:approximations converge}. The second term in \eqref{equ:sD_gaussian_case_rewrite} also converges. Therefore, we can take the limit on both sides and obtain
	\[
	s(t, x_t) = \mathbb{E}[-C A_t^{-1} \nabla\Phi(X_0) | X_t = x_t] + C C_t^{-1} x_t.
	\]
	The formula \eqref{equ:density_pt_N0C} for $p_t$ holds in infinite dimensions too, and therefore we can write the conditional expectation as 
	\begin{equation*}
		f(t, x_t) = \mathbb{E}[-C A_t^{-1} \nabla\Phi(\tilde{X}_0) | \tilde{X}_t = x_t]
		= \frac{\int C A_t^{-1}\nabla \Phi (x) \exp (- \Phi (x)) \mathd
			\mathcal{N} (A_t x_t, Q_t)}{\int \exp (-\Phi (x)) \mathd \mathcal{N} (A_t
			x_t, Q_t)}.
	\end{equation*}
	\paragraph{Step 2: Local Lipschitzness in with respect to $\|\cdot\|$}
	Since $C C_t^{-1}$ is bounded by step $0$, it suffices to show that $f$ is locally Lipschitz. 
	We will now bound the difference
	\begin{eqnarray*}
		& & f(t, y_t) - f(t, x_t) \\
		& = & \frac{\int C A_t^{-1}\nabla \Phi (x) \exp (- \Phi (x)) \mathd
			\mathcal{N} (A_t x_t, Q_t)}{\int \exp (-\Phi (x)) \mathd \mathcal{N} (A_t
			x_t, Q_t)} 
		- \frac{\int C A_t^{-1}\nabla \Phi (x) \exp (- \Phi (x)) \mathd
			\mathcal{N} (A_t y_t, Q_t)}{\int \exp (-\Phi (x)) \mathd \mathcal{N} (A_t
			y_t, Q_t)}\\
		& \eqqcolon  & \frac{B_1}{Z_1} - \frac{B_2}{Z_2} = \left( \frac{1}{Z_1} -
		\frac{1}{Z_2} \right) B_2 - \frac{1}{Z_1} (B_1 - B_2) = \left( \frac{Z_1 -
			Z_2}{Z_1 Z_2} \right) B_2 - \frac{1}{Z_1} (B_1 - B_2) .
	\end{eqnarray*}
	We will fix an $R \geqslant 0$ and assume that $\| x_t \|, \| y_t \|
	\leqslant R$. Then, since $A_t$ is bounded, there exists an $\tilde{R}$ such that $\| A_t x_t \|, \| A_t y_t \| \leqslant \tilde{R}$.
	
	Then
	\begin{eqnarray*}
		Z_1 & = & \int \exp (- \Phi (x)) \mathd \mathcal{N} (A_t x_t, Q_t) (x ) =
		\int \exp (- \Phi (x + A_t x_t)) \mathd \mathcal{N} (0, Q_t) (x )\\
		& \geqslant & \int \exp (- (E_1 + E_2 \| x + A_t x_t \|^2)) \mathd
		\mathcal{N} (0, Q_t) (x )\\
		& = & \exp (- E_1 + 2 \| A_t x_t \|^2) \int \exp (- E_2 \| x \|^2) \mathd
		\mathcal{N} (0, Q_t) (x )\\
		& \gtrsim & E \exp (- 2 \| A_t x \|^2) \geqslant E \exp (- 2 \tilde{R}^2)
	\end{eqnarray*}
	where $E$ is a finite constant that only depends on $C_{\mu}, C, L$ and the
	$E_i$. A similar bound holds for $Z_2$. For
	\begin{eqnarray*}
		Z_1 - Z_2 & = & \int \exp (- \Phi (x + A_t x_t)) - \exp (- \Phi (x + A_t
		y_t)) \mathd \mathcal{N} (0, Q_t) (x )\\
		& \leqslant & \int \exp (-E_0) L \| A_t x_t - A_t y_t \| \mathd
		\mathcal{N} (0, Q_t) (x )\\
		& = & \exp (- E_0) L \| A_t x_t - A_t y_t \| \leqslant E L \| x_t - y_t,
		\|
	\end{eqnarray*}
	where we used that if $\Phi$ is $C^1$ and its derivative is bounded by $L$, then $\exp(-\Phi)$ has a derivative bounded by $\exp(-\inf \Phi)L$.
	Furthermore, we get that
	\begin{eqnarray*}
		& & \int C A_t^{-1} \nabla \Phi (x) \exp (- \Phi (x)) \mathd \mathcal{N} (A_t x_t, Q_t) \\
		& \leqslant & \exp (- E_0) \int C A_t^{-1} \nabla \Phi (x + A_t x_t) \mathd
		\mathcal{N} (0, Q_t)\\
		& \leqslant & \exp (- E_0)  \int \| C A_t^{-1}\| (\| \nabla \Phi (0) \| + L \| x
		+ A_t x_t \|) \mathd \mathcal{N} (0, Q_t)\\
		& \leqslant & \exp (- E_0) \| C A_t^{-1}\| \left( \| \nabla \Phi (0) \| + L \|
		A_t x_t \| + L \int \| x \| \mathd \mathcal{N} (0, Q_t) \right)\\
		& \leqslant & E (1 + \| x_t \|) \leqslant E (1 + R),
	\end{eqnarray*}
	where we used that $C A_t^{-1}$ and $A_t$ are bounded. 
	We also get that
	\begin{eqnarray*}
		&  & \| B_1 - B_2 \|\\
		& = & \int \| C A_t^{-1} \nabla \Phi (x + A_t x_t) \| | \exp (- \Phi (x + A_t
		x_t)) - \exp (- \Phi (x + A_t y_t)) | \mathd \mathcal{N} (0, Q_t) (x )\\
		&  & + \int \| C A_t^{-1}\nabla \Phi (x + A_t y_t) - C A_t^{-1}\nabla \Phi (x + A_t x_t)
		\| \exp (- \Phi (x + A_t y_t)) \mathd \mathcal{N} (0, Q_t) (x )\\
		& \leqslant & \exp (-E_0) L \| A_t x_t - A_t y_t \| \| C A_t^{-1}\| \left( \|
		\nabla \Phi (0) \| + L \left(\| A_t x_t \| + \int \| x \| \mathd \mathcal{N}
		(0, Q_t) (x ) \right)\right)\\
		&  & + L  \| C A_t^{-1} \| \| A_t x_t - A_t y_t \| \exp (- E_0)\\
		& \leqslant & E \| x_t - y_t \| (1 + R),
	\end{eqnarray*}
	where we again used the boundedness of $C A_t^{-1}$ and $A_t$. Putting it all together, we get that
	\begin{eqnarray*}
		&  & \left\| C \nabla \log \frac{\mathd p_t}{\mathd \mathcal{N} (0, C_t)}
		(x_t) - C \nabla \log \frac{\mathd p_t}{\mathd \mathcal{N} (0, C_t)} (y_t)
		\right\|\\
		& \leqslant & E \exp (4 \tilde{R}^2) L \| x_t - y_t \| + E \exp (2
		\tilde{R}) \| x_t - y_t \| (1 + R)
		\leqslant E \exp (4 \tilde{R}^2)  \| x_t - y_t \| .
	\end{eqnarray*}

    \paragraph{Step 3: Strong uniqueness and existence} Using the local Lipschitzness, we apply Grönwall to obtain strong uniqueness of solutions. This is a standard argument and similar to what we did in Step 3 of the proof of Theorem \ref{thm:uniqueness_gaussian}. Alternatively, see, for example, \citet[Theorem 2.5 in Section 5.2.B]{karatzas1991brownian}.

    We can now prove weak existence of the reverse SDE. By Theorem \ref{theorem SDEs well defined}, the time reversal will be a weak solution with initial condition $\mathbb{P}_T$. Denote the path measure of $Y$ by $\mathbb{Q}$. Under the assumptions of the Theorem, $\mathcal{N}(0, C)$ will be absolutely continuous with respect to $p_T$. We define $\tilde{Q}$ by
    \[
	\frac{\mathd \mathbb{\tilde{Q}}}{\mathd \mathbb{Q}}(y_{[0, T]}) = \frac{\mathd \mathcal N(0,C)}{\mathd \mathbb{P}_T}(y_0).
    \] 
    $\tilde{\mathbb{Q}}$ is then the path measure of a solution $\tilde{Y}$ to \ref{reverse SDE inf abs} which has initial condition $\mathcal N(0,C)$, therefore we have constructed a weak solution.    
    With that we conclude the proof since, weak existence together with strong uniqueness imply strong existence, see \citet[Section 5.3]{karatzas1991brownian}.

\end{proof}

\section{Wasserstein-Bound Proof}\label{sec:distance_proof}

We now prove Theorem \ref{theorem distance}:

\begin{proof}
	We prove the theorem with the finite-dimensional notation, but one can 
	replace $\nabla \log p_t$ by $s(t, \cdot)$ and nothing changes.
	We will partition $[0, T]$ into $\tau = \{ 0 = t_0, \ldots, t_N = T \}$. 
    For the given partition we denote
	\begin{eqnarray*}
		\lfloor t \rfloor =  \max_{t_i \in \tau} \{ t_i \leqslant t \},
        \quad \lceil t \rceil  =  \min_{t_i \in \tau} \{ t_i \geqslant t \}, 
        \quad \Delta = \max_{k = 0, \ldots, N - 1} t_{k + 1} - t_k.
	\end{eqnarray*}
	We couple two strong solutions of $Y_t$ and $\tilde{Y}_t$ for the same Brownian motion $B_t$. These strong solutions exist because of the assumptions of the theorem and Theorem \ref{thm:uniqueness_manifold} or Theorem \ref{thm:uniqueness_gaussian}.
    The difference between $Y_t$ and $\tilde{Y}_t$ can
	be bounded as follows:
	\begin{eqnarray*}
		\mathd \| Y_t - \tilde{Y}_t \|  & = & \frac{1}{2} \| Y_t - \tilde{Y}_t \| 
		+ \frac{1}{\| Y_t - \tilde{Y}_t \|} \langle Y_t - \tilde{Y}_t, C \nabla
		\log p_{T - t} (Y_t) - \tilde{s} (T - \lfloor t \rfloor,
		\tilde{Y}_{\lfloor t \rfloor}) \rangle \\
		& \leqslant & \frac{1}{2} \| Y_t - \tilde{Y}_t \| + \| C \nabla \log p_{T
			- t} (Y_t) - s_{\theta} (T - t, \tilde{Y}_{\lfloor t \rfloor}) \| .
	\end{eqnarray*}
	We bound
	\begin{eqnarray*}
		&  & \| C \nabla \log p_{T - t} (Y_t) - \tilde{s} (T - t, \tilde{Y}_t)
		\|\\
		& \leqslant & \| C \nabla \log p_{T - t} (Y_t) - C \nabla \log p_{T -
			\lfloor t \rfloor} (Y_{\lfloor t \rfloor}) \|\\
		&  & + \| C \nabla \log p_{T - \lfloor t \rfloor} (Y_{\lfloor t \rfloor})
		- \tilde{s} (\lfloor t \rfloor, Y_{\lfloor t \rfloor}) \| + \| \tilde{s}
		(\lfloor t \rfloor, Y_{\lfloor t \rfloor}) - \tilde{s} (\lfloor t \rfloor,
		\tilde{Y}_{\lfloor t \rfloor}) \|\\
		& \leqslant & \| \nabla_U \log p_{T - t} (Y_t) - \nabla_U \log p_{T -
			\lfloor t \rfloor} (Y_{\lfloor t \rfloor}) \|\\
		&  & + \| C \nabla \log p_{T - \lfloor t \rfloor} (Y_{\lfloor t \rfloor})
		- \tilde{s} (\lfloor t \rfloor, Y_{\lfloor t \rfloor}) \| + L_s \|
		Y_{\lfloor t \rfloor} - \tilde{Y}_{\lfloor t \rfloor} \| .
	\end{eqnarray*}
	We take the supremum to get rid of the delay term $\| Y_{\lfloor t \rfloor}
	- \tilde{Y}_{\lfloor t \rfloor} \|$ and obtain
	\begin{eqnarray*}
		 \sup_{\tau \leqslant s} \| Y_r - \tilde{Y}_r \|
		& \leqslant & L_s' \int_0^s \sup_{r \leqslant t} \| Y_r - \tilde{Y}_r \|
		\mathd t + \int_0^s \| \nabla_U \log p_{T - t} (Y_t) - \nabla_U \log p_{T
			- \lfloor t \rfloor} (Y_{\lfloor t \rfloor}) \| \mathd t\\
		&  & + \int \left\| \nabla_U \log \frac{p_{T - \lfloor t \rfloor}}{\nu}
		(Y_{\lfloor t \rfloor}) - \tilde{s} (\lfloor t \rfloor, Y_{\lfloor t
			\rfloor}) \right\| \mathd t,
	\end{eqnarray*}
	where $L_s' = L + \frac{1}{2}$.
	Squaring the above expression and taking expectations, we arrive at
	\begin{eqnarray*}
		&  & \mathbb{E} [\sup_{\tau \leqslant s} \| Y_r - \tilde{Y}_r \|^2]\\
		& \lesssim & {L_s'}^2 \int_0^s \mathbb{E} [\sup_{r \leqslant t} \| Y_r -
		\tilde{Y}_r \|^2] \mathd t + \int_0^s \mathbb{E} [\| \nabla_U \log p_{T -
			t} (Y_t) - \nabla_U \log p_{T - \lfloor t \rfloor} (Y_{\lfloor t \rfloor})
		\|^2] \mathd t\\
		&  & +\mathbb{E} \left[ \left\| \nabla_U \log \frac{p_{T - \lfloor t
				\rfloor}}{\nu} (Y_{\lfloor t \rfloor}) - \tilde{s} (\lfloor t \rfloor,
		Y_{\lfloor t \rfloor}) \right\|^2 \right] \mathd t\\
		& = & {L_s'}^2 \int_0^s \mathbb{E} [\sup_{r \leqslant t} \| Y_r -
		\tilde{Y}_r \|^2] \mathd t + \int_0^s B_1 + B_2 \mathd t
	\end{eqnarray*}
	We start by bounding $B_1$:
	\begin{eqnarray*}
		B_1 & \leqslant & \mathbb{E} \left[ \left\| \nabla_U \log p_{T - t} (Y_t)
		- e^{\frac{(t - \lfloor t \rfloor)}{2}} \nabla_U \log p_{T - \lfloor t
			\rfloor} (Y_{\lfloor t \rfloor}) \right\|^2 \right]\\
		&  & + \left( 1 - e^{\frac{(t - \lfloor t \rfloor)}{2}} \right)^2
		\mathbb{E} [\| \nabla_U \log p_{T - \lfloor t \rfloor} (Y_{\lfloor t
			\rfloor}) \|^2]\\
		& = & \mathbb{E} [\| \nabla_U \log p_{T - t} (Y_t) \|^2] -\mathbb{E}
		\left[ \left\| e^{\frac{(t - \lfloor t \rfloor)}{2}} \nabla_U \log p_{T -
			\lfloor t \rfloor} (Y_{\lfloor t \rfloor}) \right\|^2 \right]\\
		&  & + \left( 1 - e^{\frac{(t - \lfloor t \rfloor)}{2}} \right)^2
		\mathbb{E} [\| \nabla_U \log p_{T - \lfloor t \rfloor} (Y_{\lfloor t
			\rfloor}) \|^2]
	\end{eqnarray*}
	where we used that the $L^2$ norm of the a martingale $M_t$ difference is
	the difference of the $L^2$ norms, i.e., $\mathbb{E} [\| M_t - M_s \|^2]
	=\mathbb{E} [\| M_t \|^2] -\mathbb{E} [\| M_s \|^2]$ for $t \geqslant s$.
	Then
	\begin{align*}
		\int_0^T B_1 \mathd t \leqslant & \sum_{i = 1}^N (\mathbb{E} [\|
		\nabla_U \log p_{T - t_{k + 1}} (Y_{t_{k + 1}}) \|^2] -\mathbb{E} [\|
		e^{(t_{k + 1} - t_k)} \nabla_U \log p_{T - t_k} (Y_{t_k}) \|^2]) (t_{k +
			1} - t_k)\\
		&  + \left( 1 - e^{\frac{\Delta t}{2}} \right)^2 \mathbb{E} [\|
		\nabla_U \log p_T (Y_T) \|^2]\\
		& \leqslant \Delta t \sum_{i = 1}^N (\mathbb{E} [\| \nabla_U \log p_{T
			- t_{k + 1}} (Y_{t_{k + 1}}) \|^2] -\mathbb{E} [\| e^{(t_{k + 1} - t_k)}
		\nabla_U \log p_{T - t_k} (Y_{t_k}) \|^2])\\
		&  + \left( 1 - e^{\frac{\Delta t}{2}} \right)^2 \mathbb{E} [\|
		\nabla_U \log p_T (Y_T) \|^2]\\
		& \leqslant \Delta t\mathbb{E} [\| \nabla_U \log p_0 (Y_T) \|^2] +
		\left( 1 - e^{\frac{\Delta t}{2}} \right)^2 \mathbb{E} [\| \nabla_U \log
		p_T (Y_T) \|^2]\\
		  = & O (\Delta t) \mathbb{E} [\| \nabla_U \log p_0 (Y_T) \|^2]
	\end{align*}
	where we used that the $L^2$ norm of a martingale is increasing. The term
	$B_2$ is nothing more than the loss.
	
	Putting it all together, we arrive at
	\begin{eqnarray*}
		\mathbb{E} [\sup_{r \leqslant s} \| Y_r - \tilde{Y}_r \|^2] & \leqslant
		& L^2 \int_0^s \mathbb{E} [\sup_{r \leqslant t} \| Y_r - \tilde{Y}_r \|^2]
		\mathd t + O (\Delta t) \mathbb{E} [\| \nabla_U \log p_0 (Y_T) \|^2] +
		\tmop{Loss}\\
		& = & L^2 \int_0^s \mathbb{E} [\sup_{r \leqslant t} \| Y_r - \tilde{Y}_r
		\|^2] + \tmop{Error}
	\end{eqnarray*}
	and can apply Gr\"onwall to get that
	\begin{eqnarray*}
		\mathbb{E} [\sup_{r \leqslant s} \| Y_r - \tilde{Y}_r
		\|^2] & \leqslant & (\mathbb{E} [\| Y_0 - \tilde{Y}_0 \|^2]
		+ \tmop{Error}) \exp (L^2 s) .
	\end{eqnarray*}
	Since $Y_T \sim \mudata$ and $\tilde{Y}_T \sim
	\musample$ we found a coupling of $\mudata$ and
	$\musample$ and bounded its $L^2$ distance. We have not picked the
	coupling of $Y_0 \sim p_T$ and $\tilde{Y}_0 \sim \Normal (0, C)$ yet.
	Therefore we just pick a $\varepsilon$-optimal coupling in the squared Wasserstein
	distance, i.e., $\mathbb{E} [\| Y_0 - \tilde{Y}_0 \|^2]
	\leqslant \mathcal{W}_2^2 (p_T, \Normal (0, C)) + \varepsilon$ and obtain
	\begin{eqnarray*}
		\mathcal{W}_2^2 (\musample, \mudata) \leqslant
		\mathbb{E} [\sup_{r \leqslant T} \| Y_r - \tilde{Y}_r \|^2]  &
		\leqslant & (\mathcal{W}_2 (p_T, \Normal (0, C)) + \varepsilon +
		\tmop{Error} ) \exp (L^2 T) .
	\end{eqnarray*}
	Since $\varepsilon$ was arbitrary, the statement of the theorem follows, we actually get
	\begin{eqnarray*}
		\mathcal{W}_2^2 (\musample, \mudata) \leqslant
		\mathbb{E} [\sup_{r \leqslant T} \| Y_r - \tilde{Y}_r \|^2]  &
		\leqslant & (\mathcal{W}_2 (p_T, \Normal (0, C)) +
		\tmop{Error} ) \exp (L^2 T) .
	\end{eqnarray*}
	Finally, $\mathcal{W}_2^2(p_T, \Normal(0, C))$ can be upper bounded by 
	\[
	\mathcal{W}_2^2(p_T, \Normal(0, C)) \le \exp(-T) \mathcal{W}_2^2(\mudata, \Normal(0, C))
	\]
	since the Ornstein-Uhlenbeck forward process is contracting with rate $\exp(-t)$ in the squared $\mathcal{W}_2^2$-distance. From this, the statement of the theorem follows.
\end{proof}

\begin{proposition}
    Let $U$ be the Cameron-Martin space associated to a measure $\mathcal{N}(0, C)$ taking values in $H$. Then, $\|\cdot\|_U: H \to [0, \infty]$ is lower-semicontinuous and convex on $H$.
    \label{prop:U_norm_lower_semicontinuous}
\end{proposition}
\begin{proof}
    Let $(e_i, c_i)$ be the eigenvectors and eigenvalues of $C$. Let $f_k \to f$ in $H$. We will prove lower semicontinuity for $\|\cdot\|_U^2$, the result for $\|\cdot\|_U$ then follows. Then,
    \[
        \|f\|_U^2 
        = \sum_{d=1}^\infty \lim_{k \to \infty} \langle f_k, e_i \rangle c_{i}^{-1}
        \leqslant \liminf_{k \to \infty} \sum_{d=1}^\infty  \langle f_k, e_i \rangle c_{i}^{-1}
        = \liminf_{k \to \infty} \|f_k\|_U^2,
    \]
    which proves lower semi-continuity. Convexity follows since $\|\cdot\|_U$ is convex when restricted to $U$, and infinite otherwise. 
\end{proof}

\bibliography{lib}

\end{document}